\documentclass{article}

\usepackage{microtype}
\usepackage{graphicx}
\usepackage{booktabs} 

\usepackage{hyperref}
\usepackage{array}
\usepackage{multirow}

\usepackage[accepted]{icml2024}

\usepackage{amsmath}
\usepackage{amssymb}
\usepackage{mathtools}
\usepackage{amsthm}

\usepackage{caption}
\usepackage{subcaption}

\usepackage[capitalize,noabbrev]{cleveref}

\theoremstyle{plain}
\newtheorem{theorem}{Theorem}[section]

\newtheorem{lemma}[theorem]{Lemma}

\theoremstyle{definition}
\newtheorem{definition}[theorem]{Definition}

\theoremstyle{remark}

\usepackage[textsize=tiny]{todonotes}

\usepackage{color}
\newcommand{\mw}[1]{\textcolor{blue}{MW: #1}}

\icmltitlerunning{A Provably Effective Expert Pruning Method for Fine-tuned MoE}

\begin{document}

\twocolumn[
\icmltitle{A Provably Effective Method for Pruning Experts in Fine-tuned 
 Sparse Mixture-of-Experts}



\icmlsetsymbol{equal}{*}

\begin{icmlauthorlist}
\icmlauthor{Mohammed Nowaz Rabbani Chowdhury}{ecse_rpi}
\icmlauthor{Meng Wang}{ecse_rpi}
\icmlauthor{Kaoutar El Maghraoui}{ibm}
\icmlauthor{Naigang Wang}{ibm}
\icmlauthor{Pin-Yu Chen}{ibm}
\icmlauthor{Christopher Carothers}{csi_rpi}
\end{icmlauthorlist}

\icmlaffiliation{ecse_rpi}{Department of Electrical, Computer and Systems Engineering, Rensselaer Polytechnic Institute, NY, USA}
\icmlaffiliation{ibm}{IBM Research, Yorktown Heights, NY, USA}
\icmlaffiliation{csi_rpi}{Department of Computer Science, Rensselaer Polytechnic Institute, NY, USA}

\icmlcorrespondingauthor{Mohammed Nowaz Rabbani Chowdhury}{chowdm2@rpi.edu}
\icmlcorrespondingauthor{Meng Wang}{wangm7@rpi.edu}

\icmlkeywords{Mixture-of-Experts, Pruning}

\vskip 0.3in
]



\printAffiliationsAndNotice{}  

\begin{abstract}

The sparsely gated mixture of experts (MoE) architecture sends different inputs to different subnetworks (experts), 
through trainable routers. MoE reduces the training computation significantly for large models, but its deployment can be still memory/computation expensive for some downstream tasks.  Model pruning is a popular approach to reduce inference computation, but its application in MoE architecture is largely unexplored. To the best of our knowledge,  this paper provides the first provably efficient technique for pruning experts in fine-tuned MoE
models. 
We theoretically prove that prioritizing the pruning of the experts with a  smaller change of the router’s $l_2$ norm from the pre-trained model guarantees the preservation of test accuracy,
while significantly reducing the model size and the computational requirements.
Although our theoretical analysis is centered on binary classification tasks on simplified MoE architecture, our expert pruning method is verified on large vision MoE models such as V-MoE and $\text{E}^3$-MoE fine-tuned on benchmark datasets such as CIFAR-10, CIFAR-100, and ImageNet.   


\end{abstract}

\section{Introduction}\label{intro}

In deep learning, a typical approach is to adapt the same large pre-trained model (often based on the Transformer architecture \citep{vaswani2017attention}) to a plethora of downstream tasks. This method circumvents the substantial cost associated with training separate models for each task. 
This approach has consistently achieved state-of-the-art results in many application domains \citep{devlin2019bert,dosovitskiy2020image,wortsman2022model,yu2022coca}. Notably, a larger pre-trained model usually leads to better performance 
\citep{chen2022pali,chowdhery2023palm,dehghani2023scaling}. However, an increase in model size typically leads to a proportional rise in the computational resources required for training, scaling linearly with the number of parameters.

The sparsely gated mixture of experts (MoE) has been introduced to reduce the training cost of large models \citep{shazeer2017outrageously,lepikhin2020gshard}.  
The MoE layer in a Transformer encoder block replaces the single feed-forward network (FFN) module with multiple FFN modules, referred to as the \textit{experts}. Each expert is associated with a trainable router that selectively activates the expert based on the input tokens.  The sparse computation over the tokens allows for the expansion of the model size with only a sub-linear increase in the training compute \citep{shazeer2017outrageously,lepikhin2020gshard,fedus2022switch}.  

Due to their large model sizes, the deployment of MoE models still needs significant memory requirements, and inference compute \citep{riquelme2021scaling,zhou2022mixtureofexperts}, which may restrict their applications in 
resource-constrained environments.
Because different experts often learn diverse features \citep{lepikhin2020gshard,riquelme2021scaling,fedus2022switch}, and not all experts are essential for a particular downstream task (\citet{riquelme2021scaling}, Figure 29),   pruning irrelevant and redundant experts may reduce the inference compute and memory requirement of a fine-tuned MoE model while maintaining the inference accuracy. 

Despite the extensive research on neural network pruning \citep{han2015learning,han2016deep,han2016eie,frankle2018lottery,li2020efficient,pmlr-v202-frantar23a}, pruning experts in the MoE architecture has not been much explored. To the best of our knowledge,  only recent work by
\citet{chen2022task} and \citet{koishekenov-etal-2023-memory} 
prune non-essential experts, 
which are identified by the total number of tokens received by each expert throughout the fine-tuning stage or over the validation set after fine-tuning. These pruning methods can reduce memory requirements by reducing the number of experts and the communication among experts but cannot reduce the inference compute because tokens are routed to other experts after pruning. Moreover, no theoretical analysis has been provided, and it is not clear whether the total token count is the correct measure to identify essential experts. 
A fundamental question remains to be open: 
\begin{center}
    \textit{For a given downstream task, what characteristic of an MoE layer provably separates the essential experts so that removing the rest does not hurt generalization?}
\end{center}

This paper addresses this question empirically and theoretically 
on a binary supervised classification task\footnote{We consider binary tasks for the simplicity of the analysis. Our analysis can be extended to the multi-class case at the cost of higher complexity in the analysis.}. 
We analyze the training dynamics of the routers and experts of a pre-trained MoE layer during the fine-tuning stage. 
Our analysis reveals that prioritizing the pruning of the experts with smaller \textit{$l_2$  norm} change of the router's 
weights from the pre-trained model guarantees the preservation of test accuracy, while significantly reducing the model size and the computational requirements. The significance of our contributions are summarized as follows:

\textbf{1. Theoretical generalization analysis}: To the best of our knowledge, this paper provides \textbf{the first provably effective technique for pruning experts in  MoE models}. We theoretically prove that experts who learned task-relevant features (that determine labels) have larger changes of their router's $l_2$ norm during the fine-tuning stage than those experts that learn task-irrelevant features. 
We then prove that pruning experts with smaller changes of their router's $l_2$ norm can guarantee the pruning of irrelevant and redundant experts, while maintaining the same test accuracy in a wide range of pruning ratios.

\begin{figure}[ht]
    \centering
    \includegraphics[width=0.8\linewidth]{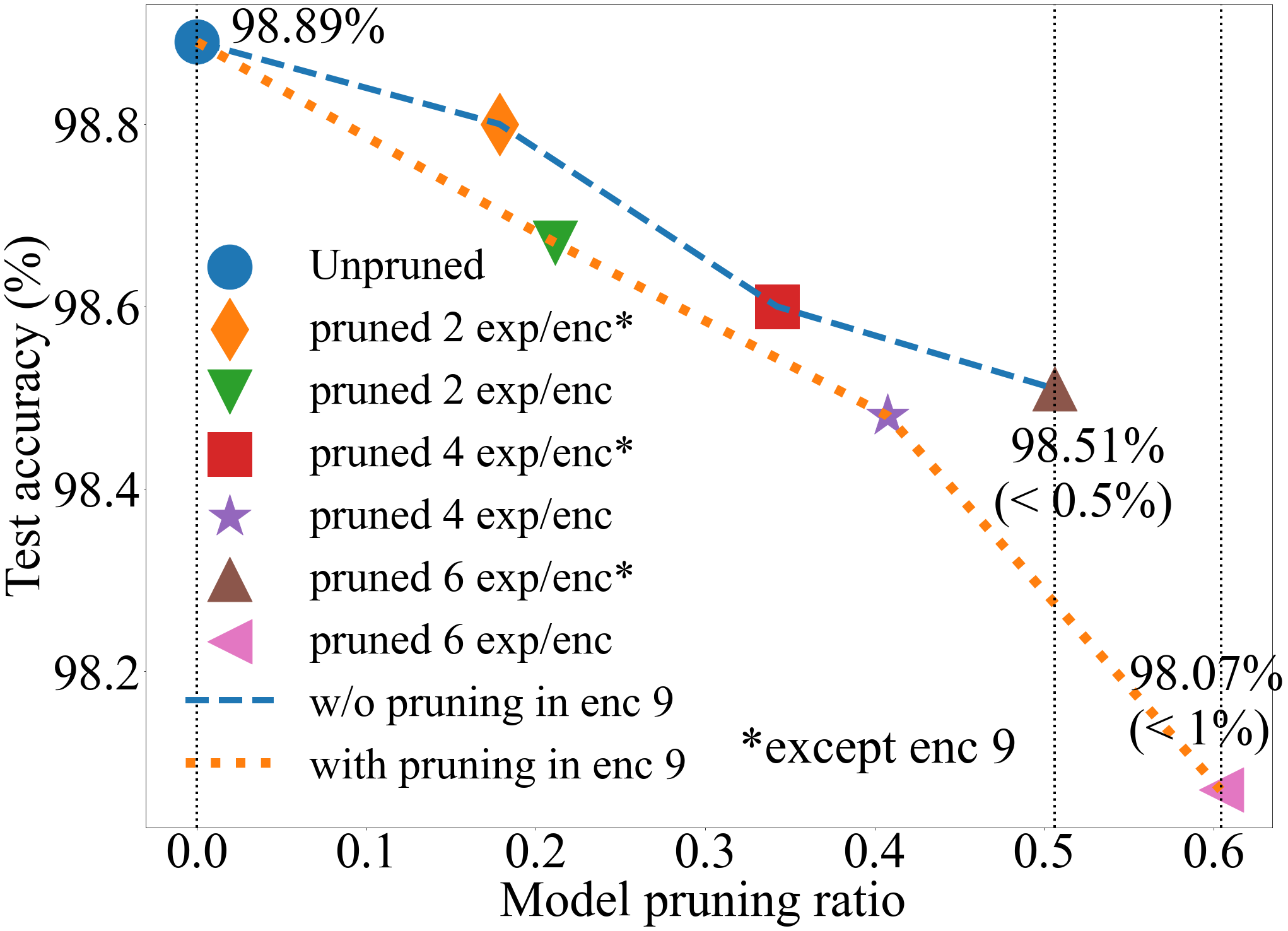}
    \vspace{-2mm}
    \caption{Generalization performance of the pruned VMoE on CIFAR-10 \textbf{with} post-pruning fine-tuning. `pruned 2 exp/enc' implies pruning two experts from each MoE encoder.}
    \label{cifar_10_pruned_finetuned}
\end{figure}

\textbf{2. Empirical validation}: We provide \textbf{experimental demonstration of the proposed pruning technique's effectiveness on state-of-the-art vision MoE models}. We evaluate on several vision MoE (VMoE) \citep{riquelme2021scaling} and ensembles of vision MoE (known as the efficient ensemble of experts, $\text{E}^3$) \citep{allingham2022sparse} models with thousands of millions of parameters, fine-tuned on benchmark datasets such as CIFAR-10, CIFAR-100, and ImageNet. 
For example, as shown in Figure \ref{cifar_10_pruned_finetuned} 
our method can prune 75\% of the experts of the fine-tuned V-MoE model on CIFAR-10 to reduce 60\% of the memory requirements, while maintaining the model accuracy within 1\% of the un-pruned model. Moreover, the method can reduce 40\% of the inference FLOPs and 40\% of inference time (see Section \ref{experiments} for details). Furthermore, our pruning technique is designed for seamless integration with contemporary digital hardware accelerators, such as GPUs and TPUs, and can be implemented without the need for specialized software or hardware modifications.


\section{Related Works}

\textbf{Mixture-of-Experts.} The sparsely gated MoE 
sends different tokens of an input sequence  to different experts in language models \citep{shazeer2017outrageously,lepikhin2020gshard,fedus2022switch,du2022glam} 
and vision models \citep{riquelme2021scaling,puigcerver2022adversarial,allingham2022sparse}. 
To address the challenge of load balancing among the experts \citep{lewis2021base},    \textit{expert-choice routing} is introduced in \citet{zhou2022mixtureofexperts} so that, 
instead of selecting experts for each token, the router selects tokens for each expert. 

Despite the empirical success of MoE, its theoretical analysis is underexplored except for a few recent works. Specifically,  \citet{chen2022towards} provides the theoretical generalization analysis of MoE with sample-level routing, and \citet{pmlr-v202-chowdhury23a} proves the computational efficiency of modern MoE with patch-level routing.

\textbf{Pruning deep neural networks.} 
Network pruning has been widely explored recently to reduce the computation and memory cost of deep neural networks (DNN)  \citep{han2016eie,luo2017thinet,lee2019snip,liu2021group,jaiswal2023instant}
Unstructured pruning methods prune individual weights \citep{han2015learning,han2016deep,frankle2018lottery,wang2020structured,liu2021ebert}, while structured pruning methods remove neurons, channels, and layers \citep{li2016pruning,tung2018clip,nonnenmacher2021sosp}. Although unstructured pruning can lead to a smaller model size because of the higher flexibility in pruning, the resulting irregular sparsity can lead to computational and memory overhead in practice. In contrast, structured pruning produces more regular and structured architecture,  leading to possibly better hardware utilization and computational efficiency. 

For large Transformer-based models, some focus on pruning pre-trained models \citep{chen2020lottery,zafrir2021prune,li2024training} while others focus on task-specific pruning during or after the fine-tuning stage \cite{wang2020structured,li2020train,sanh2020movement}. Note that all of these methods are compatible with our method and can be implemented together to further compress MoE models.

\textbf{Convergence and Generalization Analyses of Neural Networks.} 
The Neural Tangent Kernel (NTK) based approaches \citep{jacot2018neural,lee2019wide,du2019gradient,allen2019convergence,li2022generalization}
assume the model weights stay close to the initialization during the training process, which might not reflect the practical learning dynamics. 
The model estimation approach \citep{zhong2017recovery,zhang2020improved,zhang2020fast,fu2020guaranteed,zhang2023convergence,li2024does} requires the input data to satisfy the Gaussian distribution. 
Recent works based on the feature learning framework  \citep{daniely2020learning,shalev2020computational,shi2021theoretical,allen2022feature,zhang2022joint,li2023a,allen-zhu2023towards,pmlr-v202-chowdhury23a} can better characterize the practical learning dynamics in which the neural network gradually learns important features and discards unimportant features.  
Our theoretical analysis follows the feature learning framework.

\section{Method}

\subsection{The Mixture-of-Experts Architecture}

An MoE layer consists of multiple position-wise FFN modules referred to as experts, each associated with a router. The routers are collectively referred to as the gating network. Generally, MoE is implemented in the transformer encoder and decoder. As input samples in a transformer are tokenized, the MoE layer receives tokens as input.

Let us denote $x=\left[x^{(1)^T},x^{(2)^T}, ..., x^{(n)^T}\right] \in \mathbb{R}^{dn}$ as the input sample of an MoE layer tokenized into $n$ tokens of dimension $d$. Here, $x^{(j)}\in\mathbb{R}^d$ denotes the $j$-th token where $j\in[n]$. The MoE layer generates the $n$ corresponding tokens of dimension $d^\prime$ of the output $x_{out}=\left[x_{out}^{(1)^T},x_{out}^{(2)^T}, ..., x_{out}^{(n)^T}\right]$. The layer with $k$ experts computes the output $x_{out}^{(j)}\in\mathbb{R}^{d^\prime}$ for the token $x^{(j)}$ as,
\begin{align}
    &x_{out}^{(j)}=\sum_{s\in[k]}f_s\left(x^{(j)}\right) \text{ where, }\label{eq_1}\\ 
    &f_s\left(x^{(j)}\right) = W_2^{(s)}\sigma\left(W_1^{(s)^T}x^{(j)}\right)G_j^{(s)}\label{eq_2}
\end{align}
$f_s\left(x^{(j)}\right)$ is the output for the expert $s\in[k]$ for input token $x^{(j)}$. Here, $W_1^{(s)}\in\mathbb{R}^{d\times m}$ represents the hidden layer weights of the expert $s\in[k]$ with hidden dimension $m$, i.e. with $m$ hidden neurons. We denote the neuron $r\in[m]$ of expert $s\in[k]$ (the $r$-th column of $W_1^{(s)}$) by $w_r^{(s)}$. $\sigma(\cdot)$ is the element-wise activation function, and $W_2^{(s)}\in\mathbb{R}^{d^\prime\times m}$ is the output layer weights converting hidden representations into output tokens.

\textbf{The gating network}. $G_j^{(s)}\in[0,1]$ 
is the output of the gating network, referred to as the \textit{gating value} associated with the expert $s$ and the $j$-th token. To show how $G_j^{(s)}$ is calculated, for the MoE layer with $k$ experts, the gating network contains $k$ corresponding trainable routers denoted as $\{w_s\}_{s=1}^k\in\mathbb{R}^d$. 
We follow the definition in \citet{riquelme2021scaling} that the \textit{routing value} associated with the router $w_s$ and token $x^{(j)}$ as $g_{j}^{(s)}:=\langle w_s,x^{(j)}\rangle$.


The routing function is characterized as either \textit{token-choice routing} \citep{fedus2022switch} or the \textit{expert-choice routing} \citep{zhou2022mixtureofexperts}. 
Token-choice routing selects 
the top $l$ experts for each token, according to the routing values of that token over the $k$ experts. 
Let $J_j\subset[k]$ with $|J_j|=l$ denote the index set  
of the top $l$ experts for the $j$-th token. Then gating values of the top $l$ experts are non-zero and  calculated as the softmax function over their routing values, while gating values for other experts are zero,  i.e., 

\begin{equation}\label{eqn:gate_token}
G_j^{(s)}:=\begin{cases}
    e^{g_j^{(s)}}/\sum_{i\in J_j}e^{g_j^{(i)}} & \text{ if } s\in J_j\\
    0  \quad \text{ else}
\end{cases}.
\end{equation}

Similarly, expert-choice routing selects top $l$ tokens for each expert, 
according to the routing values of that expert over the $n$ tokens of the input. 
Let   $J_s\subset[n]$ with $|J_s|=l$ denote the index set  
of the top $l$ tokens for the $s$-th expert.  The gating values are computed by  

\begin{equation}\label{eqn:gate_expert}
G_j^{(s)}:=\begin{cases}
    e^{g_j^{(s)}}/\sum_{i\in J_s}e^{g_i^{(s)}} & \text{if } j\in J_s\\
    0 \quad \text{ else}
\end{cases}.
\end{equation}

See Figure \ref{arch} for a visual description of the two routing methods.

Note that zero gating values in (\ref{eqn:gate_token})-(\ref{eqn:gate_expert}) are introduced for mathematical completeness. In the implementation of sparse computation,    $f_s\left(x^{(j)}\right)$ is directly set as zero when $s$ is not in $J_s$ for token-choice routing or when $j$ is not in $J_s$ for expert-choice routing, without actually computing (\ref{eq_2}). 
That means tokens are only \textbf{\textit{routed}} to the experts with non-zero gating values. 

\begin{figure}[ht]
    \centering
    \includegraphics[width=\linewidth]{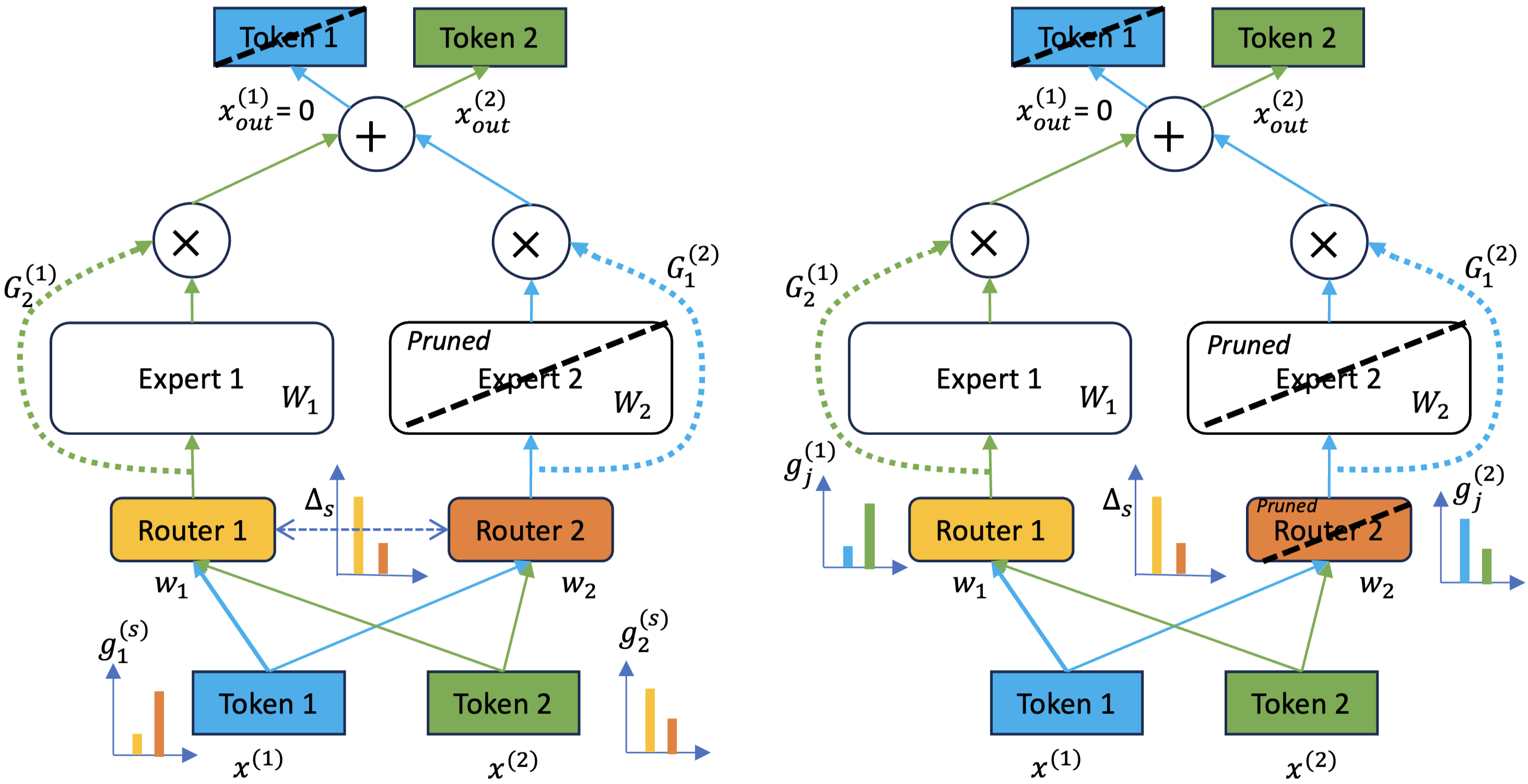}
    \vspace{-2mm}
    \caption{\textbf{Left:} Token-choice routing:  each token selects experts based on the routing values over the experts. 
    \textbf{Right:} Expert-choice routing: Each expert selects tokens based on the routing value over the tokens. In both cases, the experts with a smaller norm change of router's weights are pruned (Expert 2). The output tokens for the pruned experts are set to zero (Token 1 in the figure). In the left, the routers of the pruned experts are retained to calculate the gating value. 
    In the right, the routers of the pruned experts are also pruned.}
    \label{arch}
\end{figure}

\subsection{Expert Pruning Method in MoE}\label{pruning_method}


Let   $\{w_s^{(0)},W_1^{(s,0)},W_2^{(s,0)}\}_{s=1}^k$ denote the pre-trained weights of an MoE layer.
The model is fine-tuned using the stochastic gradient descent (SGD) with batch size $B$, the expert learning rate $\eta_e$ for parameters $W_1^{(s)}$ and $W_2^{(s)}$, and the router learning rate $\eta_r$ for $w_s$.
Let $\{w_s^{(T)},W_1^{(s,T)},W_2^{(s,T)}\}_{s=1}^k$ denote the corresponding fine-tuned weights after $T$ iterations. 
We define the \textit{change of router's} $l_2$ \textit{norm} from the pre-trained model for any expert $s\in[k]$ after $T$ training steps as,
\begin{center}
    $\Delta_s^{(T)}:=\|w_s^{(T)}\|-\|w_s^{(0)}\|$
\end{center}
Let   $S_{k^\prime}$ denote the index set of remaining experts after pruning, where  $k^\prime$ is the number of unpruned experts.  $S_{k^\prime}$ is selected based on $\Delta_s^{(T)}$. For example, $S_{k^\prime}$ can contain experts  with \textbf{top $k^\prime$} values of $\{\Delta_s^{(T)}\}_{s=1}^k$. 
We define the \textit{expert-pruning-ratio} of the pruned model as $\rho:=1-(k^\prime/k)$.  
Note that the computation of nonzero gating values $G_j^{(s)}$ is independent of the expert pruning. 
In token-choice routing, 
each token is only routed to those experts in $S_{k^\prime}$ that are among the selected $l$ experts for that token. The resulting gating values of token-choice routing with expert pruning is 

\begin{equation}\label{eqn:gate_token_prune}
G_j^{(s)}:=\begin{cases}
    e^{g_j^{(s)}}/\sum_{i\in J_j}e^{g_j^{(i)}} & \text{ if } s\in J_j \cap S_{k'}\\
    0  \quad \text{ else}
\end{cases}.
\end{equation}

In expert-choice routing, because the routers select $l$ tokens for each expert, if an expert is pruned, its corresponding router is also pruned. 
The resulting gating values of expert-choice routing with expert pruning is 

\begin{equation}\label{eqn:gate_expert_prune}
G_j^{(s)}:=\begin{cases}
    e^{g_j^{(s)}}/\sum_{i\in J_s}e^{g_i^{(s)}} & \text{if } j\in J_s \text{ and } s \in S_{k'}\\
    0 \quad \text{ else}
\end{cases}.
\end{equation}

See Figure \ref{arch} for a visual description of the pruning method for the two routing techniques.

The model after pruning can be  directly employed in a downstream task. It can also be fine-tuned again post-pruning using SGD with batch size $B$ and learning rates $\eta_e$ and $\eta_r$ for $T'$ iterations before deployment. 

\section{Theoretical Guarantees of the Expert Pruning Method}
\subsection{Key Theoretical Findings}
 We consider the setup that some tokens represent task-specific features that determine the data label in the downstream tasks, while some other tokens represent task-irrelevant features that do not affect the label.  Before presenting our analysis setup and the formal theoretical results, we first present the key insights. 

\textbf{(I) Experts learning task-specific features have a large change in router's norm, while experts not learning task-specific features have a small change in router's norm}. We theoretically show that the experts that learned task-specific features for the downstream task to a sufficient extent after pre-training continue to learn the task-specific features during fine-tuning, which leads to a large \textit{change of the router's} $l_2$ \textit{norm} in the fine-tuned model (Lemma \ref{lemma_1}). In contrast,  the experts that do not learn task-specific features will still only learn irrelevant features during fine-tuning and 
have a small \textit{change of the router's} $l_2$ \textit{norm} in the fine-tuned model.



 \textbf{(II) Post-pruning fine-tuning promotes unpruned experts to learn   task-specific features}. We show that the routers of the unpruned experts always selects task-specific patterns. As a result, post-pruning fine-tuning can promote the neurons of these experts to learn task-specific features. 

\textbf{(III) The pruned model provides guaranteed generalization for a wide range of pruning ratios}. Using the above two findings, we formally show all the experts not learning task-specific features can be pruned without hurting generalization accuracy. Moreover, if the pruned model can be further fine-tuned, one can prune up to $1-O(1/k)$ fraction of experts with the least changes of router's $l_2$ norm while maintaining the generalization accuracy. 

\subsection{The Analysis Setup}\label{setup}

\textbf{Network architecture}. We consider fine-tuning a pre-trained MoE layer with $k$ experts on a binary classification task.  Any input samples $(x,y)$ is drawn from an unknown distribution $\mathcal{D}$ where $y\in\{+1,-1\}$. The output of the fine-tuned model is defined as,
\begin{equation}\label{analyzed_model}
    f(x):=\sum_{j=1}^nx_{out}^{(j)}=\sum_{j=1}^n\sum_{s\in[k]}f_s\left(x^{(j)}\right)
\end{equation}
Here, the output token dimension, $d^\prime = 1$. We replace the output layer weights $W_2^{(s)}$ of the pre-trained model by a fixed classification head $a^{(s)}\Vec{1}$ where $\Vec{1}$ is a vector of $1$s with dimension $m$ and where $a^{(s)}$ 
is generated from $\text{Unif}(\{+1,-1\})$. 
In other words, each expert is positively or negatively connected to the model's output. Following the typical setup of theoretical generalization analysis of neural networks \citep{li2018learning,brutzkus2018sgd,allen2019learning,arora2019fine,zhang2022joint}, $a^{(s)}$ is not updated during training.  The activation function $\sigma(\cdot)$ is the rectified linear unit (ReLU) i.e. $\sigma(z)=\text{ReLU}(z):=\max(z,0)$ for any $z\in\mathbb{R}$. 
 We analyze the expert-choice routing MoE with $l=O(1)$ and $k=O(\sqrt{d})$, but our theoretical results are also validated in experiments on token-choice routing MoEs.
 
\textbf{Training method}. The learning method minimizes the hinge loss, $\hat{l}(f(x),y) = \max(1-yf(x),0)$ while the gradient is evaluated on $l(f(x),y) = 1-yf(x)$, same as the setting in \citet{zhang2022joint}.
We employ vanilla SGD using a batch size of $B$ for $T$ training steps with learning rate $\eta_e$ and $\eta_r$ in the experts and the routers, respectively.  

\textbf{Pruning   model}.        Let us denote $S_1$ and $S_2$ as the two sets of experts positively and negatively connected to the output, respectively, i.e., $S_1:= \{s\in[k]:a^{(s)}=+1\}$ and $S_2:=\{s\in[k]:a^{(s)}=-1\}$. We prune the experts based on the \textit{change of router's $l_2$ norm} separately over the set $S_1$ and $S_2$. More specifically, given an expert pruning ratio $\rho$, we retain the experts with top $|S_1|(1-\rho)$ values of $\{\Delta_s^{(T)}\}_{s\in S_1}$ and top $|S_2|(1-\rho)$ values of $\{\Delta_s^{(T)}\}_{s\in S_2}$ to construct the set $S_{k^\prime}$ while pruning rest of the experts.   

\textbf{The data model}\footnote{Assumptions on the structure of the input data are required to analyze the feature-learning dynamics of neural networks. Similar assumptions have been made in many recent works \citep{brutzkus2021optimization,shi2021theoretical,karp2021local,allen2022feature,chen2022towards,zhang2022joint,li2022generalization,zhang2022joint,allen-zhu2023towards, pmlr-v202-chowdhury23a}.}. 
The tokens are drawn from a fixed orthonormal pattern set, denoted by $\mathcal{P}$, which includes  $d$ patterns in $\mathbb{R}^d$, where $d=\Omega(n)$. $\mathcal{P}$ includes two \textit{task-specific} patterns, denoted by $o_1$ and $o_2$, which determine the labels for class-1 ($y=+1$) and class-2 ($y=-1$), respectively.
Each input sample $x$ contains exactly one task-specific pattern, which determines the label. $\mathcal{P}$ also contains $d-2$ \textit{task-irrelevant}  patterns, denoted by $\{q_i\}_{i=1}^{d-2}$, which do  not affect labels. Each $q_i$  appears in both classes with the same probability. The probability of $q_i$ appearing in $x$ can vary for different $i$, but is $O(1/d)$ for all $i$. We denote the data generating distribution by $\mathcal{D}$.

\textbf{Experts' proficiency measure.} We introduce a probability measure to quantify the quality of a router's capability of selecting task-specific features. Let $p_1^{(s,t)}$  and $p_2^{(s,t)}$  denote the \textit{proficiency measure} of router $s$ at iteration $t$ in selecting the task-specific feature $o_1$  and $o_2$  with a gating value of at least $1/l$, respectively.   Specifically, 
\begin{align*}
    p_1^{(s,t)}:=\mathbb{P}[(x,+1)\sim\mathcal{D}:&\exists j\in J_s^{(t)} \text{ s.t. } \\&x^{(j)}=o_1 \text{ and, }G_j^{(s,t)}\ge1/l]
\end{align*}
Likewise for $p_2^{(s,t)}$. 
The larger values indicate the higher chances of expert $s$ in selecting task-specific features. 

\subsection{Main Generalizalization Results of Expert Pruning}

The generalization results of the pruned model without and with post-pruning fine-tuning are summarized in  Theorems \ref{thm_1} and \ref{thm_2}. We also present three important lemmas that lead to the theorems.   Lemmas \ref{lemma_1} and \ref{lemma_2} jointly show that  
the \textit{change of router's} $l_2$ \textit{norm} in an important expert that learns task-specific features is significantly larger than the $l_2$ norm change of the router of unimportant expert that learns task-irrelevant features. Lemma \ref{lemma_3} shows that post-pruning fine-turning can promote the unpruned experts to be   specialized in learning task-specific features. 

\begin{lemma}[Important experts become more specialized]\label{lemma_1}
Suppose the expert learning rate $\eta_e$, the router learning rate  $\eta_r$,  the batch-size $B$, and the number of iterations $T$ satisfy 
\begin{equation}\label{eqn:condition1}
    \eta_r=O(\eta_e/mdl^2),  B=\Omega(l^2d^2) 
    \end{equation}
    \begin{equation}\label{eqn:condition2}
    T=\Omega( l^2\sqrt{d\log l}/\eta_e).
\end{equation}

    For any expert $s\in S_1$ such that $p_1^{(s,0)}=\Omega(1)$, 
    we have
    
    (i) $p_1^{(s,T)}=1$, \\
    (ii) for every $(x,+1)\sim\mathcal{D}$, $G_j^{(s,T)}(x)>1/2$, if $x^{(j)}=o_1$,\\
    (iii) $\langle w_r^{(s,T)},o_1\rangle=\Omega( l\sqrt{d\log l})$, for a constant fraction  $r\in[m]$,\\ 
       (iv) $\Delta_s^{(T)}>\cfrac{3}{2}\log l$.
       The counterpart results also hold  for experts in $S_2$.
   
\end{lemma}
Lemma \ref{lemma_1} shows that expert $s$ in the pre-trained model that learns  the \textit{task-specific} feature (say, $o_1$) to some extent  (important experts, i.e., $p_1^{(s,0)}=\Omega(1)$) will become more specialized in learning the task-specific features after fine-tuning $T$ iterations, and the norm of the corresponding router has a significant increase.  
Specifically, (i) indicates that the router $s$ will always select $o_1$ as one of the $l$ tokens in all class-1 samples. (ii) shows that for a class-1 sample $x$, the gating value that corresponds to the token $o_1$ is large, at least $1/2$. (iii) shows that a constant fraction of neurons in expert $s$ has a large component along the direction of the task-specific feature. (iv) shows that the router norm has a significant increase after fine-tuning.



In contrast, Lemma \ref{lemma_2} shows that experts that do not learn   the downstream task-specific features (unimportant experts) will still only learn task-irrelevant features after fine-tuning,  and the norm changes of the corresponding routers are relatively small. 
\begin{lemma}[Unimportant experts stay unimportant]\label{lemma_2}
Suppose (\ref{eqn:condition1}) and (\ref{eqn:condition2}) hold. 
     For any expert $s\in S_1$ such that $p_1^{(s,0)}=O(1/d)$, 
     we have (i) $p_1^{(s,T)}=O(1/d)$,
     (ii) $\Delta_s^{(T)}=O(\log^2 l/\sqrt{d})+O(l^4\log^2l/d^2)$.      The counterpart results also hold  for experts in $S_2$.
\end{lemma}

Lemma \ref{lemma_2} (i) indicates that the router of the fine-tuned expert still can only select task-specific features for a very small fraction of data, less than $O(1/d)$, while the router fails to select task-specific features for the dominating fraction of the samples. Lemma \ref{lemma_2} (ii) shows that the router norm stays small. That is because $l$ is $O(1)$, much smaller than $d$. 

Next we will present the generalization result  of expert pruning without post-pruning fine-tuning. 
Let   $f^{(T)}(x)$ denote the output of the unpruned model in equation (\ref{analyzed_model}) after $T$ SGD iterations. Let $\hat{f}^{(T;\rho)}(x)$ denote the resulting  pruned model where $\rho$ (in $[0,1)$)  fraction  of experts are pruned.   
\begin{theorem}[Generalization of pruned model with no post-pruning fine-tuning]\label{thm_1}
   Suppose (\ref{eqn:condition1}) and (\ref{eqn:condition2}) hold, the number of experts 
      $k=O(\sqrt{d})$, and at least $\gamma$ fraction of $s\in S_1$ with $p_1^{(s,0)}=O(1/d)$ and $s^\prime\in S_2$ with $p_2^{(s^\prime,0)}=O(1/d)$. Then, we have
 for any $0 \leq \rho\le\gamma$, 
\begin{equation}
    \mathbb{P}\left[\forall (x,y)\sim\mathcal{D}: y\hat{f}^{(T;\rho)}(x)>0\right]=1.
\end{equation}
   
\end{theorem}

Theorem \ref{thm_1} shows that the pruned model with no post-pruning fine-tuning can still achieve zero generalization if up to $\gamma$
 fraction of experts with the smallest router $l_2$-norm change are removed. The intuition is that those $\gamma$ fraction of experts do not learn task-specific features, and, thus, removing these experts does not affect generalization.



We next analyze the impact of post-pruning fine-tuning. 
If the   pruned model is fine-tuned again for $T'$ iterations, let  $\hat{f}^{(T;\rho, T^\prime)}(x)$ denote resulting model after post-pruning fine-tuning.  Lemma \ref{lemma_3}
shows that the routers with large norm change are specialized in selecting task-specific patterns. If these experts remain in the model after pruning, post-pruning fine-tuning 
can ensure that the hidden neurons of the experts are trained to  learn task-specific features.
\begin{lemma}[Post-pruning fine-tuning promotes experts to learn task-specific features]\label{lemma_3}
    Suppose (\ref{eqn:condition1}) and (\ref{eqn:condition2}) hold. 
     For any expert $s\in S_1$ such that $\Delta_s^{(T)}>\frac{3}{2}\log l$, we have\\
     (i) $p_1^{(s,T)}=1$, \\
    (ii) for every $(x,+1)\sim\mathcal{D}$, $G_j^{(s,T)}(x)>1/2$, if $x^{(j)}=o_1$,\\
    Moreover, after pruning experts with the ratio $\rho\geq\gamma$, if $s\in S_{k^\prime}$, and the number of post-pruning fine-tuning steps   satisfies
    \begin{equation} \label{eqn:T'}
    T^\prime=\Omega(kl^2\sqrt{\log l}/\eta_e) 
    \end{equation}
   then  we have 
     \begin{equation}
 \langle w_r^{(s,T,T^\prime)},o_1\rangle=\Omega( kl^2\sqrt{\log l})
\end{equation}
         for a constant fraction $r\in[m]$, where $w_r^{(s,T,T^\prime)}$ is the resulting weights of neuron $r$ in expert $s$.  The counterpart results hold for experts in $S_2$.
\end{lemma}

The first half of Lemma \ref{lemma_3} shows that if expert $s$ has large router norm change after pre-pruning fine-tuning, then router $s$ is specialized in selecting task-specific patterns and the corresponding gating values for   task-specific patterns are large. The second of Lemma \ref{lemma_3} shows that if we prune all unimportant experts and then continue to further fine-tuning expert $s$, then a constant fraction of neurons in expert $s$ will be specialized in learning task-specific features. 
Therefore, if we allow post-pruning fine-tuning, the tolerable pruning rate can be   large,  as described in  Theorem \ref{thm_2}.

\begin{theorem}[Generalization of pruned model with  post-pruning fine-tuning]\label{thm_2}
 Suppose (\ref{eqn:condition1}), (\ref{eqn:condition2}) and (\ref{eqn:T'}) hold, and the number of experts 
      $k=O(\sqrt{d})$. 
      Then for any $\rho\le1-O(1/k)$,  
      \begin{equation}
          \mathbb{P}\left[\forall (x,y)\sim\mathcal{D}: y\hat{f}^{(T;\rho, T^\prime)}(x)>0\right]=1.
      \end{equation}
\end{theorem}

Theorem \ref{thm_2} guarantees that a high expert-pruning-ratio ($1-O(1/k)$) is achievable without hurting the generalization performance using post-pruning fine-tuning. The intuition is that post-pruning fine-tuning allows some experts to become more specialized in learning task-specific patterns, and therefore, pruning more experts does not hurt generalization overall.   The required number of post-pruning fine-tuning iterations $T'$ is less than pre-pruning fine-tuning iterations $T$, because $k=O(\sqrt{d})$.

\section{Experimental Results}\label{experiments}
\begin{figure}[ht]
 \vspace{-3mm}
    \begin{subfigure}[(a)]{0.90\linewidth}
        \centering
        \includegraphics[width=0.70\linewidth]{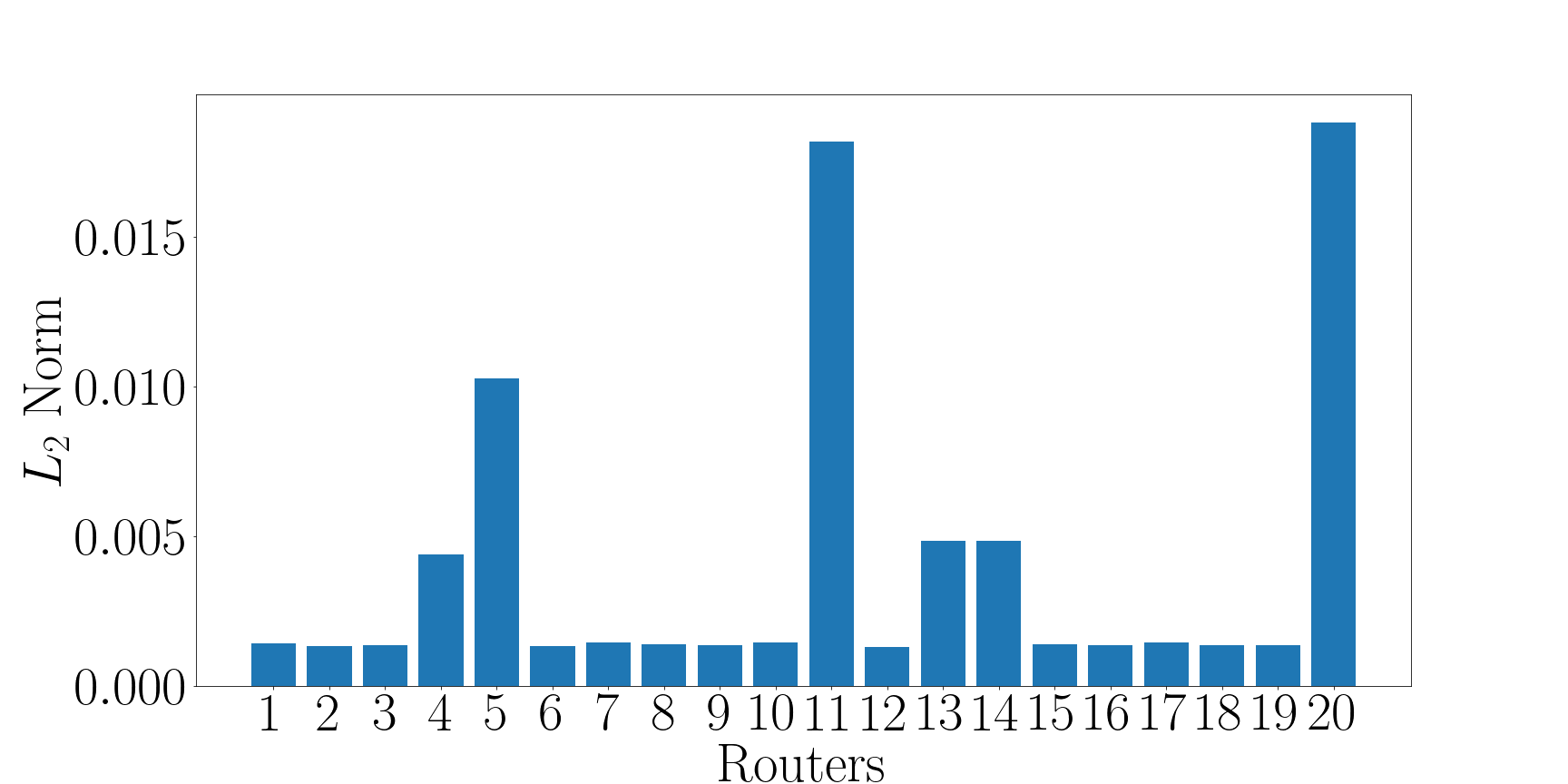}
        \caption{}
        \label{router_norm}
    \end{subfigure}
    \hfill
    \begin{subfigure}[(b)]{0.495\linewidth}
        \centering
        \includegraphics[width=\linewidth]{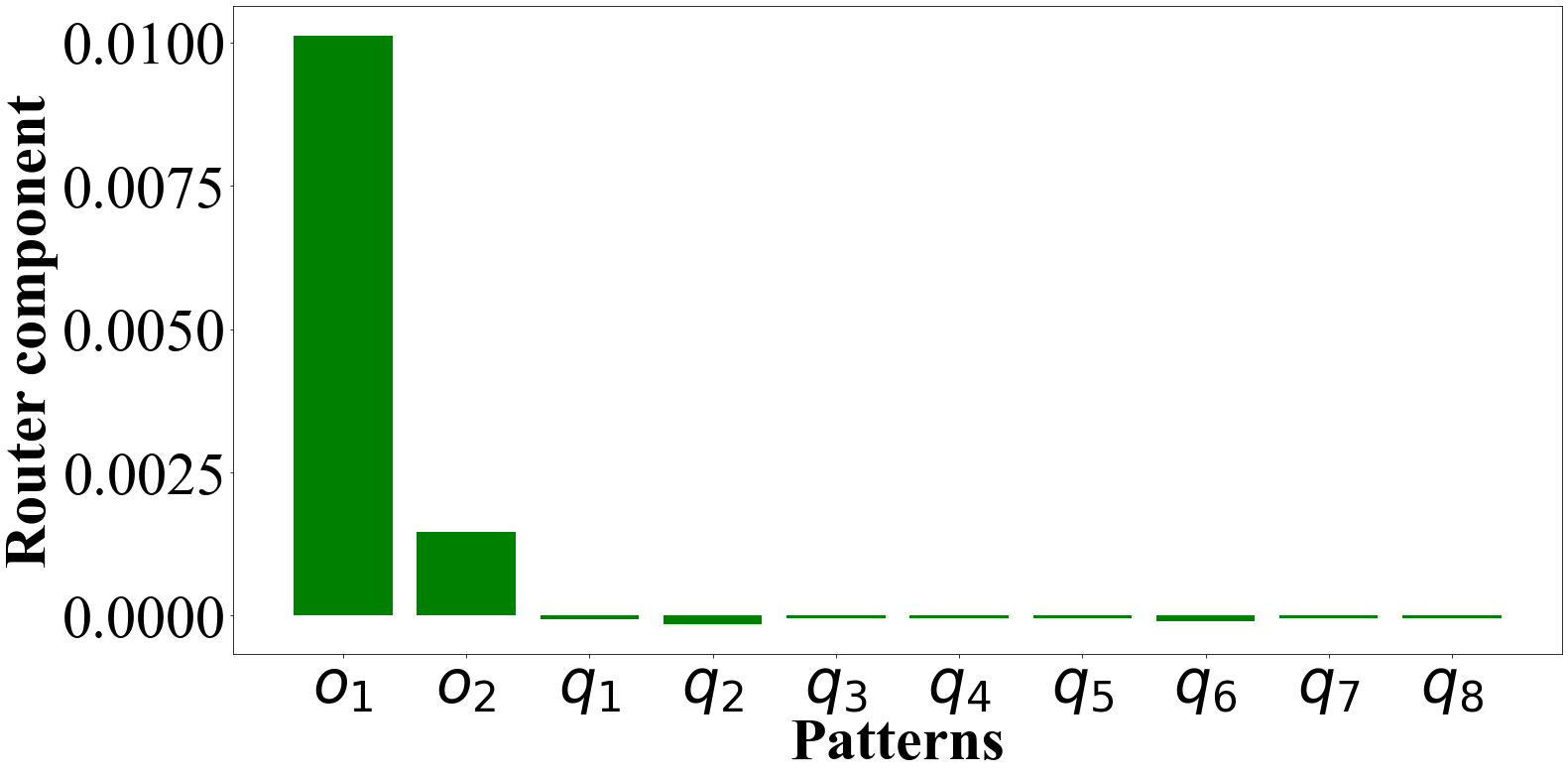}
        \caption{Router 5}
        \label{router_component_large_1}
    \end{subfigure}
    \hfill
    \begin{subfigure}[(c)]{0.495\linewidth}
        \centering
        \includegraphics[width=\linewidth]{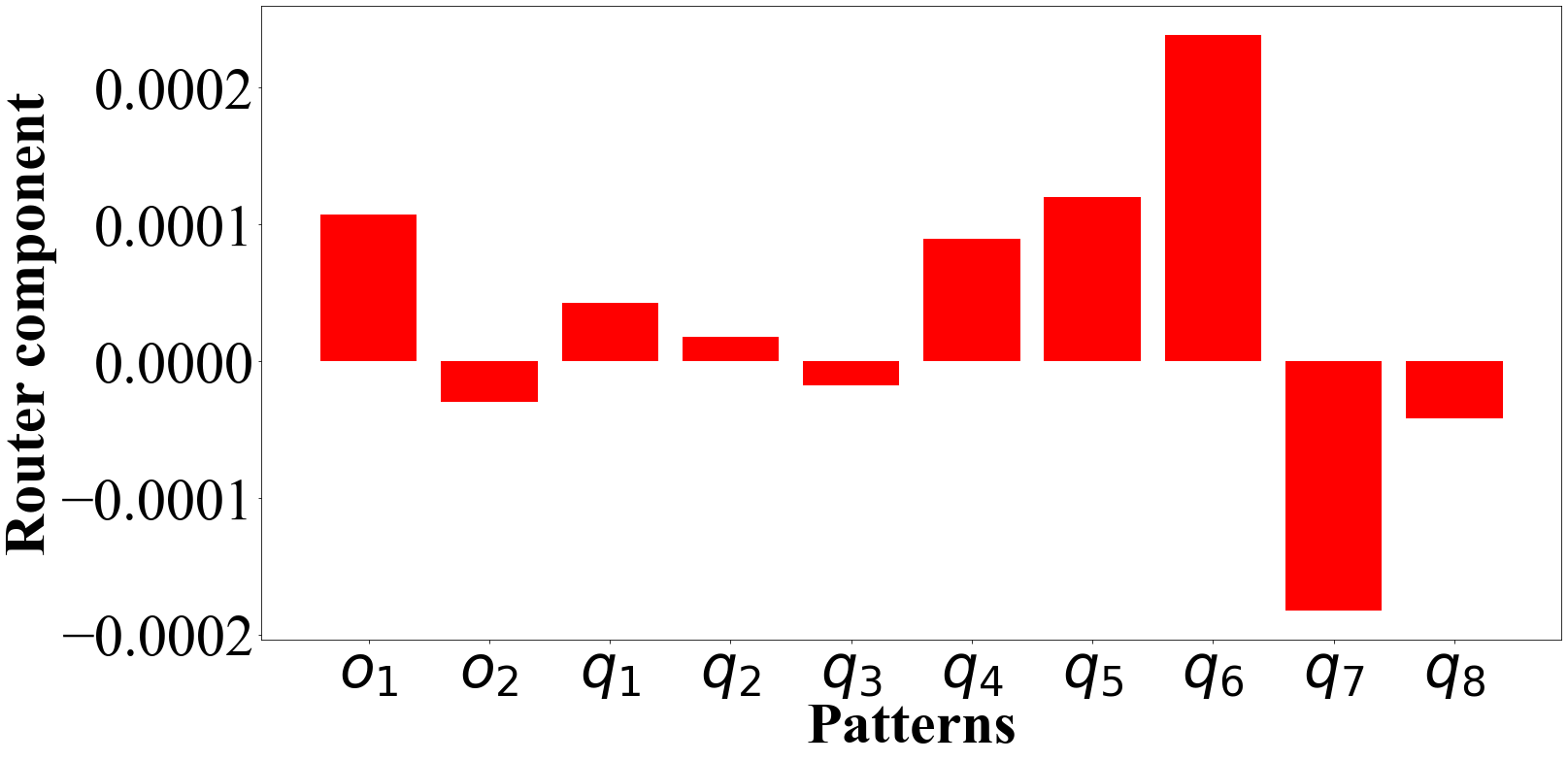}
        \caption{Router 1}
        \label{router_component_small_1}
    \end{subfigure}
    \hfill
    \begin{subfigure}[(d)]{0.48\linewidth}
        \centering
        \includegraphics[width=\linewidth]{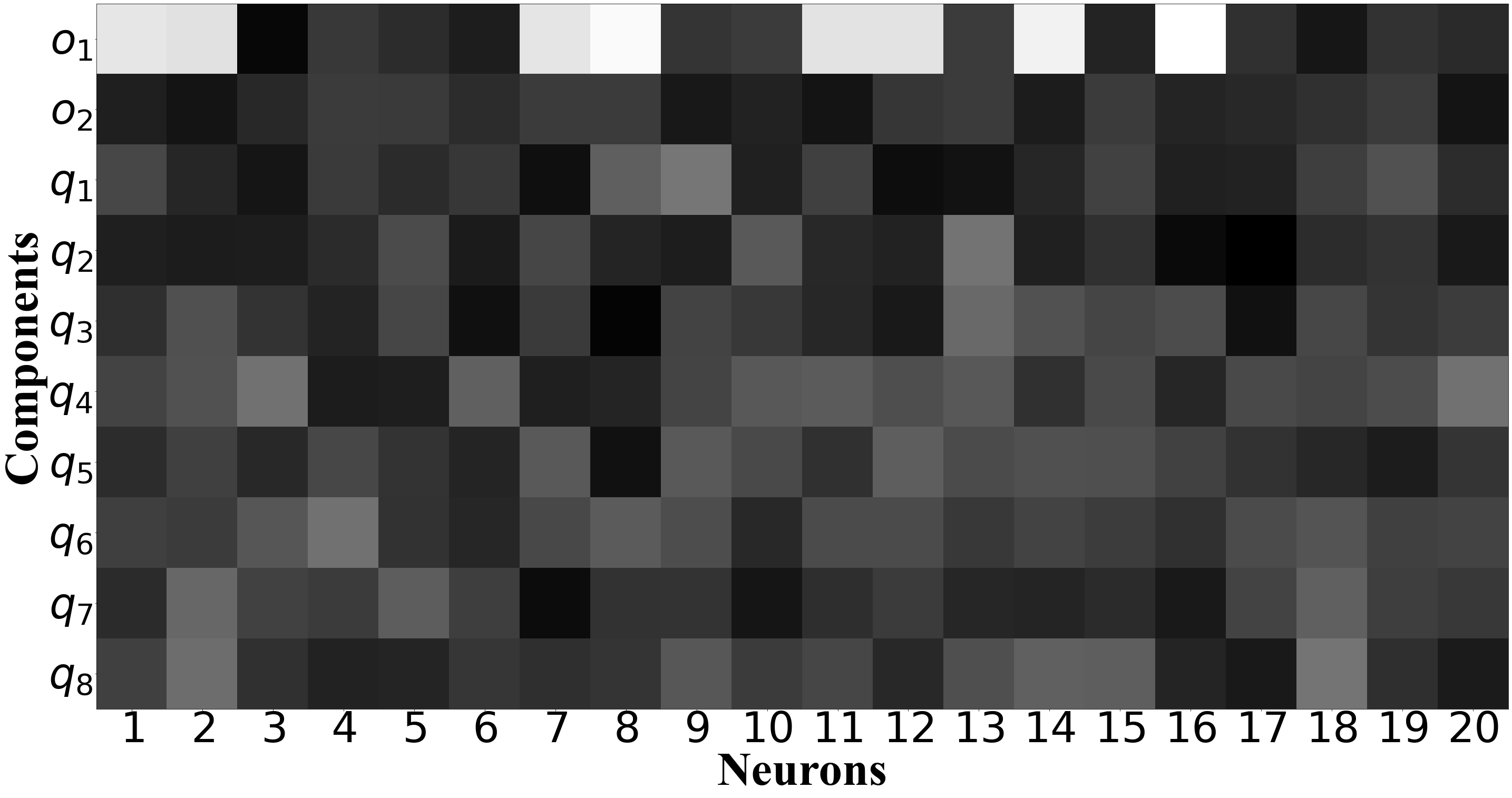}
        \caption{Expert 5}
        \label{neuron_component_large_1}
    \end{subfigure}
    \hfill
    \begin{subfigure}[(e)]{0.48\linewidth}
        \centering
        \includegraphics[width=\linewidth]{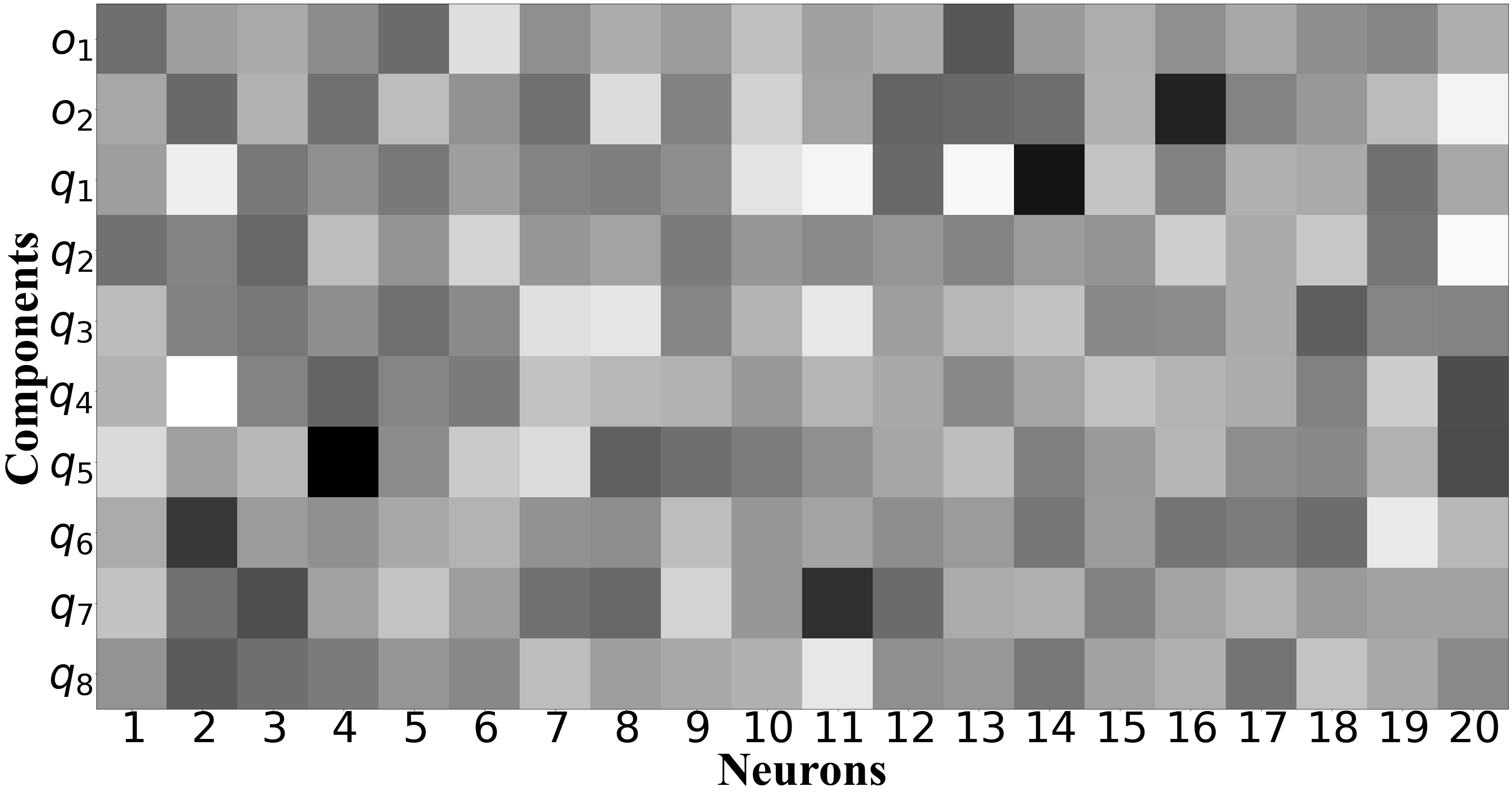}
        \caption{Expert 1}
        \label{neuron_component_small_1}
    \end{subfigure}
    \caption{
    (a) The norm of the post-training router weights, (b)(c) Projections of router weights to different directions, 
    (d)(e) Projections of neuron weights to different directions (larger pixel intensity represents larger component of the router weights).} 
    \vspace{-3mm}
\end{figure}

\begin{figure*}[ht]
    \begin{subfigure}[(a)]{0.32\linewidth}
        \centering
        \includegraphics[width=\linewidth]{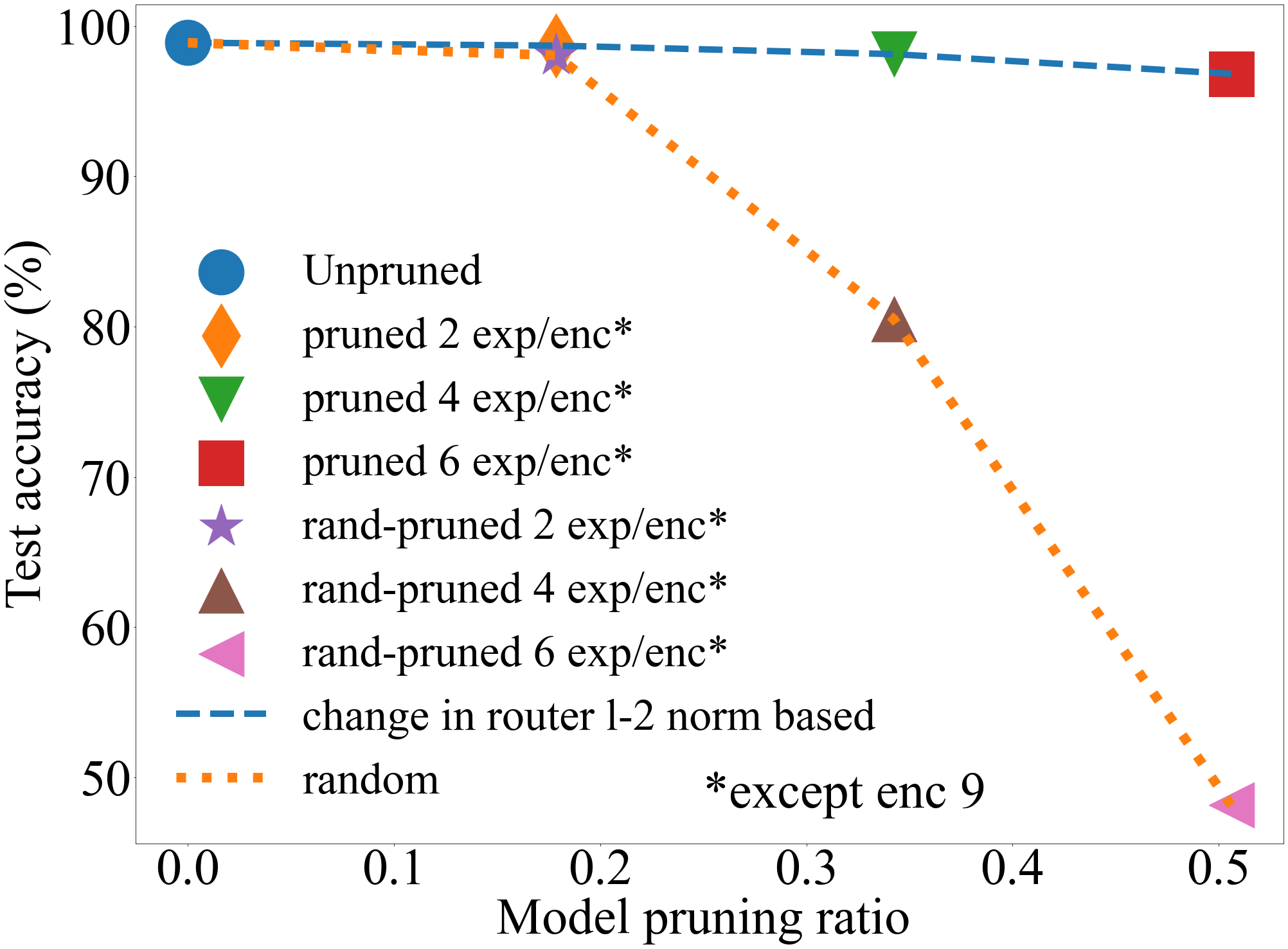}
        \caption{}
        \label{cifar_10_random}
    \end{subfigure}
    \hfill
    \begin{subfigure}[(b)]{0.32\linewidth}
        \centering
        \includegraphics[width=\linewidth]{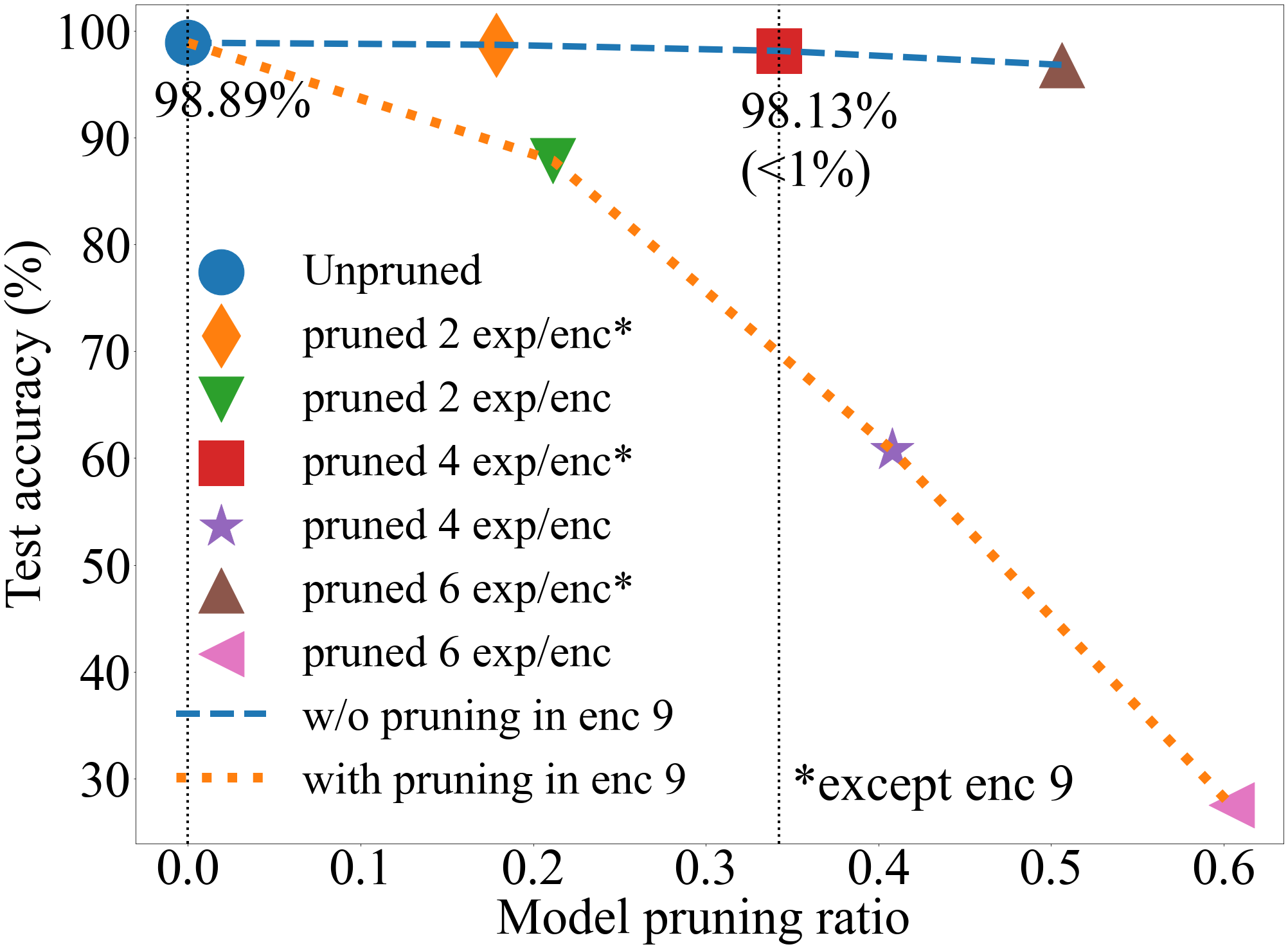}
        \caption{}
        \label{cifar_10_only_pruned}
    \end{subfigure}
    \hfill
    \begin{subfigure}[(c)]{0.32\linewidth}
        \centering
        \includegraphics[width=\linewidth]{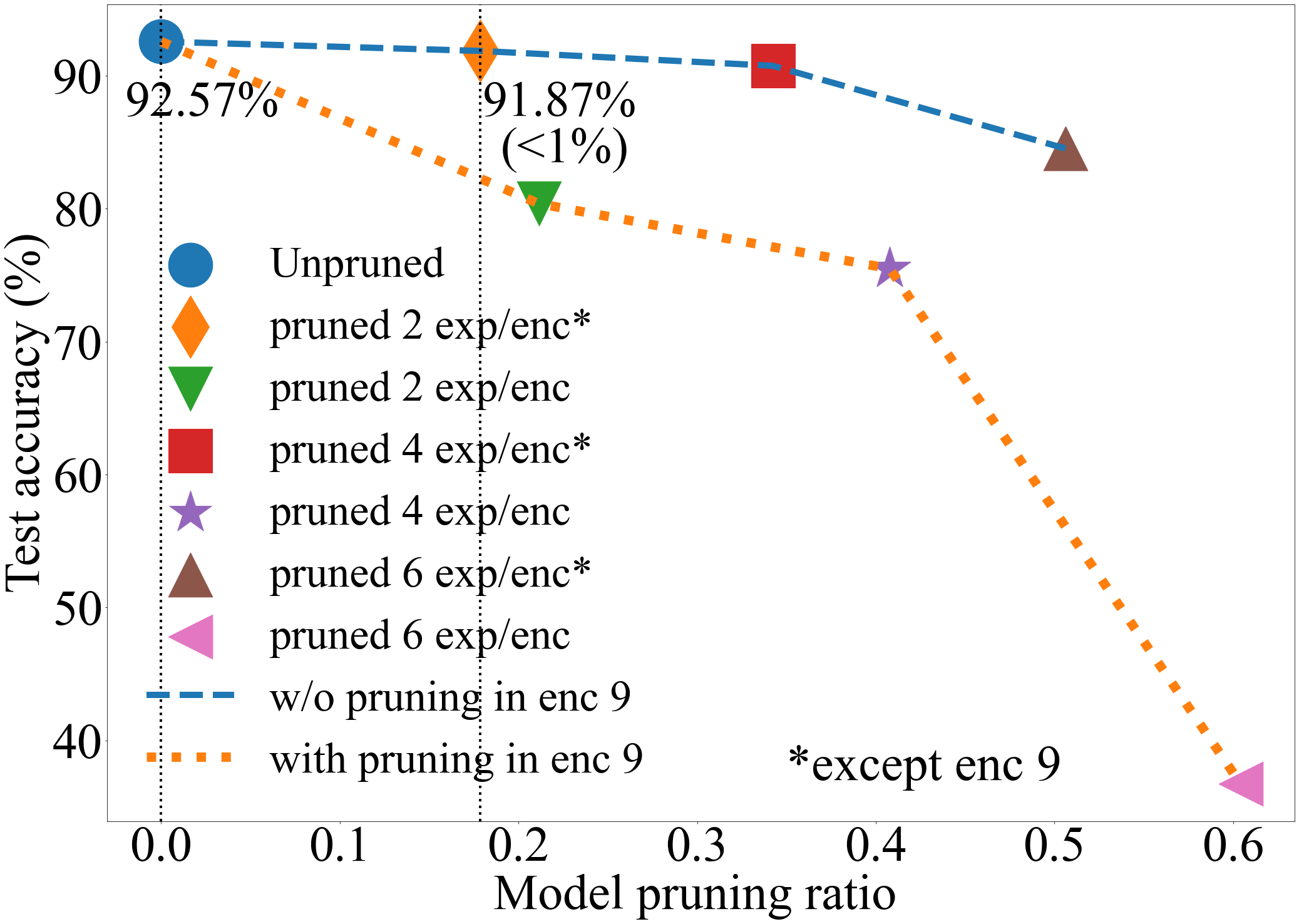}
        \caption{}
        \label{cifar_100_only_pruned}
    \end{subfigure}
    \hfill
    \begin{subfigure}[(d)]{0.32\linewidth}
        \centering
        \includegraphics[width=\linewidth]{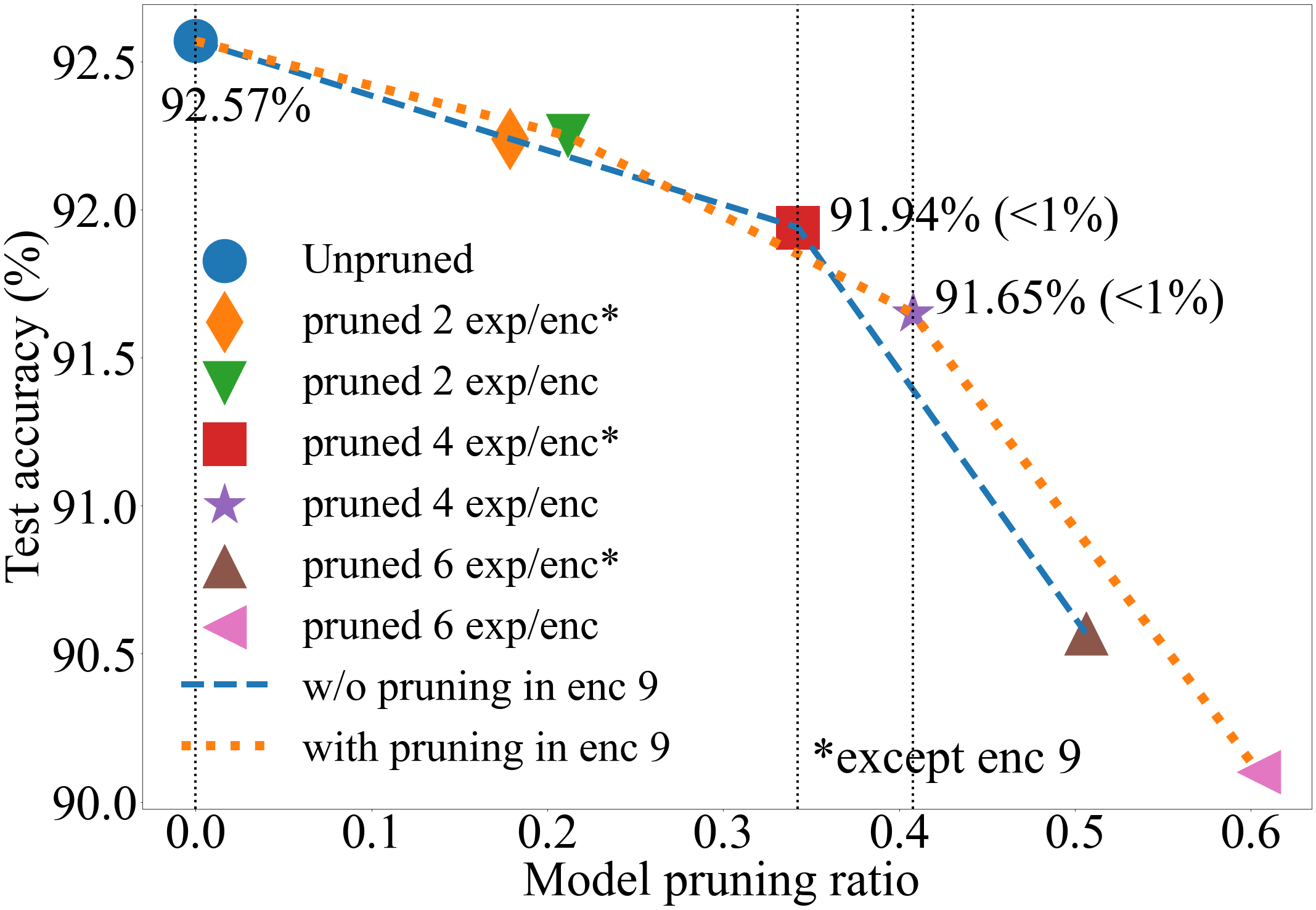}
        \caption{}
        \label{cifar_100_pruned_finetuned}
    \end{subfigure}
    \hfill
    \begin{subfigure}[(e)]{0.32\linewidth}
        \centering
        \includegraphics[width=\linewidth]{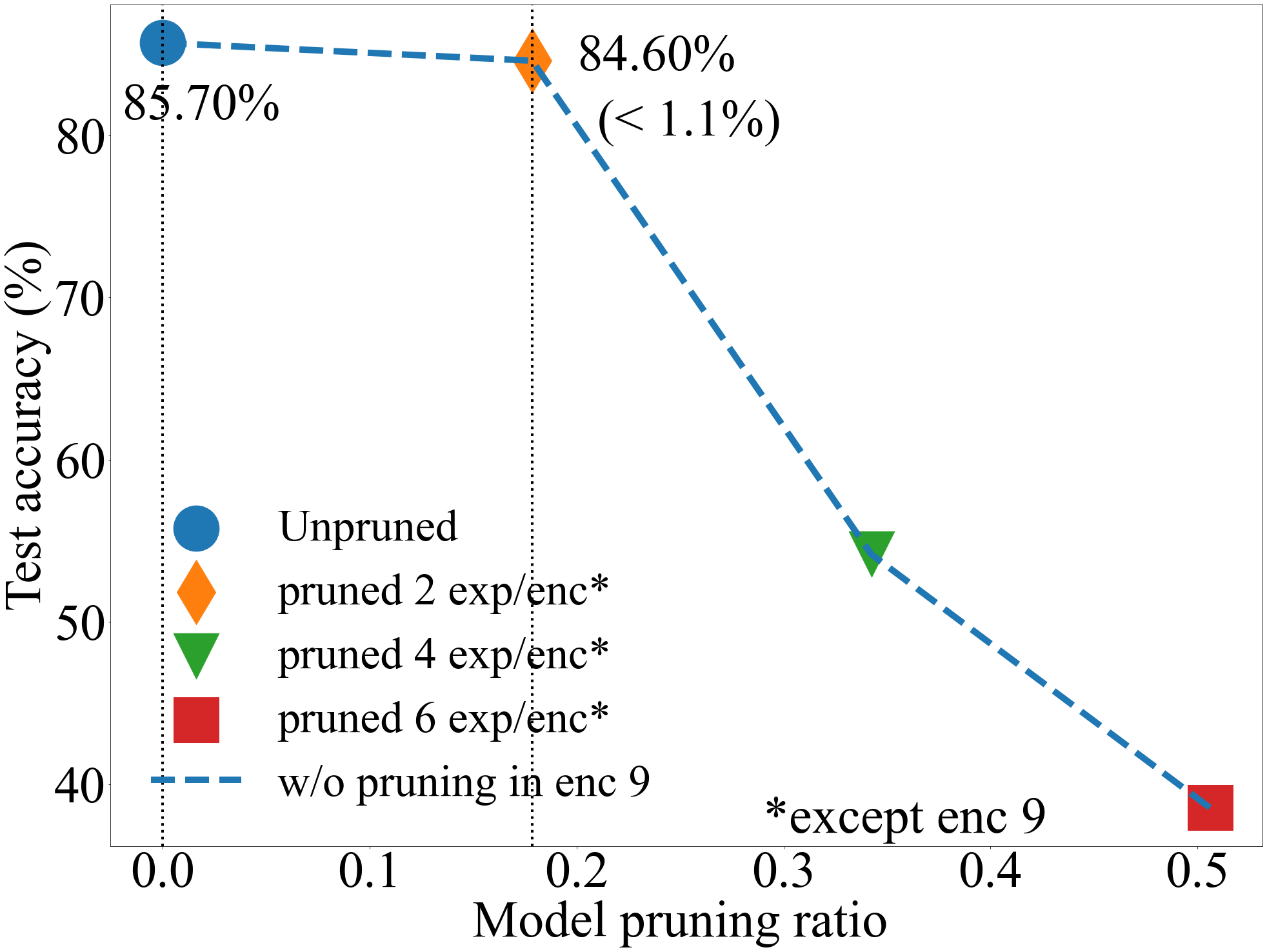}
        \caption{}
        \label{imagenet_only_pruned}
    \end{subfigure}
    \hfill
    \begin{subfigure}[(f)]{0.32\linewidth}
        \centering
        \includegraphics[width=\linewidth]{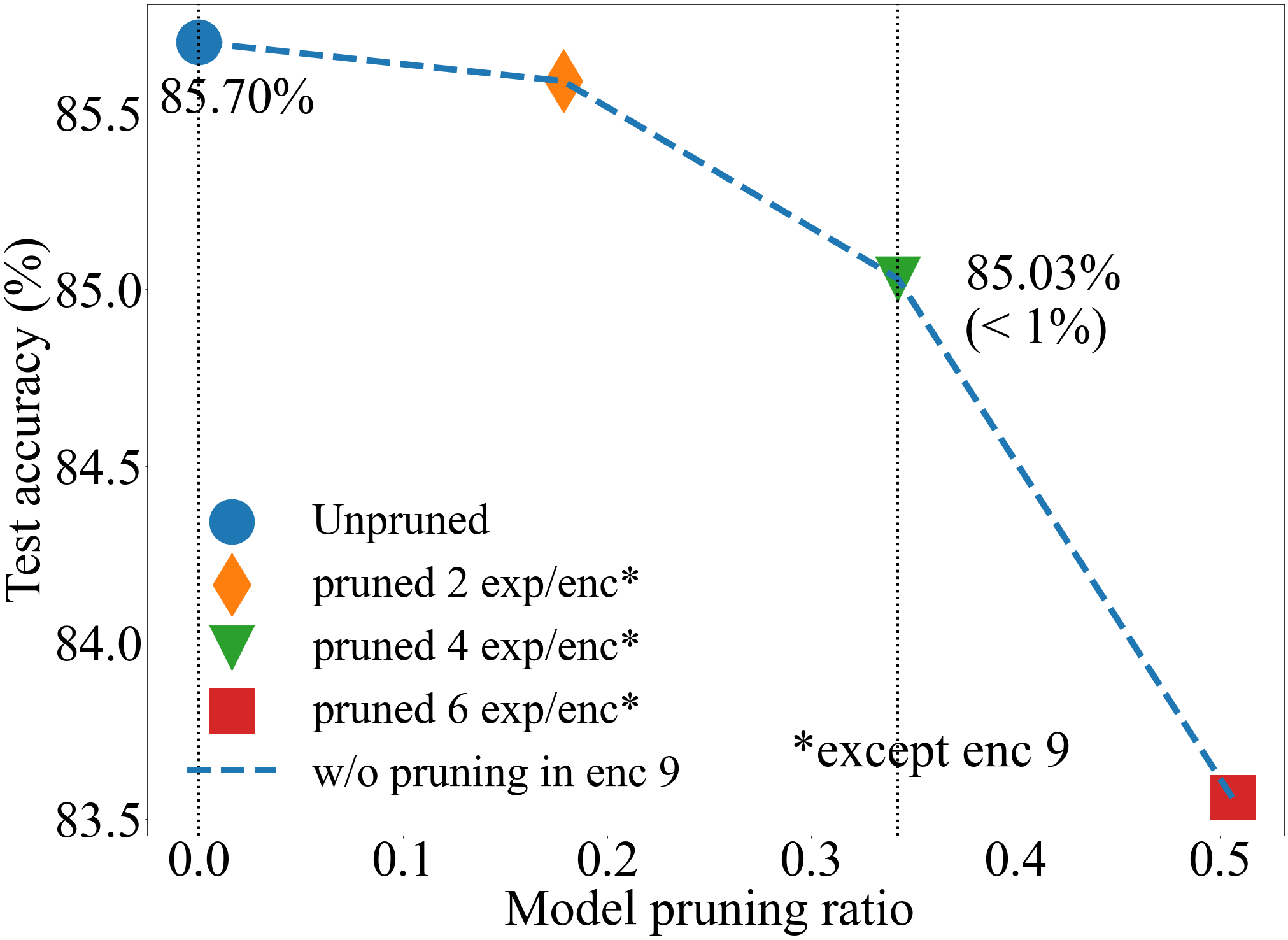}
        \caption{}
        \label{imagenet_pruned_fine_tuned}
    \end{subfigure}
    \caption{Generalization performance of the pruned \textbf{V-MoE} models: (a) Comparison with random pruning on CIFAR-10, (b) On CIFAR-10 \textbf{w/o} post-pruning fine-tuning, (c) On CIFAR-100 \text{w/o} post-pruning fine-tuning, (d) On CIFAR-100 \textbf{with} post-pruning fine-tuning, (e) On ImageNet \textbf{w/o} post-pruning fine-tuning, (f) On ImageNet \textbf{with} post-pruning fine-tuning}
    \vspace{-3mm}
   
\end{figure*}
\subsection{Experiments on Synthetic Data}\label{synthetic}

We verify our theoretical findings on synthetic data for the analyzed model. The data are generated by following the description given in section \ref{setup}. We selected $d=200$ and $n=100$ to generate the data. We select 20 experts in the MoE layer (i.e., $k=20$), each selecting five tokens (i.e., $l=5$). The first ten experts are positively connected to the output, while the last ten are negatively connected. We initialize the routers and the neurons of the experts randomly following zero-mean Gaussian distribution with very small variance ($1\times 10^{-8}$ and $1\times 10^{-4}$, respectively) and train using SGD up to zero training error.

We present the norms of the post-training routerweights for the 20 routers in Figure \ref{router_norm}. One can see that in both groups, a few routers have significantly larger weight norms  than others, e.g., experts 4 and 5 in the positive group, experts 11 and 20 in negative group.  Figure \ref{router_component_large_1} and \ref{router_component_small_1} visualize the projection of router weights in different directions.  The weight of Router 5 has a significant  norm, and  
has a large component along the $o_1$ direction. In contrast, the weight of Router 1 has a small  norm, and it does not have a significant component along $o_1$. Figure \ref{neuron_component_large_1}  and \ref{neuron_component_small_1} show the projection of the neuron weights of experts 5 and 1 in these directions. One can see that a constant fraction of neurons in expert 5 have  a large component along $o_1$, while neurons in expert 1 do not.  We have similar results for the negative group, see   section \ref{more_synthetic_exp} in the appendix.  These results are consistent with Lemmas \ref{lemma_1}  and \ref{lemma_2}. 


\subsection{Experiments on State-of-the-art Vision MoE Models}
We implement the proposed \textit{change in router's $l_2$ norm} based expert pruning method in state-of-the-art vision MoE models fine-tuned on several benchmark datasets such as CIFAR-10, CIFAR-100 \citep{krizhevsky2009learning} and ImageNet \citep{russakovsky2015imagenet}. More specifically, we implement the method to prune the 
released fine-tuned sparse vision MoE (V-MoE) by \citet{riquelme2021scaling}. Moreover, we present results on the released fine-tuned efficient ensembles of experts in MoE ($\text{E}^3$-MoE) by \citet{allingham2022sparse}.

\textbf{The V-MoE model}. The model contain 12 transformer encoders (encoders 0, 1, ..., and 11). Every odd encoder (i.e., encoders 1, 3, 5, 7, 9, and 11) is the MoE encoder. Each MoE encoder contains $8$ experts (i.e., $k=8$).

\textbf{The $\text{E}^3$-MoE model}. The model contain 8 transformer encoders (i.e., encoders 0, 1, ..., and 7), where the last two encoders contain the ensemble of MoEs (i.e., encoders 5 and 7). There are 2 MoE ensembles in the model and each MoE ensembles contains 4 experts per MoE encoder (i.e., $k=4$).

Both of these models implement the \textit{token-choice} routing with $l=2$ and $l=1$, respectively. More details can be found in section \ref{m} of appendix.

\textbf{Pruning and post-pruning fine-tuning details}. We use both the pre-trained and the fine-tuned versions of the released model to calculate the change in the router's norm for each expert. We apply our pruning method over each MoE layer separately in V-MoE model while applying the method separately over the experts in each ensemble group of each MoE encoder in the $\text{E}^3$-MoE model. We implement the pruned model in parallel into two NVIDIA RTX A5000 GPUs for inference and the post-pruning fine-tuning. For post-pruning fine-tuning, as the model size is large, we divide the batch size into half from the original case for CIFAR-10 and CIFAR-100. However, the number of steps is the same. For the same reason, we divide the original batch-size by 32 folds so as the learning rate for ImageNet and hence increase the post-training fine-tuning steps by 32 times of the original. Rest of the hyperparameters are same as in the original fine-tuning process described by the authors.

\subsubsection{Results}

\textbf{Generalization performance}. Here, we present the generalization performance of the pruned models in terms of the \textit{model pruning ratio}, which is defined as the ratio of the number of parameters pruned to the total number of parameters of the unpruned model. Primarily, we verify the effectiveness of our method compared to the random pruning baseline. As shown in Figure \ref{cifar_10_random} for CIFAR-10, our method can effectively select the task-specific experts as it retains the original performance while the random pruning of experts falls sharply after a small pruning ratio. Figure \ref{cifar_10_only_pruned} presents generalization results on CIFAR-10 without post-pruning fine-tuning. As we can see, the method can prune 50\% of the experts (35\% of the whole model) while maintaining accuracy within 1\% of the unpruned model when we do not prune from the penultimate MoE layer. However, as described in \citet{riquelme2021scaling}, due to the extra sensitivity of the penultimate MoE layer, pruning in the layer is not feasible without post-pruning fine-tuning. However, as shown in Figure \ref{cifar_10_pruned_finetuned}, we can overcome the limitation by post-pruning fine-tuning and prune up to the maximum possible percentage (75\% of the experts, 60\% of the whole model)\footnote{pruning more than 6 experts out of 8 experts is not allowed for $l=2$.}.

We found similar results for pruning V-MoE on CIFAR-100 (50\% of the experts; 41\% of the whole model is pruned) and ImageNet (42\% of the experts; 35\% of the whole model is pruned) and $\text{E}^3$-MoE on CIFAR-10 and CIFAR-100 (75\% of the experts; 45\% of the whole model is pruned for both of the datasets). The results for V-MoE on CIFAR-100 and ImageNet are presented in Figure \ref{cifar_100_only_pruned}, Figure \ref{cifar_100_pruned_finetuned} and Figure \ref{imagenet_only_pruned}, Figure \ref{imagenet_pruned_fine_tuned}, respectively. The results for $\text{E}^3$-MoE are presented in section \ref{eee_results} of Appendix. As we can see, when the size of the downstream task goes up (1000 class classification in ImageNet compared to 100 class classification in CIFAR-100 and 10 class classification in CIFAR-10) the maximum allowable pruning ratio goes down as more experts are now important for the larger sized task.

\begin{figure*}[ht]
    \begin{subfigure}[(a)]{0.32\linewidth}
        \centering
        \includegraphics[width=\linewidth]{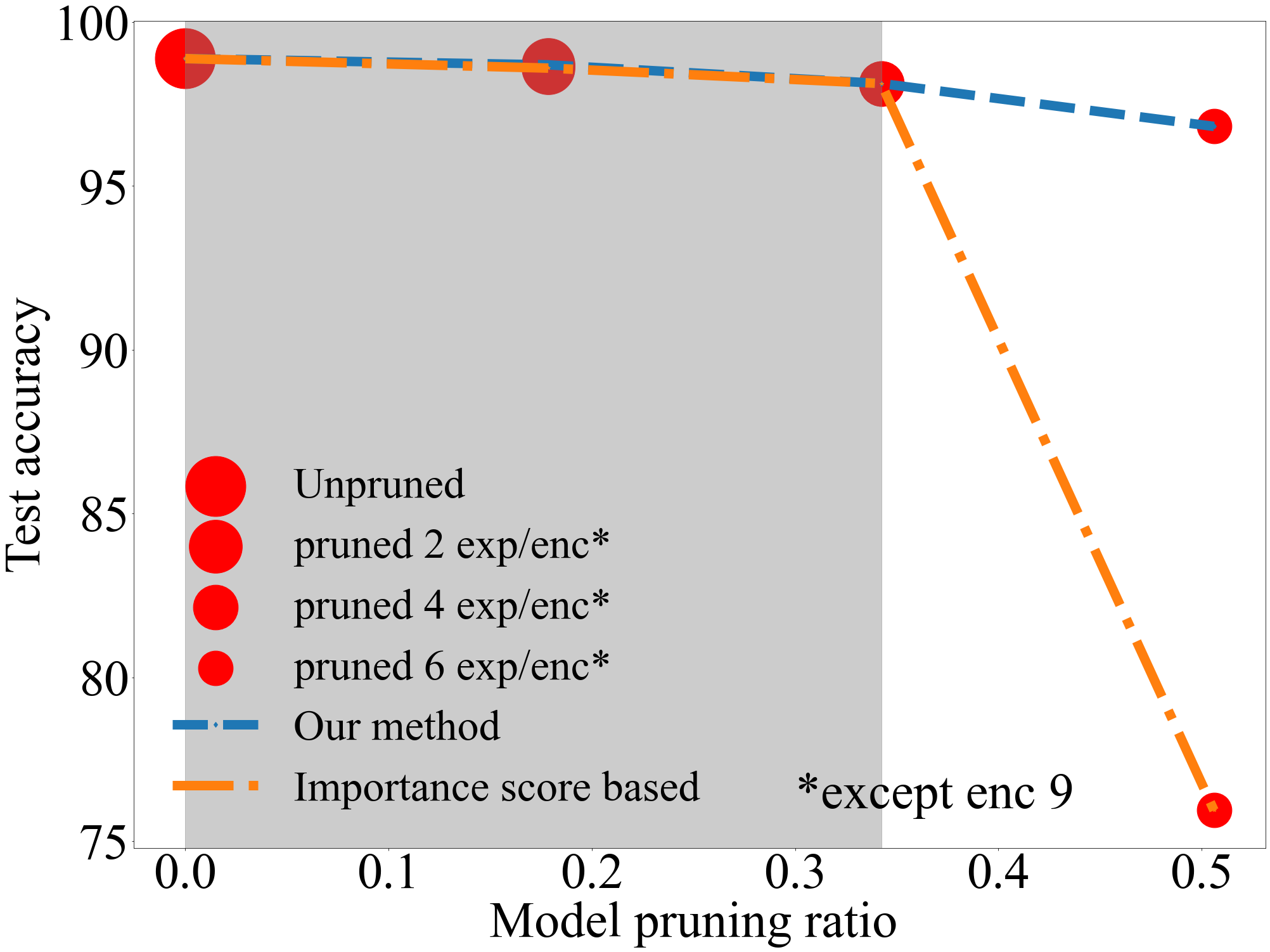}
        \caption{}
        \label{cifar_10_imp_score}
    \end{subfigure}
    \hfill
    \begin{subfigure}[(b)]{0.32\linewidth}
        \centering
        \includegraphics[width=\linewidth]{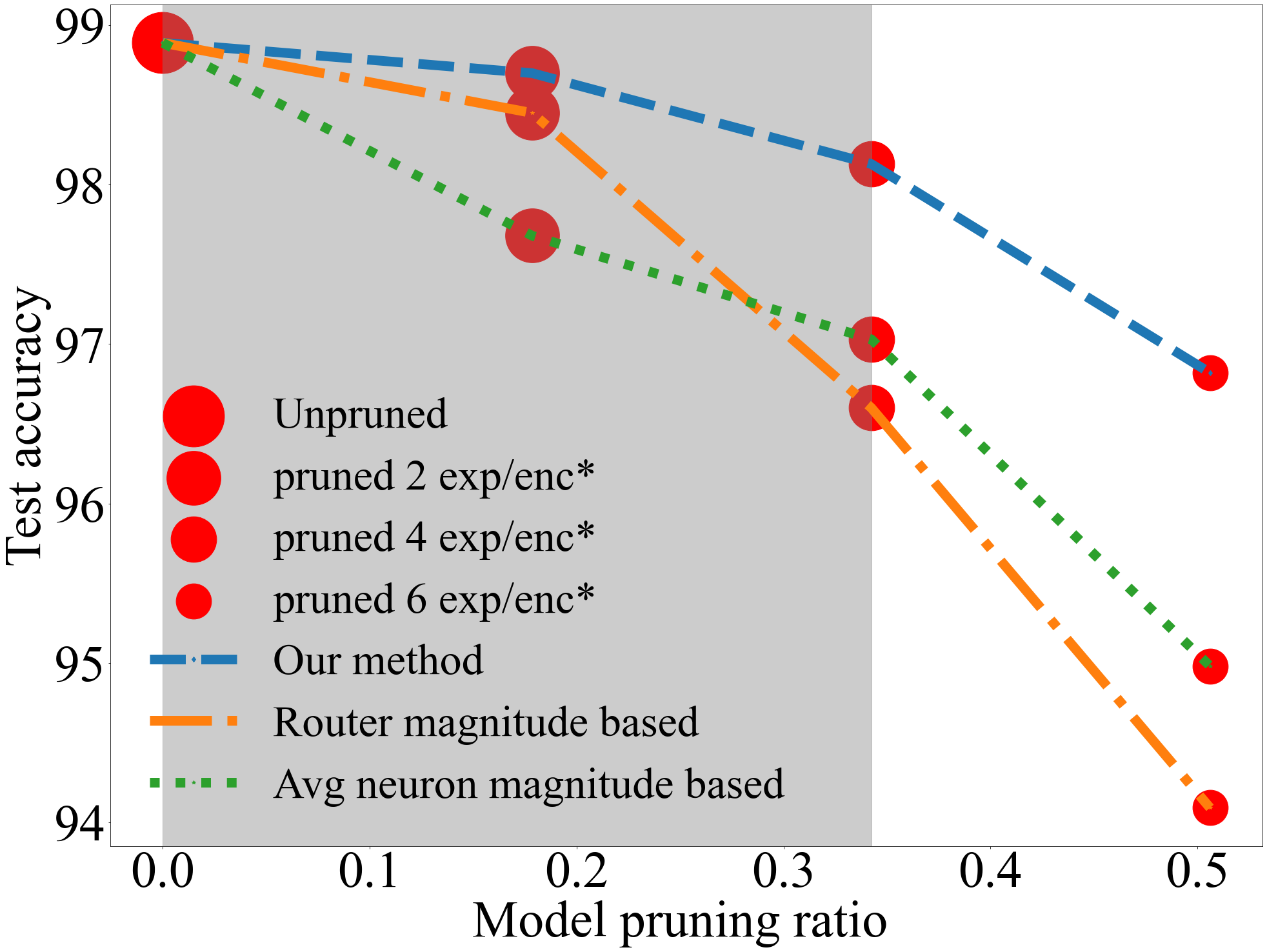}
        \caption{}
        \label{cifar_10_magn}
    \end{subfigure}
    \hfill
    \begin{subfigure}[(c)]{0.32\linewidth}
        \centering
        \includegraphics[width=\linewidth]{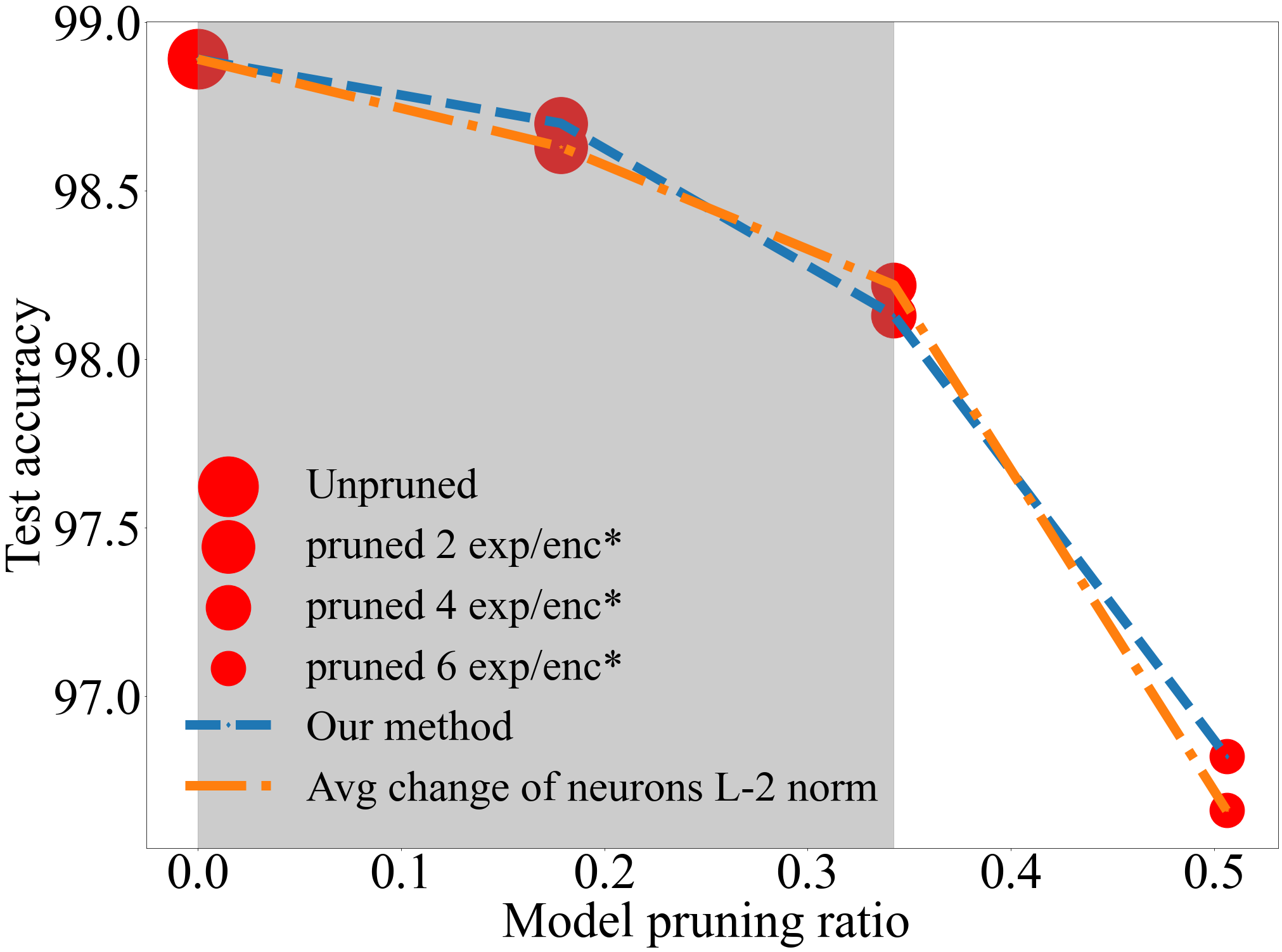}
        \caption{}
        \label{cifar_10_change_magn}
    \end{subfigure}
    \hfill
    \begin{subfigure}[(d)]{0.32\linewidth}
        \centering
        \includegraphics[width=\linewidth]{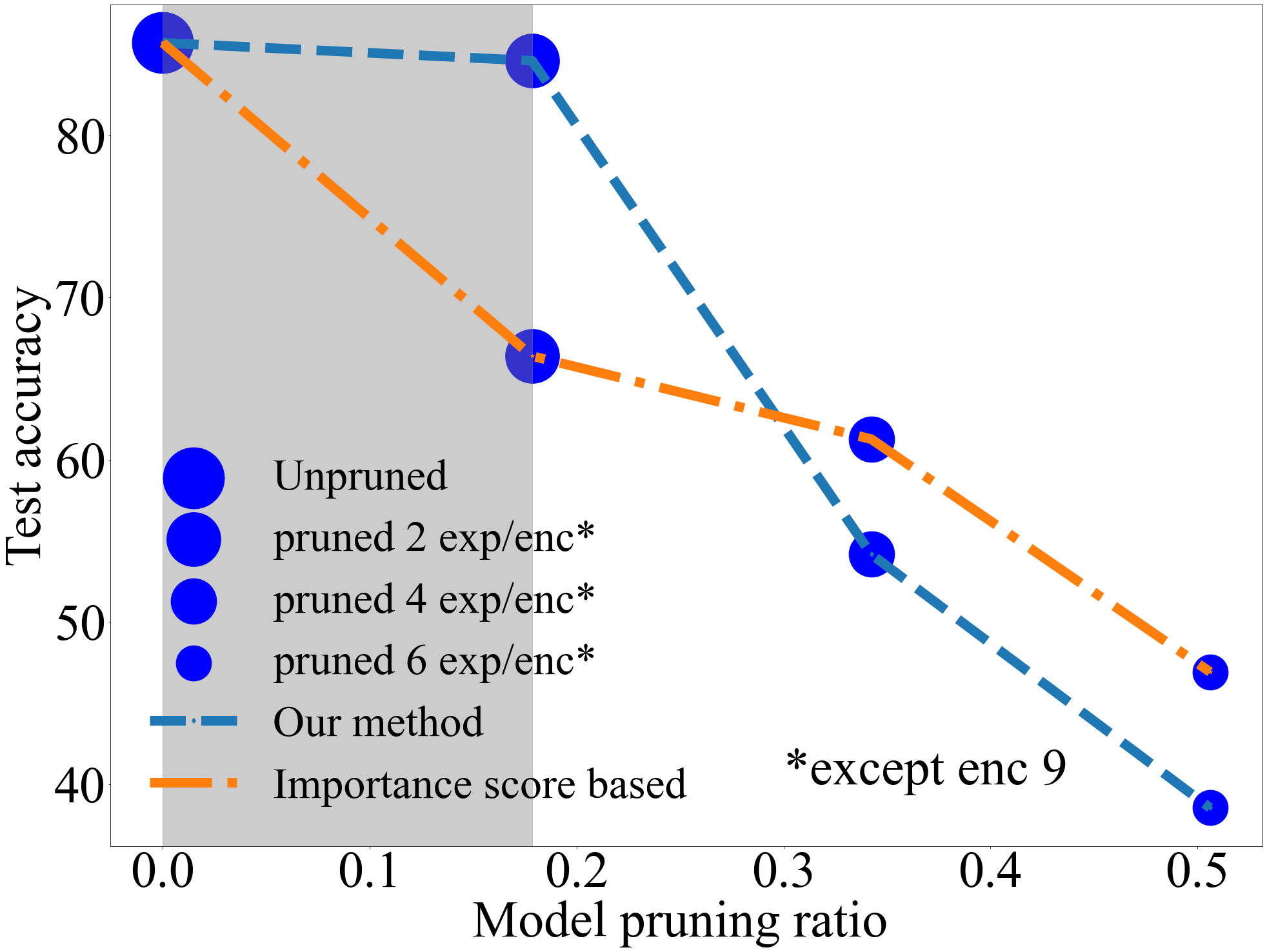}
        \caption{}
        \label{imagenet_imp_score}
    \end{subfigure}
    \hfill
    \begin{subfigure}[(e)]{0.32\linewidth}
        \centering
        \includegraphics[width=\linewidth]{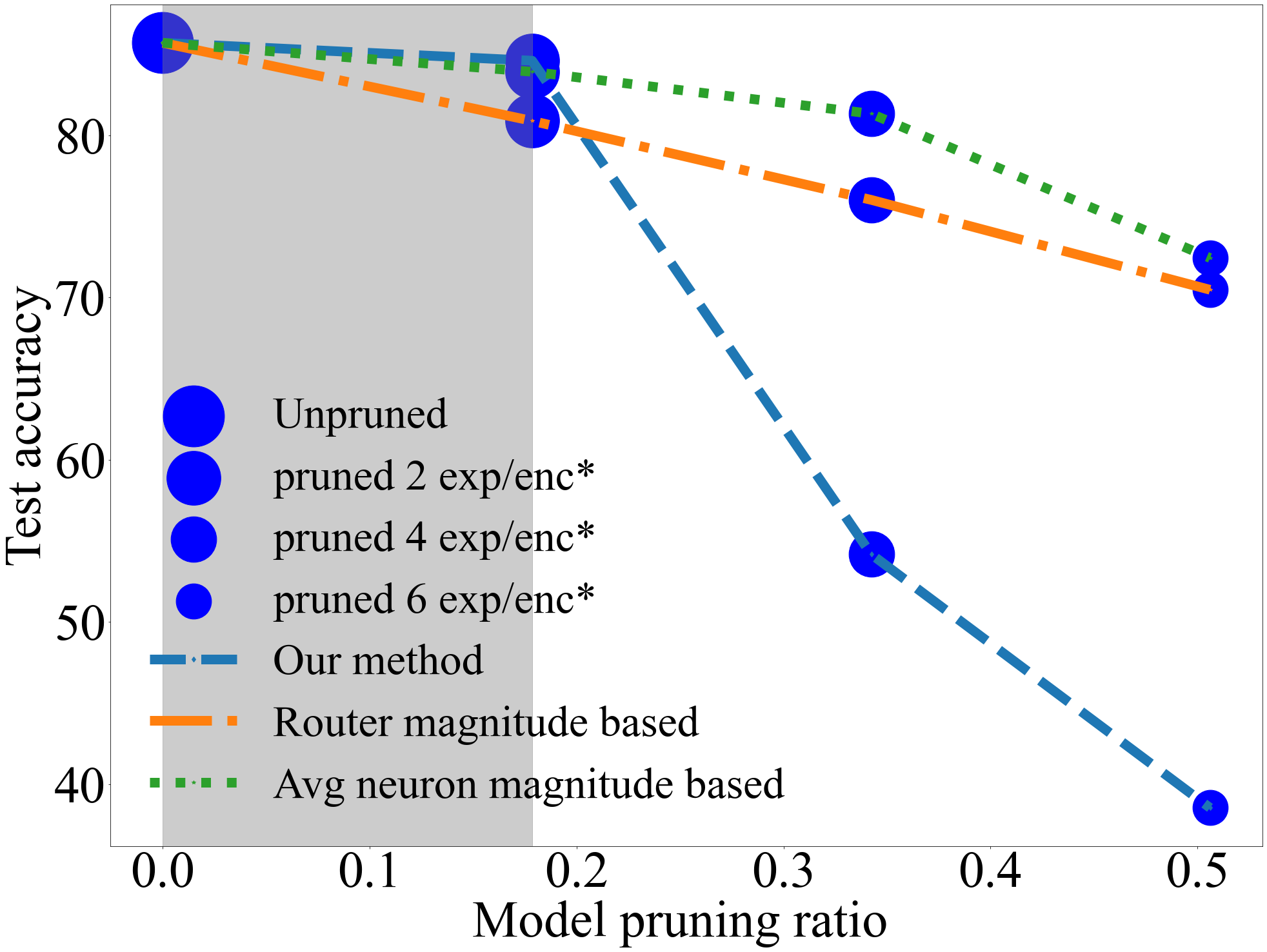}
        \caption{}
        \label{imagenet_magn}
    \end{subfigure}
    \hfill
    \begin{subfigure}[(f)]{0.32\linewidth}
        \centering
        \includegraphics[width=\linewidth]{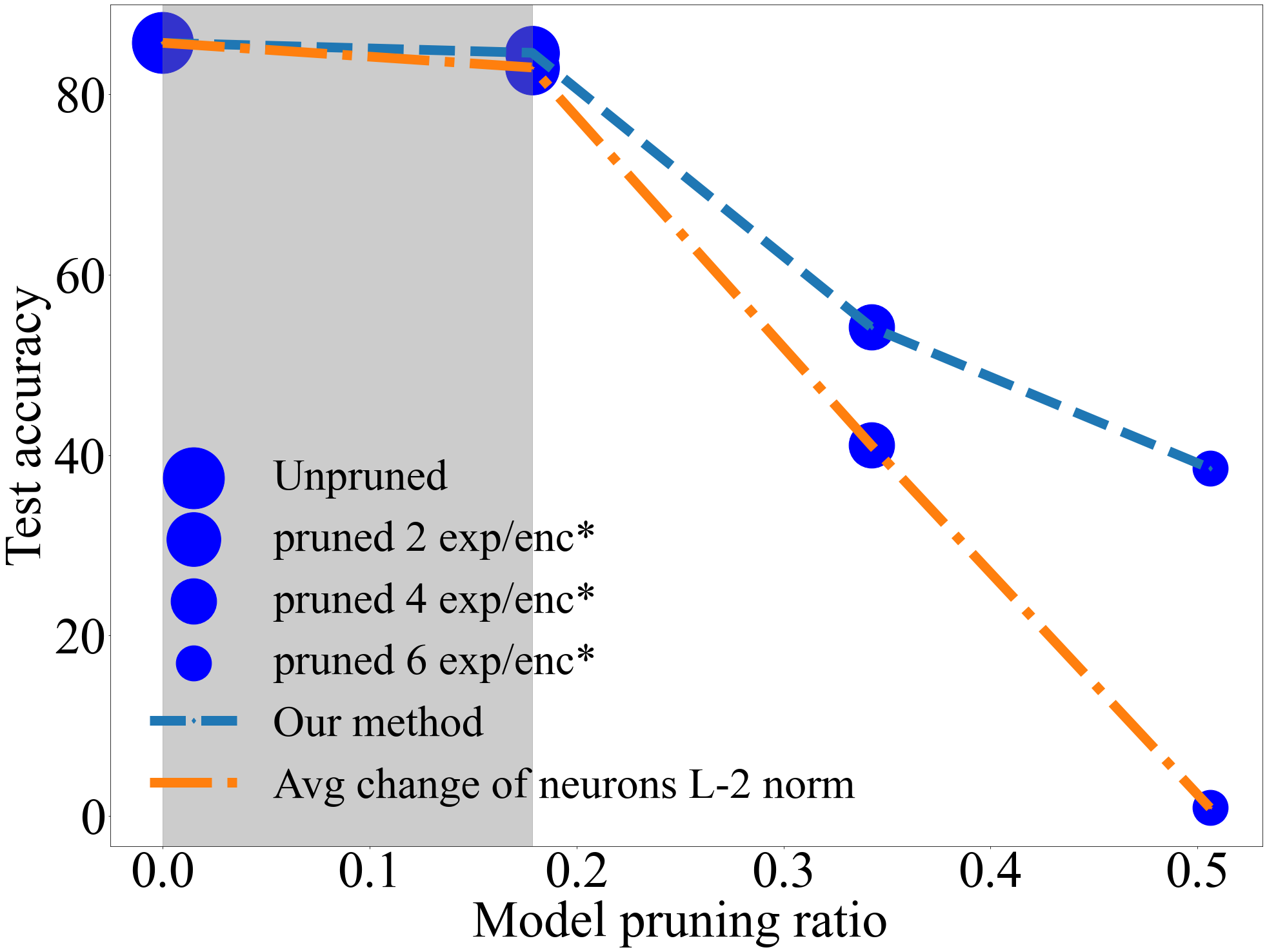}
        \caption{}
        \label{imagenet_change_magn}
    \end{subfigure}
    \caption{Comparison between different expert pruning methods: (a) vs. importance score on CIFAR-10, (b) vs.  absolute magnitude on CIFAR-10, (c)  vs. average change-in-neurons-magnitude on CIFAR-10, (d) vs. importance score on ImageNet, (e)  vs. absolute magnitude on ImageNet, (f) vs. average change-in-neurons-magnitude on ImageNet}
    \vspace{-3mm}
   
\end{figure*}
\textbf{Comparison with other methods}. We compare our proposed method of pruning experts based on the \textit{change in router's $l_2$ norm} with some other potential methods for pruning experts in MoE. Before presenting the results, we briefly describe the methods we consider for comparison as:\\
(I) \textbf{Importance score}. As mentioned in section \ref{intro}, there are only two recent works \citep{chen2022task,koishekenov-etal-2023-memory} that explored the expert pruning of MoE. Both of them used the \textit{fraction of the total number of tokens received} as the metric to order the experts according to their importance. Specifically, \citet{koishekenov-etal-2023-memory} defines the \textit{importance score} of any expert $s\in[k]$ over the validation set as follows:
    \begin{equation*}
        \text{Importance score}(s):=\text{top1}(s)\times \text{conf}(s)
    \end{equation*}
Here, $\text{top1}(s)$ is the fraction of the total tokens received as top-1 tokens by the expert $s$, and $\text{conf}(s)$ is the confidence value of the expert $s$, which is essentially the average gating value of the top-1 tokens received by the expert.\\
We also consider two magnitude-based pruning metrics:\\
(II) \textbf{Router magnitude}. The absolute $l_2$ norm of the router.\\
(III) \textbf{Average neuron magnitude}. The average value of the $l_2$ norm of the hidden neurons of the expert.\\
Similar to the \textit{change of router's $l_2$ norm} based method proposed in this paper, we also consider \\
(IV)\textbf{Average change of neuron magnitude}. The average value of the change in $l_2$ norm of the neurons of the experts.

Figure \ref{cifar_10_imp_score}, \ref{cifar_10_magn}, \ref{cifar_10_change_magn} and, Figure \ref{imagenet_imp_score}, \ref{imagenet_magn}, \ref{imagenet_change_magn} present the comparative results for pruning V-MoE on CIFAR-10 and ImageNet, respectively while the corresponding results on CIFAR-100 are presented in section \ref{baseline_cifar100} of Appendix. As we can see, the proposed method consistently provides an upper bound to all the other methods described above, at least within the range of acceptable model pruning ratio (model performance is within 1\% of the unpruned model, the gray area).

\begin{figure}[ht]
    \centering
    \includegraphics[width=0.65\linewidth]{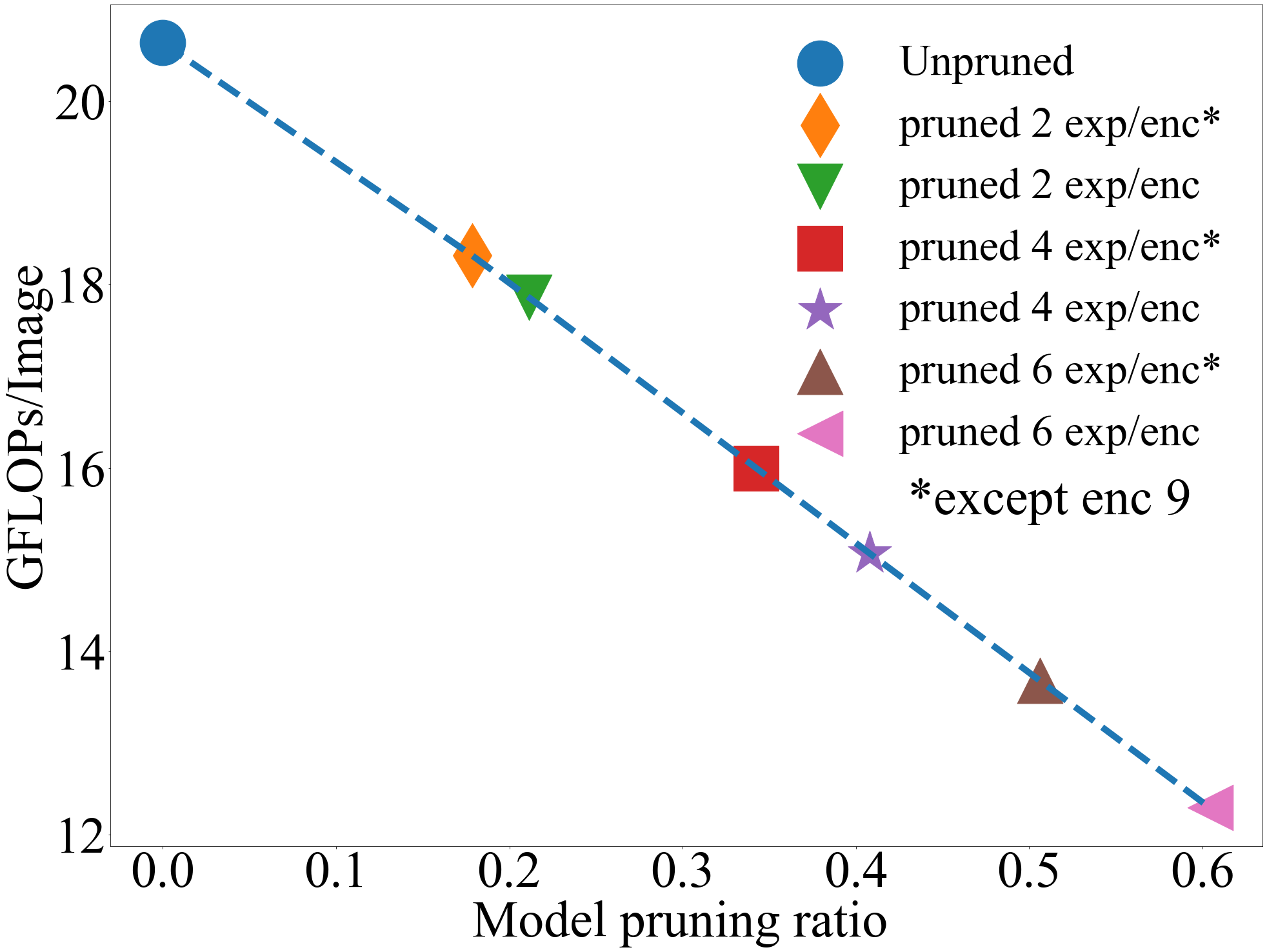}
    \vspace{-2mm}
    \caption{FLOPs per image in V-MoE on CIFAR-10}
    \label{cifar_10_flop}
    \vspace{-3mm}
\end{figure}

\textbf{Inference efficiency}. We measure the inference efficiency of the pruned model in terms of FLOPs per image and the inference time. Figure \ref{cifar_10_flop} shows the linear reduction of inference FLOPs per image with the increment of pruning ratio in pruning V-MoE on CIFAR-10. By combining the results of this figure with the results in Figure \ref{cifar_10_pruned_finetuned}, we can infer that the proposed pruning method can reduce 40\% of the inference FLOPs. We have similar results showing the linear reduction of inference time (see section \ref{inference_more_results} in Appendix for details). The linear reduction of FLOPs and time indicates the prompt implementation of the method without any loss of estimated efficiency (e.g., due to the overhead of the dedicated software kernels or hardware) as our results are generated from the computation on typical modern GPUs without any requirement of special software (e.g., dedicated CUDA kernels). We have similar results for pruning $\text{E}^3$-MoE on CIFAR-100 (linear reduction; 15\% reduction of FLOPs and time for the 1\% tolerance in accuracy) and pruning V-MoE on ImageNet (linear reduction; 14\% reduction of FLOPs and time for the 1\% tolerance in accuracy), see section \ref{inference_more_results} in appendix.

\section{Conclusion}

MoE allows faster pre-training of large deep models. However, the inference compute is at the same order as in its dense counterpart. In this paper, we theoretically and empirically investigated  pruning irrelevant and redundant experts from the fine-tuned MoE model. We show that pruning experts based on the change in the router's norm can provably maintain the generalization accuracy. The results are verified in the state-of-the-art vision MoE models. Future works include co-implementing the method with other network compression techniques such as neurons pruning, fine-grained structured sparsity, low rank factorization and quantization.

\clearpage
\newpage 
\section*{Acknowledgements}
This work was supported by IBM through the IBM-Rensselaer Future of Computing Research Collaboration. We thank all anonymous reviewers for their valuable comments and suggestions.
\section*{Impact Statement}
This paper presents work whose goal is to advance the field of Machine Learning. There are many potential societal consequences of our work, none which we feel must be specifically highlighted here.
\bibliography{Reference/reference}

\begin{thebibliography}{68}
\providecommand{\natexlab}[1]{#1}
\providecommand{\url}[1]{\texttt{#1}}
\expandafter\ifx\csname urlstyle\endcsname\relax
  \providecommand{\doi}[1]{doi: #1}\else
  \providecommand{\doi}{doi: \begingroup \urlstyle{rm}\Url}\fi

\bibitem[Allen-Zhu \& Li(2022)Allen-Zhu and Li]{allen2022feature}
Allen-Zhu, Z. and Li, Y.
\newblock Feature purification: How adversarial training performs robust deep learning.
\newblock In \emph{2021 IEEE 62nd Annual Symposium on Foundations of Computer Science (FOCS)}, pp.\  977--988. IEEE, 2022.

\bibitem[Allen-Zhu \& Li(2023)Allen-Zhu and Li]{allen-zhu2023towards}
Allen-Zhu, Z. and Li, Y.
\newblock Towards understanding ensemble, knowledge distillation and self-distillation in deep learning.
\newblock In \emph{The Eleventh International Conference on Learning Representations}, 2023.
\newblock URL \url{https://openreview.net/forum?id=Uuf2q9TfXGA}.

\bibitem[Allen-Zhu et~al.(2019{\natexlab{a}})Allen-Zhu, Li, and Liang]{allen2019learning}
Allen-Zhu, Z., Li, Y., and Liang, Y.
\newblock Learning and generalization in overparameterized neural networks, going beyond two layers.
\newblock \emph{Advances in neural information processing systems}, 32, 2019{\natexlab{a}}.

\bibitem[Allen-Zhu et~al.(2019{\natexlab{b}})Allen-Zhu, Li, and Song]{allen2019convergence}
Allen-Zhu, Z., Li, Y., and Song, Z.
\newblock A convergence theory for deep learning via over-parameterization.
\newblock In \emph{International Conference on Machine Learning}, pp.\  242--252. PMLR, 2019{\natexlab{b}}.

\bibitem[Allingham et~al.(2022)Allingham, Wenzel, Mariet, Mustafa, Puigcerver, Houlsby, Jerfel, Fortuin, Lakshminarayanan, Snoek, Tran, Ruiz, and Jenatton]{allingham2022sparse}
Allingham, J.~U., Wenzel, F., Mariet, Z.~E., Mustafa, B., Puigcerver, J., Houlsby, N., Jerfel, G., Fortuin, V., Lakshminarayanan, B., Snoek, J., Tran, D., Ruiz, C.~R., and Jenatton, R.
\newblock Sparse moes meet efficient ensembles.
\newblock \emph{Transactions on Machine Learning Research}, 2022.
\newblock ISSN 2835-8856.
\newblock URL \url{https://openreview.net/forum?id=i0ZM36d2qU}.
\newblock Expert Certification, Expert Certification.

\bibitem[Arora et~al.(2019)Arora, Du, Hu, Li, and Wang]{arora2019fine}
Arora, S., Du, S., Hu, W., Li, Z., and Wang, R.
\newblock Fine-grained analysis of optimization and generalization for overparameterized two-layer neural networks.
\newblock In \emph{International Conference on Machine Learning}, pp.\  322--332. PMLR, 2019.

\bibitem[Brutzkus \& Globerson(2021)Brutzkus and Globerson]{brutzkus2021optimization}
Brutzkus, A. and Globerson, A.
\newblock An optimization and generalization analysis for max-pooling networks.
\newblock In \emph{Uncertainty in Artificial Intelligence}, pp.\  1650--1660. PMLR, 2021.

\bibitem[Brutzkus et~al.(2018)Brutzkus, Globerson, Malach, and Shalev-Shwartz]{brutzkus2018sgd}
Brutzkus, A., Globerson, A., Malach, E., and Shalev-Shwartz, S.
\newblock {SGD} learns over-parameterized networks that provably generalize on linearly separable data.
\newblock In \emph{International Conference on Learning Representations}, 2018.

\bibitem[Chen et~al.(2020)Chen, Frankle, Chang, Liu, Zhang, Wang, and Carbin]{chen2020lottery}
Chen, T., Frankle, J., Chang, S., Liu, S., Zhang, Y., Wang, Z., and Carbin, M.
\newblock The lottery ticket hypothesis for pre-trained bert networks.
\newblock \emph{Advances in neural information processing systems}, 33:\penalty0 15834--15846, 2020.

\bibitem[Chen et~al.(2022{\natexlab{a}})Chen, Huang, Xie, Jiao, Jiang, Zhou, Li, and Wei]{chen2022task}
Chen, T., Huang, S., Xie, Y., Jiao, B., Jiang, D., Zhou, H., Li, J., and Wei, F.
\newblock Task-specific expert pruning for sparse mixture-of-experts.
\newblock \emph{arXiv preprint arXiv:2206.00277}, 2022{\natexlab{a}}.

\bibitem[Chen et~al.(2022{\natexlab{b}})Chen, Wang, Changpinyo, Piergiovanni, Padlewski, Salz, Goodman, Grycner, Mustafa, Beyer, et~al.]{chen2022pali}
Chen, X., Wang, X., Changpinyo, S., Piergiovanni, A., Padlewski, P., Salz, D., Goodman, S., Grycner, A., Mustafa, B., Beyer, L., et~al.
\newblock Pali: A jointly-scaled multilingual language-image model.
\newblock In \emph{The Eleventh International Conference on Learning Representations}, 2022{\natexlab{b}}.

\bibitem[Chen et~al.(2022{\natexlab{c}})Chen, Deng, Wu, Gu, and Li]{chen2022towards}
Chen, Z., Deng, Y., Wu, Y., Gu, Q., and Li, Y.
\newblock Towards understanding the mixture-of-experts layer in deep learning.
\newblock In Oh, A.~H., Agarwal, A., Belgrave, D., and Cho, K. (eds.), \emph{Advances in Neural Information Processing Systems}, 2022{\natexlab{c}}.
\newblock URL \url{https://openreview.net/forum?id=MaYzugDmQV}.

\bibitem[Chowdhery et~al.(2023)Chowdhery, Narang, Devlin, Bosma, Mishra, Roberts, Barham, Chung, Sutton, Gehrmann, et~al.]{chowdhery2023palm}
Chowdhery, A., Narang, S., Devlin, J., Bosma, M., Mishra, G., Roberts, A., Barham, P., Chung, H.~W., Sutton, C., Gehrmann, S., et~al.
\newblock Palm: Scaling language modeling with pathways.
\newblock \emph{Journal of Machine Learning Research}, 24\penalty0 (240):\penalty0 1--113, 2023.

\bibitem[Chowdhury et~al.(2023)Chowdhury, Zhang, Wang, Liu, and Chen]{pmlr-v202-chowdhury23a}
Chowdhury, M. N.~R., Zhang, S., Wang, M., Liu, S., and Chen, P.-Y.
\newblock Patch-level routing in mixture-of-experts is provably sample-efficient for convolutional neural networks.
\newblock In \emph{Proceedings of the 40th International Conference on Machine Learning}, volume 202 of \emph{Proceedings of Machine Learning Research}, pp.\  6074--6114. PMLR, 23--29 Jul 2023.
\newblock URL \url{https://proceedings.mlr.press/v202/chowdhury23a.html}.

\bibitem[Daniely \& Malach(2020)Daniely and Malach]{daniely2020learning}
Daniely, A. and Malach, E.
\newblock Learning parities with neural networks.
\newblock \emph{Advances in Neural Information Processing Systems}, 33:\penalty0 20356--20365, 2020.

\bibitem[Dehghani et~al.(2023)Dehghani, Djolonga, Mustafa, Padlewski, Heek, Gilmer, Steiner, Caron, Geirhos, Alabdulmohsin, et~al.]{dehghani2023scaling}
Dehghani, M., Djolonga, J., Mustafa, B., Padlewski, P., Heek, J., Gilmer, J., Steiner, A.~P., Caron, M., Geirhos, R., Alabdulmohsin, I., et~al.
\newblock Scaling vision transformers to 22 billion parameters.
\newblock In \emph{International Conference on Machine Learning}, pp.\  7480--7512. PMLR, 2023.

\bibitem[Devlin et~al.(2019)Devlin, Chang, Lee, and Toutanova]{devlin2019bert}
Devlin, J., Chang, M.-W., Lee, K., and Toutanova, K.
\newblock Bert: Pre-training of deep bidirectional transformers for language understanding.
\newblock In \emph{Proceedings of the 2019 Conference of the North American Chapter of the Association for Computational Linguistics: Human Language Technologies, Volume 1 (Long and Short Papers)}, pp.\  4171--4186, 2019.

\bibitem[Dosovitskiy et~al.(2020)Dosovitskiy, Beyer, Kolesnikov, Weissenborn, Zhai, Unterthiner, Dehghani, Minderer, Heigold, Gelly, et~al.]{dosovitskiy2020image}
Dosovitskiy, A., Beyer, L., Kolesnikov, A., Weissenborn, D., Zhai, X., Unterthiner, T., Dehghani, M., Minderer, M., Heigold, G., Gelly, S., et~al.
\newblock An image is worth 16x16 words: Transformers for image recognition at scale.
\newblock In \emph{International Conference on Learning Representations}, 2020.

\bibitem[Du et~al.(2022)Du, Huang, Dai, Tong, Lepikhin, Xu, Krikun, Zhou, Yu, Firat, et~al.]{du2022glam}
Du, N., Huang, Y., Dai, A.~M., Tong, S., Lepikhin, D., Xu, Y., Krikun, M., Zhou, Y., Yu, A.~W., Firat, O., et~al.
\newblock Glam: Efficient scaling of language models with mixture-of-experts.
\newblock In \emph{International Conference on Machine Learning}, pp.\  5547--5569. PMLR, 2022.

\bibitem[Du et~al.(2019)Du, Lee, Li, Wang, and Zhai]{du2019gradient}
Du, S., Lee, J., Li, H., Wang, L., and Zhai, X.
\newblock Gradient descent finds global minima of deep neural networks.
\newblock In \emph{International conference on machine learning}, pp.\  1675--1685. PMLR, 2019.

\bibitem[Fedus et~al.(2022)Fedus, Zoph, and Shazeer]{fedus2022switch}
Fedus, W., Zoph, B., and Shazeer, N.
\newblock Switch transformers: Scaling to trillion parameter models with simple and efficient sparsity.
\newblock \emph{Journal of Machine Learning Research}, 23\penalty0 (120):\penalty0 1--39, 2022.

\bibitem[Frankle \& Carbin(2018)Frankle and Carbin]{frankle2018lottery}
Frankle, J. and Carbin, M.
\newblock The lottery ticket hypothesis: Finding sparse, trainable neural networks.
\newblock In \emph{International Conference on Learning Representations}, 2018.

\bibitem[Frantar \& Alistarh(2023)Frantar and Alistarh]{pmlr-v202-frantar23a}
Frantar, E. and Alistarh, D.
\newblock {S}parse{GPT}: Massive language models can be accurately pruned in one-shot.
\newblock In Krause, A., Brunskill, E., Cho, K., Engelhardt, B., Sabato, S., and Scarlett, J. (eds.), \emph{Proceedings of the 40th International Conference on Machine Learning}, volume 202 of \emph{Proceedings of Machine Learning Research}, pp.\  10323--10337. PMLR, 23--29 Jul 2023.
\newblock URL \url{https://proceedings.mlr.press/v202/frantar23a.html}.

\bibitem[Fu et~al.(2020)Fu, Chi, and Liang]{fu2020guaranteed}
Fu, H., Chi, Y., and Liang, Y.
\newblock Guaranteed recovery of one-hidden-layer neural networks via cross entropy.
\newblock \emph{IEEE transactions on signal processing}, 68:\penalty0 3225--3235, 2020.

\bibitem[Han et~al.(2015)Han, Pool, Tran, and Dally]{han2015learning}
Han, S., Pool, J., Tran, J., and Dally, W.
\newblock Learning both weights and connections for efficient neural network.
\newblock \emph{Advances in neural information processing systems}, 28, 2015.

\bibitem[Han et~al.(2016{\natexlab{a}})Han, Liu, Mao, Pu, Pedram, Horowitz, and Dally]{han2016eie}
Han, S., Liu, X., Mao, H., Pu, J., Pedram, A., Horowitz, M.~A., and Dally, W.~J.
\newblock Eie: Efficient inference engine on compressed deep neural network.
\newblock \emph{ACM SIGARCH Computer Architecture News}, 44\penalty0 (3):\penalty0 243--254, 2016{\natexlab{a}}.

\bibitem[Han et~al.(2016{\natexlab{b}})Han, Mao, and Dally]{han2016deep}
Han, S., Mao, H., and Dally, W.~J.
\newblock Deep compression: Compressing deep neural networks with pruning, trained quantization and huffman coding.
\newblock In \emph{International Conference on Learning Representations}, 2016{\natexlab{b}}.

\bibitem[Jacot et~al.(2018)Jacot, Gabriel, and Hongler]{jacot2018neural}
Jacot, A., Gabriel, F., and Hongler, C.
\newblock Neural tangent kernel: Convergence and generalization in neural networks.
\newblock \emph{Advances in neural information processing systems}, 31, 2018.

\bibitem[Jaiswal et~al.(2023)Jaiswal, Liu, Chen, Ding, and Wang]{jaiswal2023instant}
Jaiswal, A.~K., Liu, S., Chen, T., Ding, Y., and Wang, Z.
\newblock Instant soup: Cheap pruning ensembles in a single pass can draw lottery tickets from large models.
\newblock In \emph{International Conference on Machine Learning}, pp.\  14691--14701. PMLR, 2023.

\bibitem[Karp et~al.(2021)Karp, Winston, Li, and Singh]{karp2021local}
Karp, S., Winston, E., Li, Y., and Singh, A.
\newblock Local signal adaptivity: Provable feature learning in neural networks beyond kernels.
\newblock \emph{Advances in Neural Information Processing Systems}, 34:\penalty0 24883--24897, 2021.

\bibitem[Koishekenov et~al.(2023)Koishekenov, Berard, and Nikoulina]{koishekenov-etal-2023-memory}
Koishekenov, Y., Berard, A., and Nikoulina, V.
\newblock Memory-efficient {NLLB}-200: Language-specific expert pruning of a massively multilingual machine translation model.
\newblock In Rogers, A., Boyd-Graber, J., and Okazaki, N. (eds.), \emph{Proceedings of the 61st Annual Meeting of the Association for Computational Linguistics (Volume 1: Long Papers)}, pp.\  3567--3585, Toronto, Canada, July 2023. Association for Computational Linguistics.
\newblock \doi{10.18653/v1/2023.acl-long.198}.
\newblock URL \url{https://aclanthology.org/2023.acl-long.198}.

\bibitem[Kolesnikov et~al.(2020)Kolesnikov, Beyer, Zhai, Puigcerver, Yung, Gelly, and Houlsby]{kolesnikov2020big}
Kolesnikov, A., Beyer, L., Zhai, X., Puigcerver, J., Yung, J., Gelly, S., and Houlsby, N.
\newblock Big transfer (bit): General visual representation learning.
\newblock In \emph{Computer Vision--ECCV 2020: 16th European Conference, Glasgow, UK, August 23--28, 2020, Proceedings, Part V 16}, pp.\  491--507. Springer, 2020.

\bibitem[Krizhevsky(2009)]{krizhevsky2009learning}
Krizhevsky, A.
\newblock Learning multiple layers of features from tiny images.
\newblock Technical report, Canadian Institute For Advanced Research, 2009.

\bibitem[Lee et~al.(2019{\natexlab{a}})Lee, Xiao, Schoenholz, Bahri, Novak, Sohl-Dickstein, and Pennington]{lee2019wide}
Lee, J., Xiao, L., Schoenholz, S., Bahri, Y., Novak, R., Sohl-Dickstein, J., and Pennington, J.
\newblock Wide neural networks of any depth evolve as linear models under gradient descent.
\newblock \emph{Advances in neural information processing systems}, 32, 2019{\natexlab{a}}.

\bibitem[Lee et~al.(2019{\natexlab{b}})Lee, Ajanthan, and Torr]{lee2019snip}
Lee, N., Ajanthan, T., and Torr, P.
\newblock Snip: single-shot network pruning based on connection sensitivity.
\newblock In \emph{International Conference on Learning Representations}. Open Review, 2019{\natexlab{b}}.

\bibitem[Lepikhin et~al.(2020)Lepikhin, Lee, Xu, Chen, Firat, Huang, Krikun, Shazeer, and Chen]{lepikhin2020gshard}
Lepikhin, D., Lee, H., Xu, Y., Chen, D., Firat, O., Huang, Y., Krikun, M., Shazeer, N., and Chen, Z.
\newblock Gshard: Scaling giant models with conditional computation and automatic sharding.
\newblock In \emph{International Conference on Learning Representations}, 2020.

\bibitem[Lewis et~al.(2021)Lewis, Bhosale, Dettmers, Goyal, and Zettlemoyer]{lewis2021base}
Lewis, M., Bhosale, S., Dettmers, T., Goyal, N., and Zettlemoyer, L.
\newblock Base layers: Simplifying training of large, sparse models.
\newblock In \emph{International Conference on Machine Learning}, pp.\  6265--6274. PMLR, 2021.

\bibitem[Li et~al.(2020{\natexlab{a}})Li, Kong, Zhang, Li, Li, Liu, and Ding]{li2020efficient}
Li, B., Kong, Z., Zhang, T., Li, J., Li, Z., Liu, H., and Ding, C.
\newblock Efficient transformer-based large scale language representations using hardware-friendly block structured pruning.
\newblock In \emph{Findings of the Association for Computational Linguistics: EMNLP 2020}, pp.\  3187--3199, 2020{\natexlab{a}}.

\bibitem[Li et~al.(2016)Li, Kadav, Durdanovic, Samet, and Graf]{li2016pruning}
Li, H., Kadav, A., Durdanovic, I., Samet, H., and Graf, H.~P.
\newblock Pruning filters for efficient convnets.
\newblock In \emph{International Conference on Learning Representations}, 2016.

\bibitem[Li et~al.(2022)Li, Wang, Liu, Chen, and Xiong]{li2022generalization}
Li, H., Wang, M., Liu, S., Chen, P.-Y., and Xiong, J.
\newblock Generalization guarantee of training graph convolutional networks with graph topology sampling.
\newblock In \emph{International Conference on Machine Learning}, pp.\  13014--13051. PMLR, 2022.

\bibitem[Li et~al.(2023)Li, Wang, Liu, and Chen]{li2023a}
Li, H., Wang, M., Liu, S., and Chen, P.-Y.
\newblock A theoretical understanding of shallow vision transformers: Learning, generalization, and sample complexity.
\newblock In \emph{The Eleventh International Conference on Learning Representations}, 2023.
\newblock URL \url{https://openreview.net/forum?id=jClGv3Qjhb}.

\bibitem[Li et~al.(2024{\natexlab{a}})Li, Wang, Lu, Cui, and Chen]{li2024training}
Li, H., Wang, M., Lu, S., Cui, X., and Chen, P.-Y.
\newblock Training nonlinear transformers for efficient in-context learning: A theoretical learning and generalization analysis.
\newblock \emph{arXiv preprint arXiv:2402.15607}, 2024{\natexlab{a}}.

\bibitem[Li et~al.(2024{\natexlab{b}})Li, Zhang, Zhang, Wang, Liu, and Chen]{li2024does}
Li, H., Zhang, S., Zhang, Y., Wang, M., Liu, S., and Chen, P.-Y.
\newblock How does promoting the minority fraction affect generalization? a theoretical study of one-hidden-layer neural network on group imbalance.
\newblock \emph{IEEE Journal of Selected Topics in Signal Processing}, 2024{\natexlab{b}}.

\bibitem[Li \& Liang(2018)Li and Liang]{li2018learning}
Li, Y. and Liang, Y.
\newblock Learning overparameterized neural networks via stochastic gradient descent on structured data.
\newblock \emph{Advances in neural information processing systems}, 31, 2018.

\bibitem[Li et~al.(2020{\natexlab{b}})Li, Wallace, Shen, Lin, Keutzer, Klein, and Gonzalez]{li2020train}
Li, Z., Wallace, E., Shen, S., Lin, K., Keutzer, K., Klein, D., and Gonzalez, J.
\newblock Train big, then compress: Rethinking model size for efficient training and inference of transformers.
\newblock In \emph{International Conference on machine learning}, pp.\  5958--5968. PMLR, 2020{\natexlab{b}}.

\bibitem[Liu et~al.(2021{\natexlab{a}})Liu, Zhang, Kuang, Zhou, Xue, Wang, Chen, Yang, Liao, and Zhang]{liu2021group}
Liu, L., Zhang, S., Kuang, Z., Zhou, A., Xue, J.-H., Wang, X., Chen, Y., Yang, W., Liao, Q., and Zhang, W.
\newblock Group fisher pruning for practical network compression.
\newblock In \emph{International Conference on Machine Learning}, pp.\  7021--7032. PMLR, 2021{\natexlab{a}}.

\bibitem[Liu et~al.(2021{\natexlab{b}})Liu, Li, Li, and Cheng]{liu2021ebert}
Liu, Z., Li, F., Li, G., and Cheng, J.
\newblock Ebert: Efficient bert inference with dynamic structured pruning.
\newblock In \emph{Findings of the Association for Computational Linguistics: ACL-IJCNLP 2021}, pp.\  4814--4823, 2021{\natexlab{b}}.

\bibitem[Luo et~al.(2017)Luo, Wu, and Lin]{luo2017thinet}
Luo, J.-H., Wu, J., and Lin, W.
\newblock Thinet: A filter level pruning method for deep neural network compression.
\newblock In \emph{Proceedings of the IEEE international conference on computer vision}, pp.\  5058--5066, 2017.

\bibitem[Nonnenmacher et~al.(2021)Nonnenmacher, Pfeil, Steinwart, and Reeb]{nonnenmacher2021sosp}
Nonnenmacher, M., Pfeil, T., Steinwart, I., and Reeb, D.
\newblock Sosp: Efficiently capturing global correlations by second-order structured pruning.
\newblock In \emph{International Conference on Learning Representations}, 2021.

\bibitem[Puigcerver et~al.(2022)Puigcerver, Jenatton, Riquelme, Awasthi, and Bhojanapalli]{puigcerver2022adversarial}
Puigcerver, J., Jenatton, R., Riquelme, C., Awasthi, P., and Bhojanapalli, S.
\newblock On the adversarial robustness of mixture of experts.
\newblock \emph{Advances in Neural Information Processing Systems}, 35:\penalty0 9660--9671, 2022.

\bibitem[Riquelme et~al.(2021)Riquelme, Puigcerver, Mustafa, Neumann, Jenatton, Susano~Pinto, Keysers, and Houlsby]{riquelme2021scaling}
Riquelme, C., Puigcerver, J., Mustafa, B., Neumann, M., Jenatton, R., Susano~Pinto, A., Keysers, D., and Houlsby, N.
\newblock Scaling vision with sparse mixture of experts.
\newblock \emph{Advances in Neural Information Processing Systems}, 34:\penalty0 8583--8595, 2021.

\bibitem[Russakovsky et~al.(2015)Russakovsky, Deng, Su, Krause, Satheesh, Ma, Huang, Karpathy, Khosla, Bernstein, et~al.]{russakovsky2015imagenet}
Russakovsky, O., Deng, J., Su, H., Krause, J., Satheesh, S., Ma, S., Huang, Z., Karpathy, A., Khosla, A., Bernstein, M., et~al.
\newblock Imagenet large scale visual recognition challenge.
\newblock \emph{International journal of computer vision}, 115:\penalty0 211--252, 2015.

\bibitem[Sanh et~al.(2020)Sanh, Wolf, and Rush]{sanh2020movement}
Sanh, V., Wolf, T., and Rush, A.
\newblock Movement pruning: Adaptive sparsity by fine-tuning.
\newblock \emph{Advances in Neural Information Processing Systems}, 33:\penalty0 20378--20389, 2020.

\bibitem[Shalev-Shwartz et~al.(2020)]{shalev2020computational}
Shalev-Shwartz, S. et~al.
\newblock Computational separation between convolutional and fully-connected networks.
\newblock In \emph{International Conference on Learning Representations}, 2020.

\bibitem[Shazeer et~al.(2017)Shazeer, Mirhoseini, Maziarz, Davis, Le, Hinton, and Dean]{shazeer2017outrageously}
Shazeer, N., Mirhoseini, A., Maziarz, K., Davis, A., Le, Q.~V., Hinton, G.~E., and Dean, J.
\newblock Outrageously large neural networks: The sparsely-gated mixture-of-experts layer.
\newblock In \emph{International Conference on Learning Representations}, 2017.

\bibitem[Shi et~al.(2021)Shi, Wei, and Liang]{shi2021theoretical}
Shi, Z., Wei, J., and Liang, Y.
\newblock A theoretical analysis on feature learning in neural networks: Emergence from inputs and advantage over fixed features.
\newblock In \emph{International Conference on Learning Representations}, 2021.

\bibitem[Tung \& Mori(2018)Tung and Mori]{tung2018clip}
Tung, F. and Mori, G.
\newblock Clip-q: Deep network compression learning by in-parallel pruning-quantization.
\newblock In \emph{Proceedings of the IEEE conference on computer vision and pattern recognition}, pp.\  7873--7882, 2018.

\bibitem[Vaswani et~al.(2017)Vaswani, Shazeer, Parmar, Uszkoreit, Jones, Gomez, Kaiser, and Polosukhin]{vaswani2017attention}
Vaswani, A., Shazeer, N., Parmar, N., Uszkoreit, J., Jones, L., Gomez, A.~N., Kaiser, {\L}., and Polosukhin, I.
\newblock Attention is all you need.
\newblock \emph{Advances in neural information processing systems}, 30, 2017.

\bibitem[Wang et~al.(2020)Wang, Wohlwend, and Lei]{wang2020structured}
Wang, Z., Wohlwend, J., and Lei, T.
\newblock Structured pruning of large language models.
\newblock In \emph{Proceedings of the 2020 Conference on Empirical Methods in Natural Language Processing (EMNLP)}, pp.\  6151--6162, 2020.

\bibitem[Wortsman et~al.(2022)Wortsman, Ilharco, Gadre, Roelofs, Gontijo-Lopes, Morcos, Namkoong, Farhadi, Carmon, Kornblith, et~al.]{wortsman2022model}
Wortsman, M., Ilharco, G., Gadre, S.~Y., Roelofs, R., Gontijo-Lopes, R., Morcos, A.~S., Namkoong, H., Farhadi, A., Carmon, Y., Kornblith, S., et~al.
\newblock Model soups: averaging weights of multiple fine-tuned models improves accuracy without increasing inference time.
\newblock In \emph{International Conference on Machine Learning}, pp.\  23965--23998. PMLR, 2022.

\bibitem[Yu et~al.(2022)Yu, Wang, Vasudevan, Yeung, Seyedhosseini, and Wu]{yu2022coca}
Yu, J., Wang, Z., Vasudevan, V., Yeung, L., Seyedhosseini, M., and Wu, Y.
\newblock Coca: Contrastive captioners are image-text foundation models.
\newblock \emph{Transactions on Machine Learning Research}, 2022.
\newblock ISSN 2835-8856.
\newblock URL \url{https://openreview.net/forum?id=Ee277P3AYC}.

\bibitem[Zafrir et~al.(2021)Zafrir, Larey, Boudoukh, Shen, and Wasserblat]{zafrir2021prune}
Zafrir, O., Larey, A., Boudoukh, G., Shen, H., and Wasserblat, M.
\newblock Prune once for all: Sparse pre-trained language models.
\newblock \emph{arXiv preprint arXiv:2111.05754}, 2021.

\bibitem[Zhang et~al.(2020{\natexlab{a}})Zhang, Wang, Liu, Chen, and Xiong]{zhang2020fast}
Zhang, S., Wang, M., Liu, S., Chen, P.-Y., and Xiong, J.
\newblock Fast learning of graph neural networks with guaranteed generalizability: one-hidden-layer case.
\newblock In \emph{International Conference on Machine Learning}, pp.\  11268--11277. PMLR, 2020{\natexlab{a}}.

\bibitem[Zhang et~al.(2020{\natexlab{b}})Zhang, Wang, Xiong, Liu, and Chen]{zhang2020improved}
Zhang, S., Wang, M., Xiong, J., Liu, S., and Chen, P.-Y.
\newblock Improved linear convergence of training {CNN}s with generalizability guarantees: A one-hidden-layer case.
\newblock \emph{IEEE Transactions on Neural Networks and Learning Systems}, 32\penalty0 (6):\penalty0 2622--2635, 2020{\natexlab{b}}.

\bibitem[Zhang et~al.(2022)Zhang, Wang, Chen, Liu, Lu, and Liu]{zhang2022joint}
Zhang, S., Wang, M., Chen, P.-Y., Liu, S., Lu, S., and Liu, M.
\newblock Joint edge-model sparse learning is provably efficient for graph neural networks.
\newblock In \emph{The Eleventh International Conference on Learning Representations}, 2022.

\bibitem[Zhang et~al.(2023)Zhang, Li, Wang, Liu, Chen, Lu, Liu, Murugesan, and Chaudhury]{zhang2023convergence}
Zhang, S., Li, H., Wang, M., Liu, M., Chen, P.-Y., Lu, S., Liu, S., Murugesan, K., and Chaudhury, S.
\newblock On the convergence and sample complexity analysis of deep q-networks with $\epsilon$-greedy exploration.
\newblock In \emph{Thirty-seventh Conference on Neural Information Processing Systems}, 2023.

\bibitem[Zhong et~al.(2017)Zhong, Song, Jain, Bartlett, and Dhillon]{zhong2017recovery}
Zhong, K., Song, Z., Jain, P., Bartlett, P.~L., and Dhillon, I.~S.
\newblock Recovery guarantees for one-hidden-layer neural networks.
\newblock In \emph{International conference on machine learning}, pp.\  4140--4149. PMLR, 2017.

\bibitem[Zhou et~al.(2022)Zhou, Lei, Liu, Du, Huang, Zhao, Dai, Chen, Le, and Laudon]{zhou2022mixtureofexperts}
Zhou, Y., Lei, T., Liu, H., Du, N., Huang, Y., Zhao, V.~Y., Dai, A.~M., Chen, Z., Le, Q.~V., and Laudon, J.
\newblock Mixture-of-experts with expert choice routing.
\newblock In Oh, A.~H., Agarwal, A., Belgrave, D., and Cho, K. (eds.), \emph{Advances in Neural Information Processing Systems}, 2022.
\newblock URL \url{https://openreview.net/forum?id=jdJo1HIVinI}.

\end{thebibliography}
\bibliographystyle{icml2024}
\newpage
\appendix
\onecolumn
\section{More Details on V-MoE and $\text{E}^3$-MoE}\label{m}

\textbf{The V-MoE models}. The released set of V-MoE models by \citet{riquelme2021scaling} includes the pre-trained version of the models on ImageNet-21k \citep{kolesnikov2020big} and the two fine-tuned versions on CIFAR-10 and ImageNet-1k (i.e., ILSVRC2012), respectively. The models contain 12 transformer encoders (encoders 0, 1, ..., and 11). Every odd encoder (i.e., encoders 1, 3, 5, 7, 9, and 11) is the MoE encoder. Each MoE encoder contains $8$ experts (i.e., $k=8$). The input images are divided into $16\times 16$ patches and then transformed into the encoder's hidden dimension of $768$ (i.e., $d=768$). The MoE hidden dimension is $3072$ (i.e., $m=3072$). The CIFAR-10 fine-tuned version contains 65 tokens for an input image (i.e., $n=65$). The ImageNet fine-tuned version contains 577 tokens for an input image (i.e., $n=577$). The MoE routers in the model implement token-choice routing where each token is routed to 2 experts (i.e., $l=2$). The total number of parameters in the model is roughly 979 million. SGD with momentum and cosine learning rate decay is implemented to obtain the fine-tuned models. The fine-tuning steps are 1000 and 10,000 for CIFAR-10 and ImageNet, respectively. More details can be found in \citet{riquelme2021scaling}.

\textbf{The $\text{E}^3$-MoE models}. \citet{allingham2022sparse} proposed the efficient ensembles of experts in MoE (i.e., $\text{E}^3$-MoE) to improve the generalization performance of V-MoE. The architecture implements the ensembles of experts where, in each MoE encoder the experts are divided into multiple ensembles. The released set of $\text{E}^3$-MoE models includes the pre-trained version on ImageNet and the fine-tuned version on CIFAR-100. The models contain 8 transformer encoders (i.e., encoders 0, 1, ..., and 7), where the last two encoders contain the ensemble of MoEs (i.e., encoders 5 and 7). There are 2 MoE ensembles in the model and each MoE ensembles contains 4 experts per MoE encoder (i.e., $k=4$). The input images are divided into $32\times 32$ patches and then transformed into encoder's hidden dimension of 512 (i.e., $d=512$) where the MoE hidden dimension is 2048 (i.e., $m=2048$). The number of tokens per image in each encoder is $65$ (i.e., $n=65$). The MoE routers implement the token-choice routing, selecting one expert per token (i.e., $l=1$). Total number of parameters in the model is roughly 167 million. Again, SGD with momentum and cosine learning rate decay is implemented during fine-tuning for 2000 steps.

\section{More Experimental Results}
\subsection{On Synthetic Data}\label{more_synthetic_exp}
As described in section \ref{synthetic}, here we present results for experts in the negative group. Figure \ref{router_component_large_2} and \ref{router_component_small_2} present the components for large router (router 20) and small router (router 12), respectively. Figure \ref{neuron_component_large_2} and \ref{neuron_component_small_2} present the components for neurons of the corresponding experts.

\begin{figure}[ht]
    \begin{subfigure}[(b)]{0.24\linewidth}
        \centering
        \includegraphics[width=\linewidth]{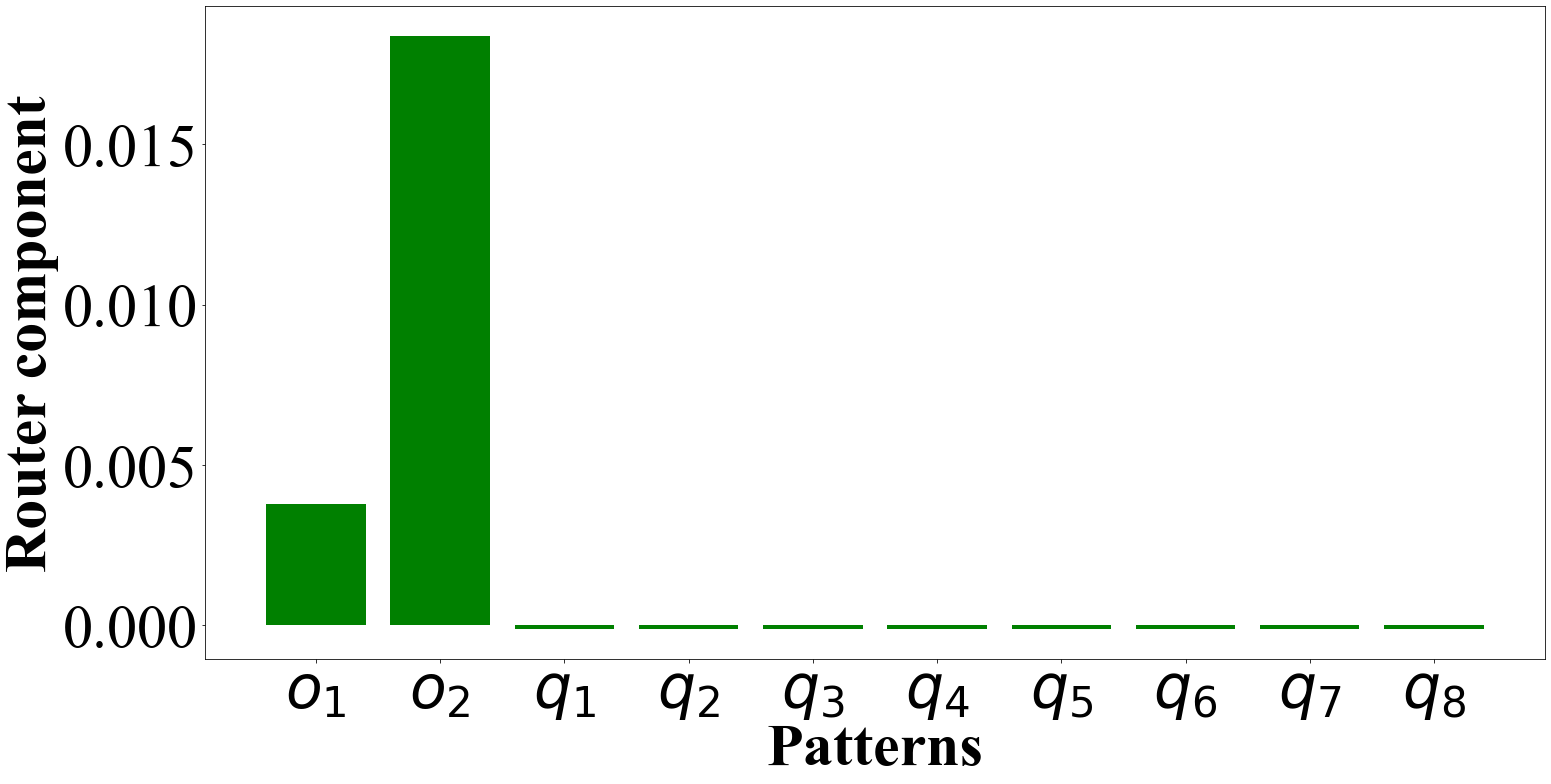}
        \caption{}
        \label{router_component_large_2}
    \end{subfigure}
    \hfill
    \begin{subfigure}[(c)]{0.24\linewidth}
        \centering
        \includegraphics[width=\linewidth]{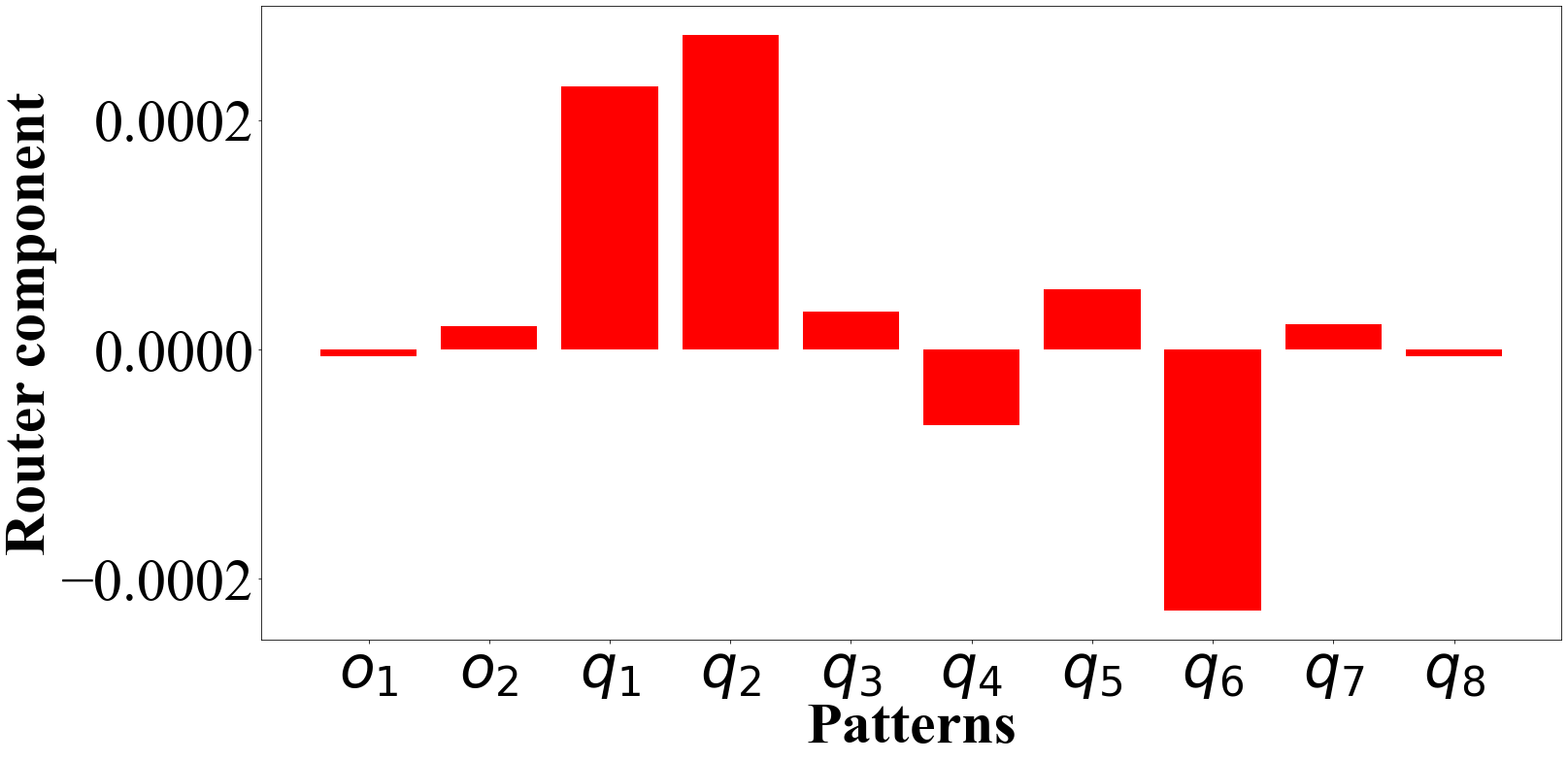}
        \caption{}
        \label{router_component_small_2}
    \end{subfigure}
    \hfill
    \begin{subfigure}[(d)]{0.24\linewidth}
        \centering
        \includegraphics[width=\linewidth]{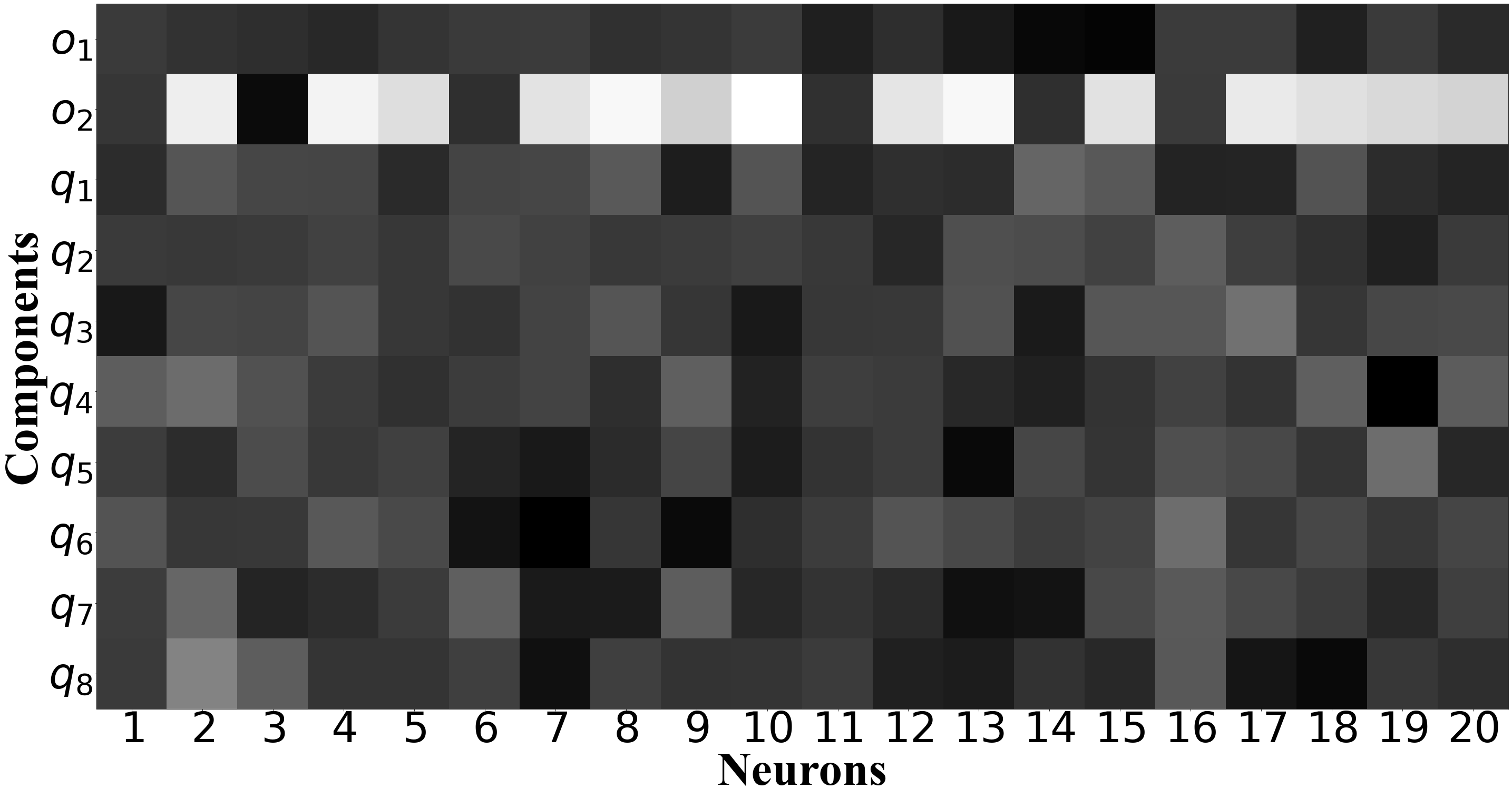}
        \caption{}
        \label{neuron_component_large_2}
    \end{subfigure}
    \hfill
    \begin{subfigure}[(e)]{0.24\linewidth}
        \centering
        \includegraphics[width=\linewidth]{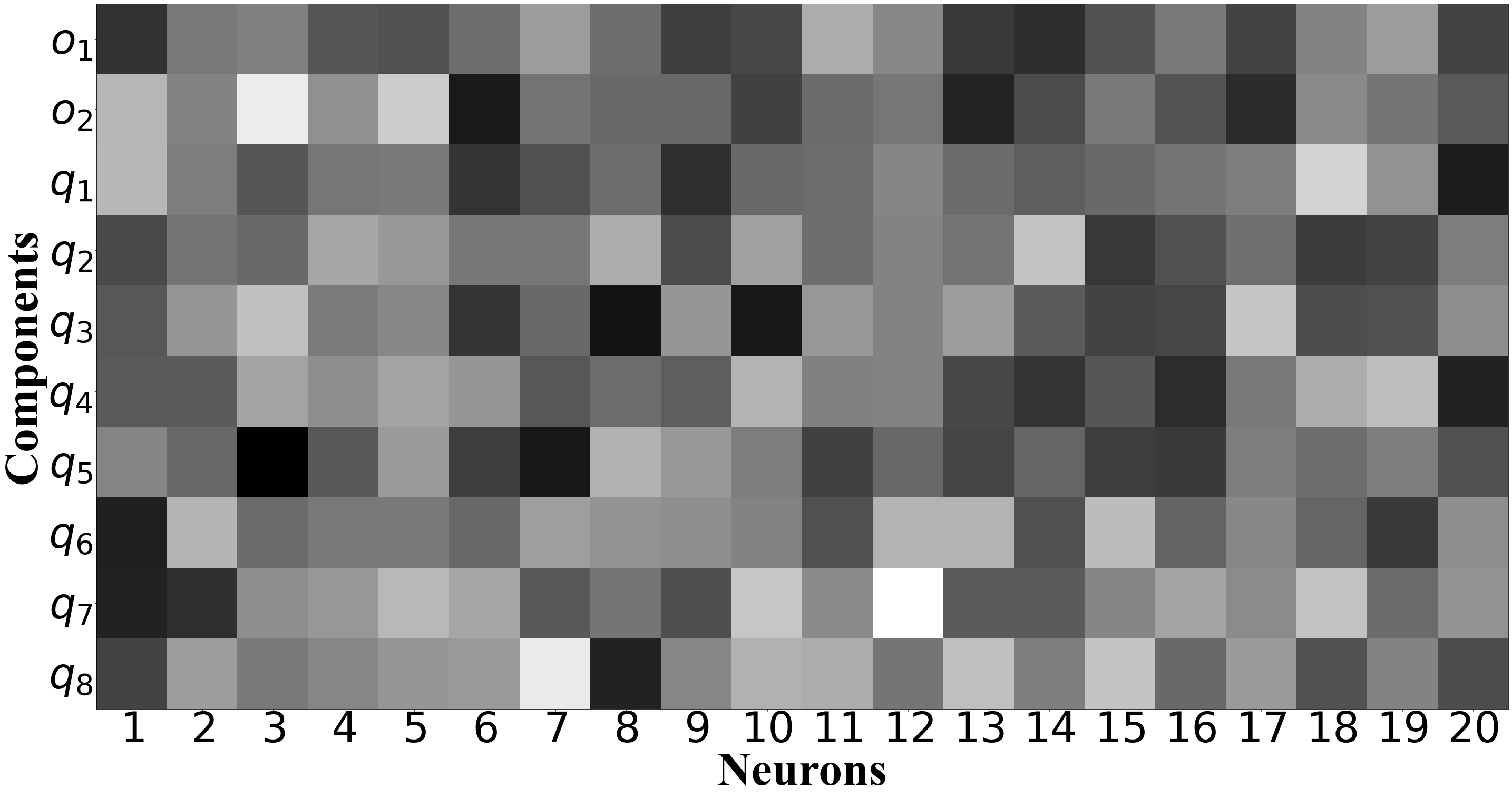}
        \caption{}
        \label{neuron_component_small_2}
    \end{subfigure}
    \caption{Verification of the theoretical findings on synthetic data: (a) Router's components for large router (router 20), (b) Router's components for small router (router 12), (c) Neurons components for expert with large router (expert 20), (d) Neurons components for expert with small router (expert 12) (larger pixel intensity represents larger component of the router weights)}
    \vspace{-1mm}
\end{figure}

\begin{figure}[ht]
    \begin{subfigure}[(a)]{0.48\linewidth}
        \centering
        \includegraphics[width=0.80\linewidth]{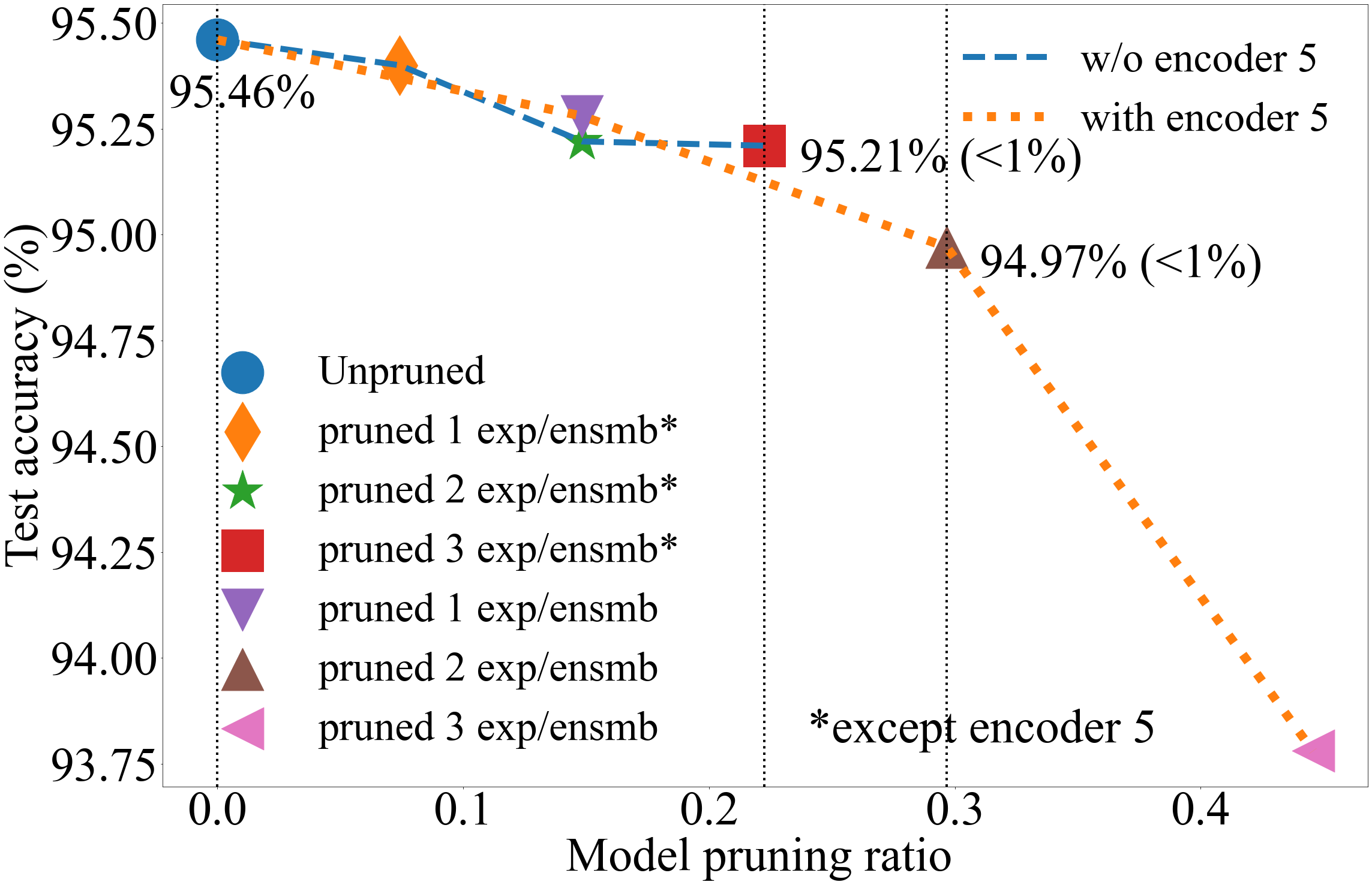}
        \caption{}
        \label{cifar_10_only_pruned_eee}
    \end{subfigure}
    \hfill
    \begin{subfigure}[(b)]{0.48\linewidth}
        \centering
        \includegraphics[width=0.80\linewidth]{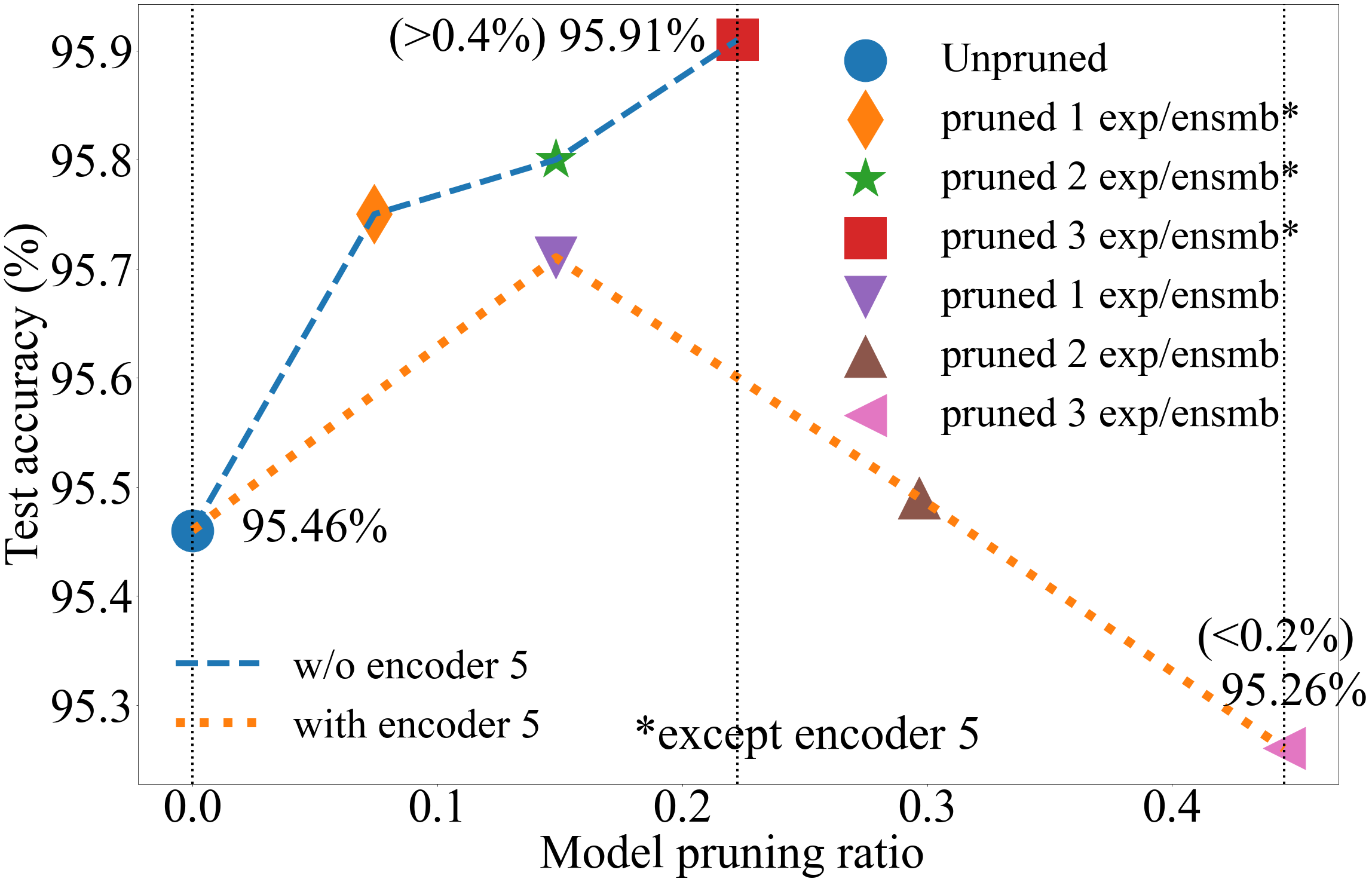}
        \caption{}
        \label{cifar_10_pruned_finetuned_eee}
    \end{subfigure}
    \hfill
    \begin{subfigure}[(c)]{0.48\linewidth}
        \centering
        \includegraphics[width=0.80\linewidth]{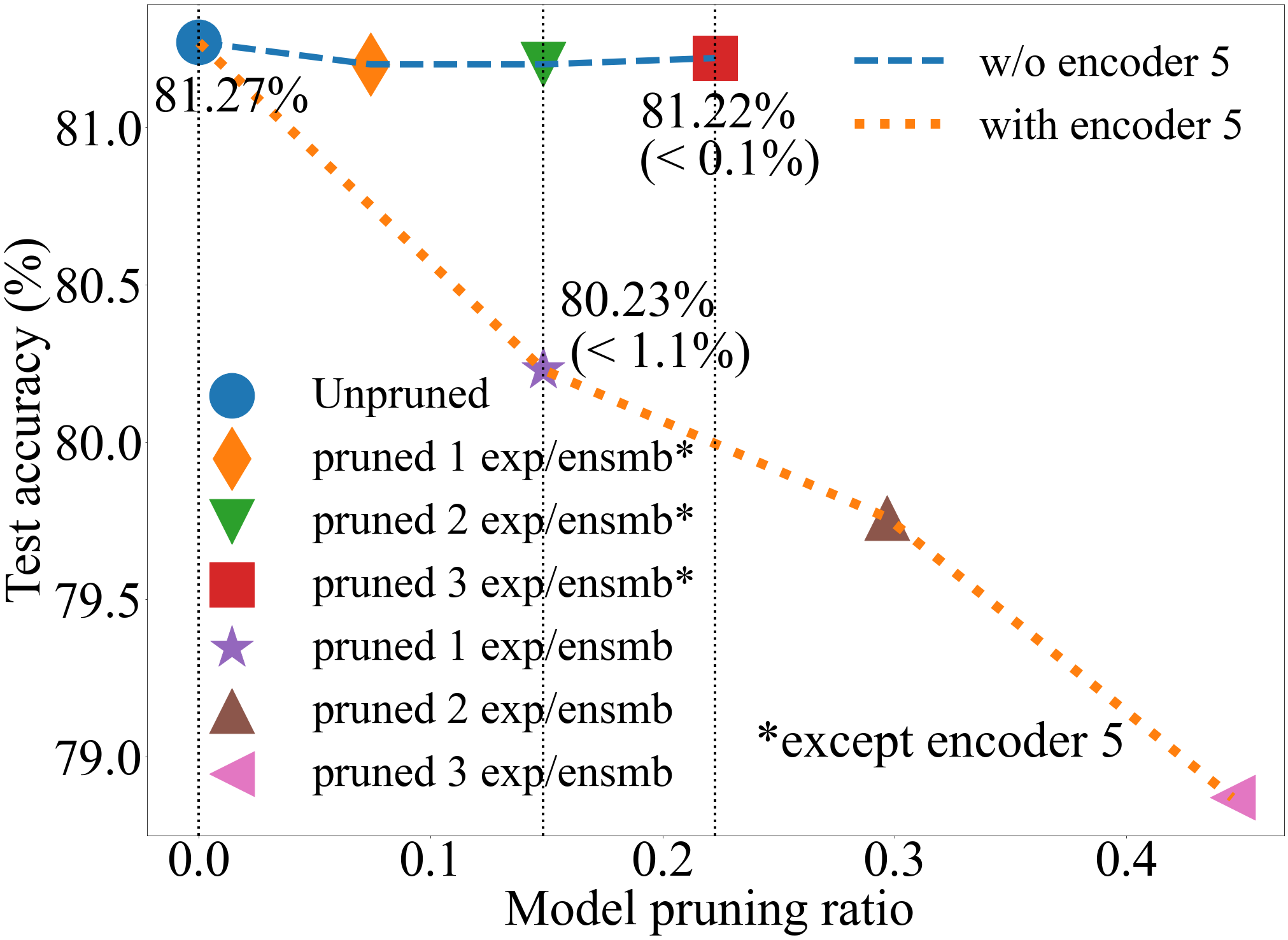}
        \caption{}
        \label{cifar_100_only_pruned_eee}
    \end{subfigure}
    \hfill
    \begin{subfigure}[(d)]{0.48\linewidth}
        \centering
        \includegraphics[width=0.80\linewidth]{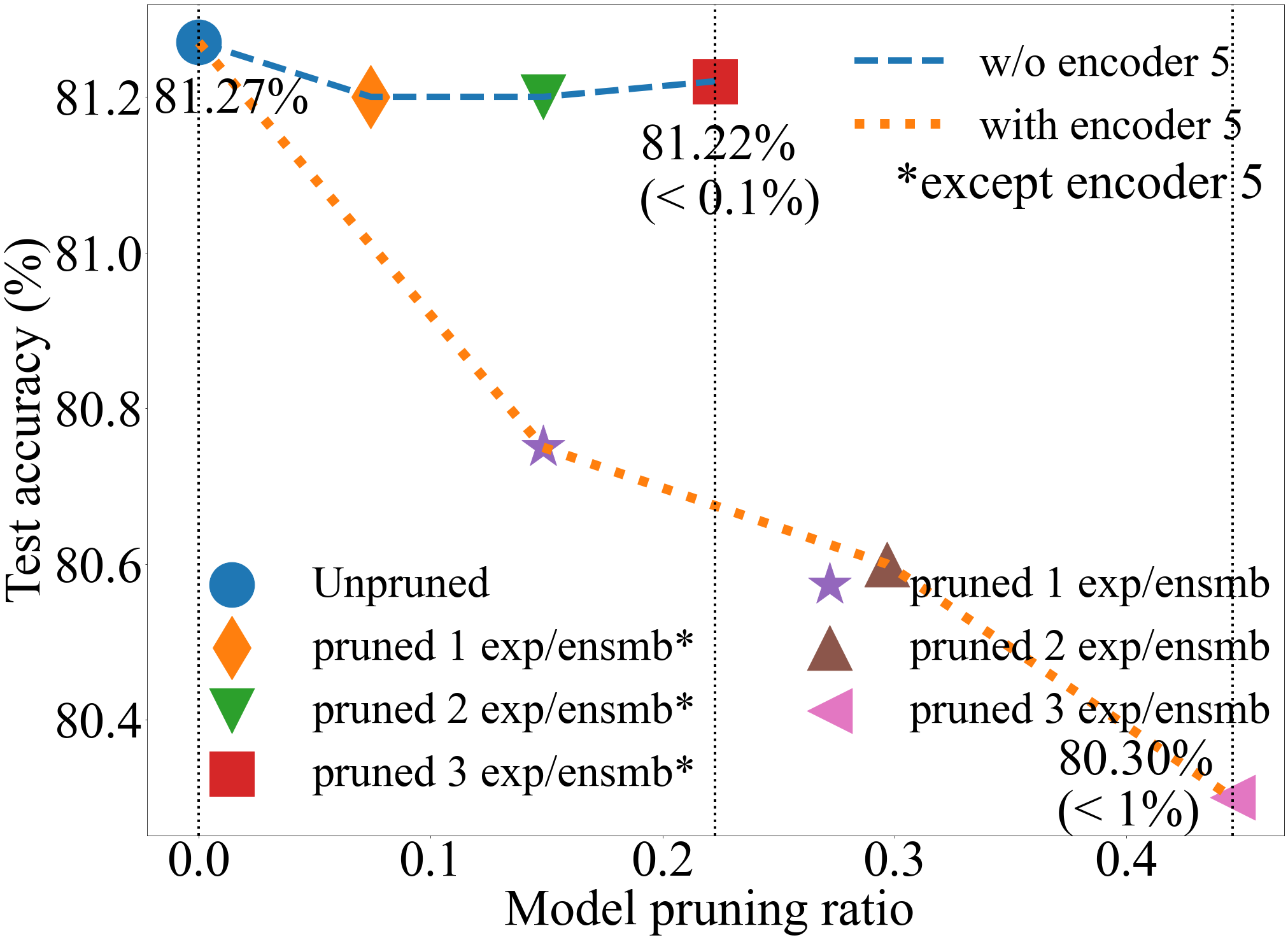}
        \caption{}
        \label{cifar_100_pruned_finetuned_eee}
    \end{subfigure}
    \caption{Generalization performance of the pruned \textbf{$\text{E}^3$-MoE} models: (a) On CIFAR-10 \textbf{without} post-pruning fine-tuning, (b) On CIFAR-10 \text{with} post-pruning fine-tuning, (c) On CIFAR-100 \textbf{without} post-pruning fine-tuning, (d) On CIFAR-100 \textbf{with} post-pruning fine-tuning}
    \vspace{-1mm}
\end{figure}
\begin{figure}[h!]
    \begin{subfigure}[(a)]{0.32\linewidth}
        \centering
        \includegraphics[width=\linewidth]{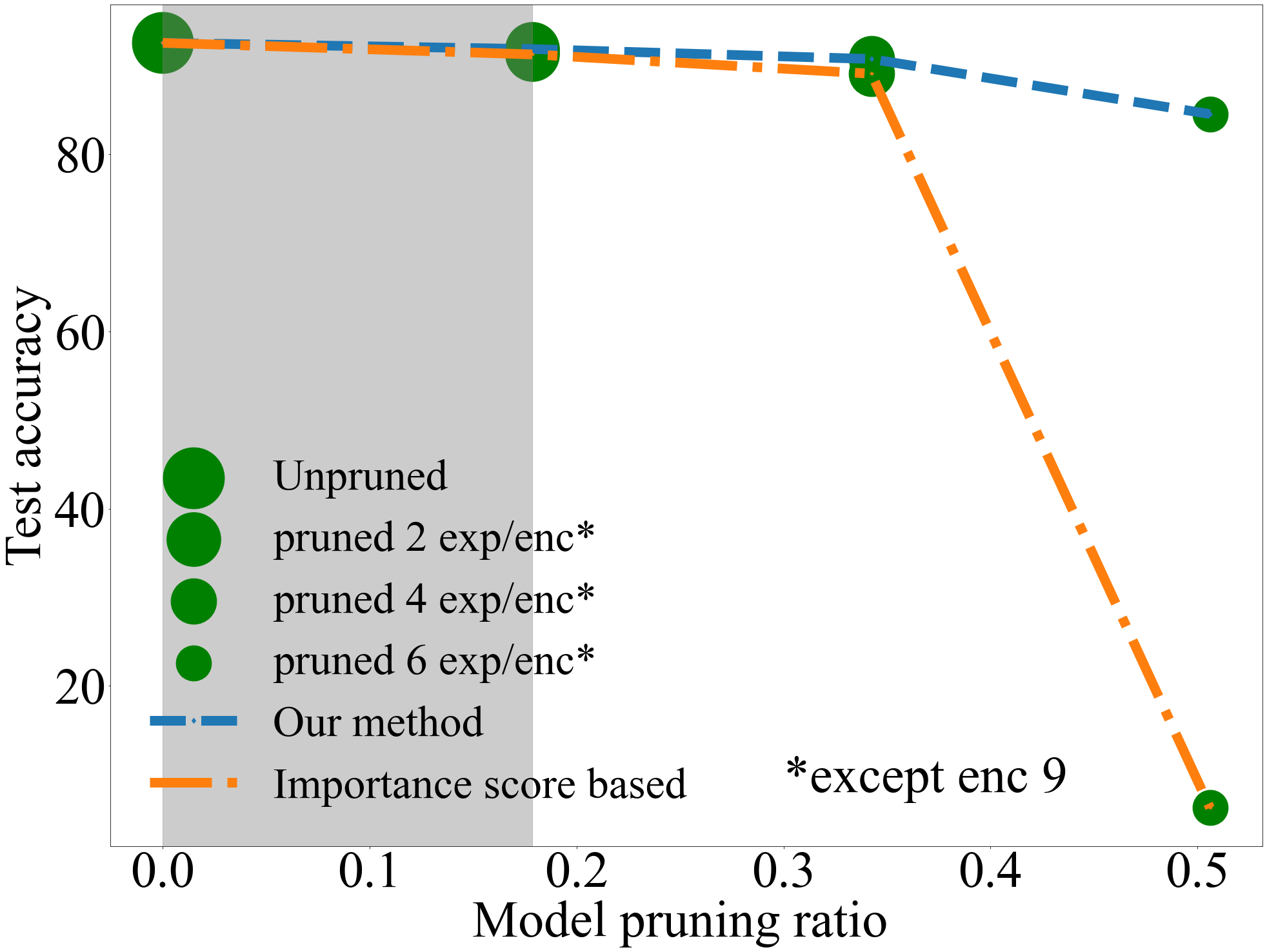}
        \caption{}
        \label{cifar_100_imp_score}
    \end{subfigure}
    \hfill
    \begin{subfigure}[(b)]{0.32\linewidth}
        \centering
        \includegraphics[width=\linewidth]{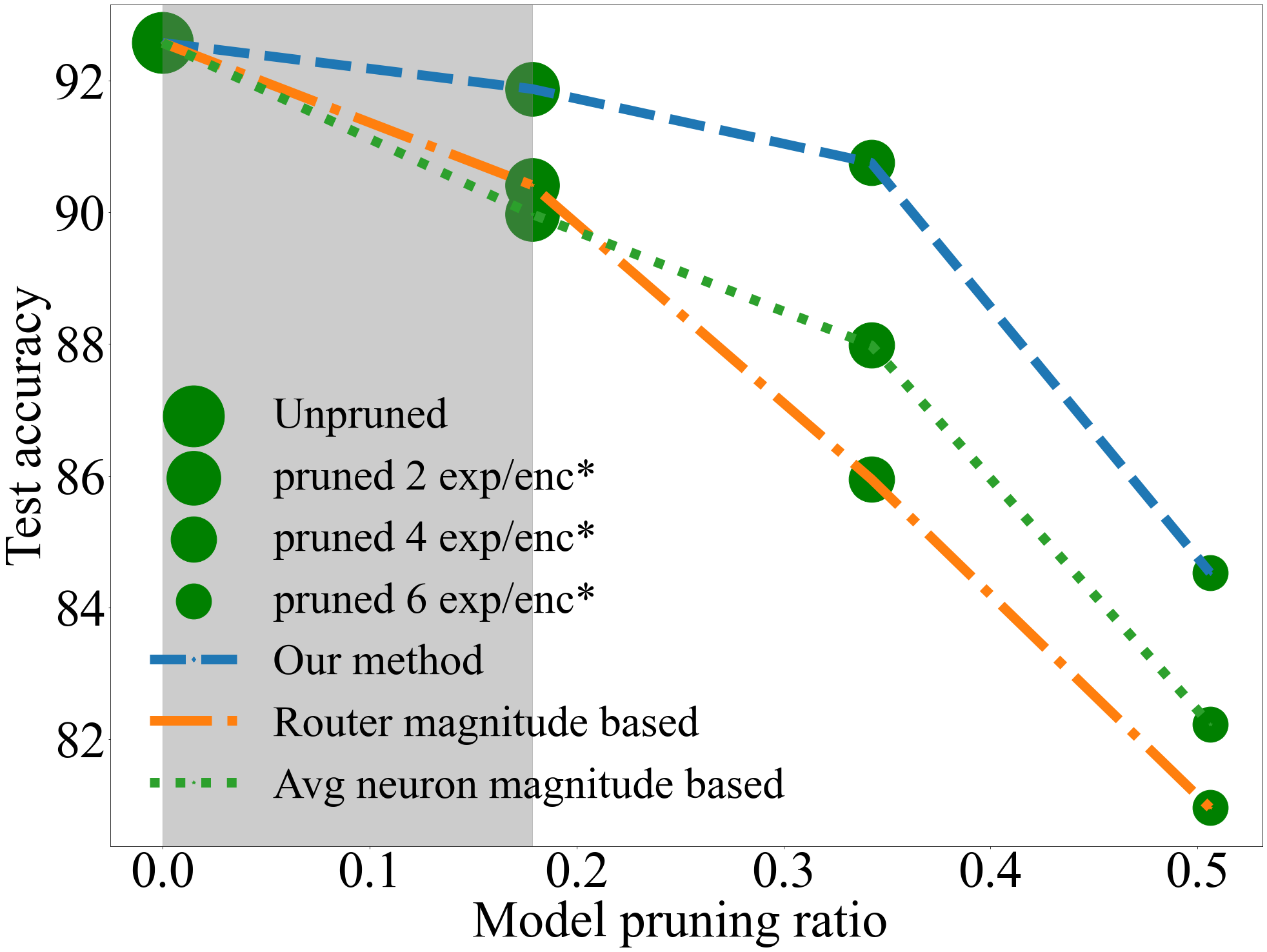}
        \caption{}
        \label{cifar_100_magn}
    \end{subfigure}
    \hfill
    \begin{subfigure}[(d)]{0.32\linewidth}
        \centering
        \includegraphics[width=\linewidth]{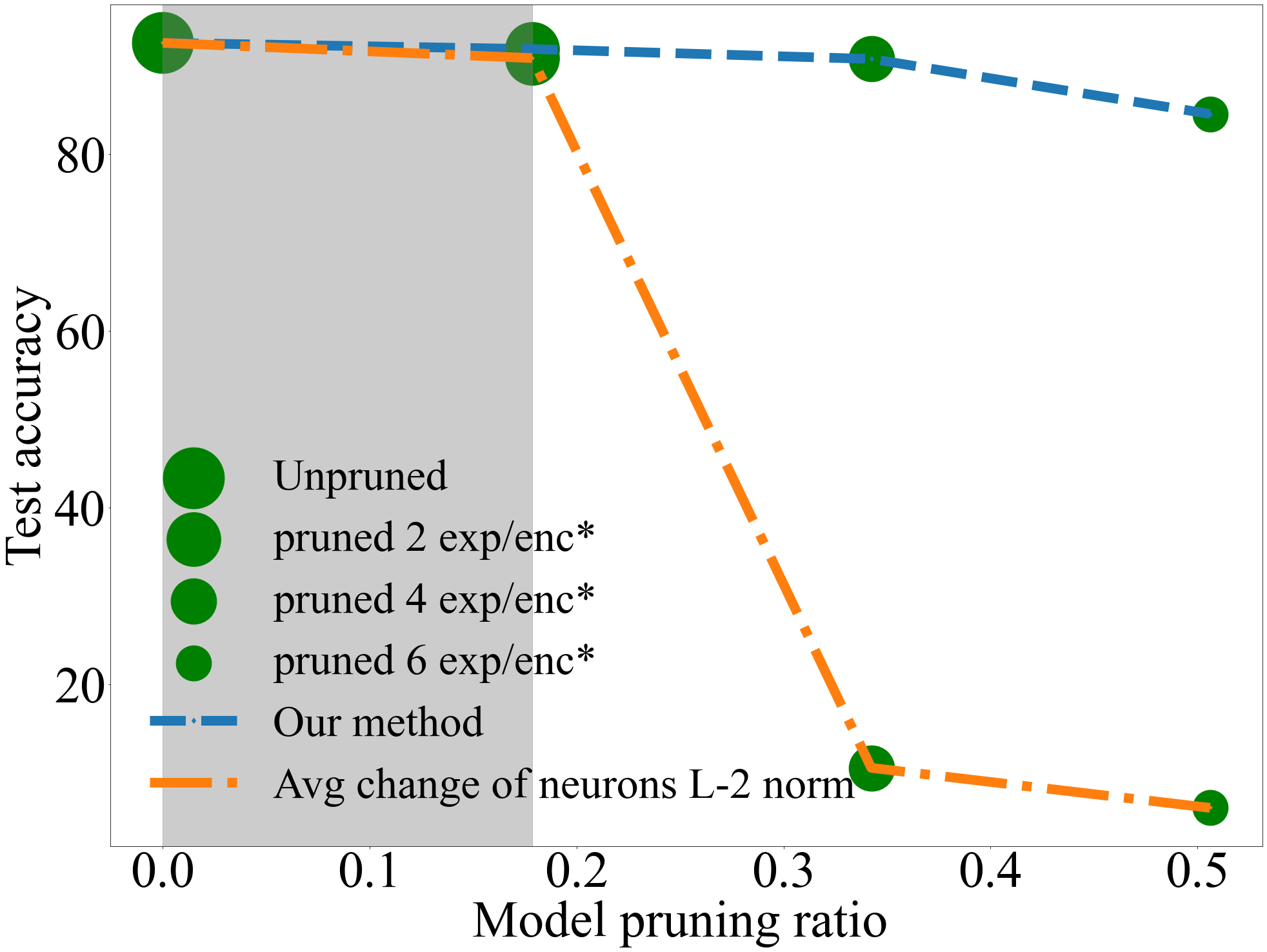}
        \caption{}
        \label{cifar_100_change_magn}
    \end{subfigure}
    \caption{Comparison between different expert pruning methods on CIFAR-100 for pruning V-MoE: (a) vs. importance score, (b) vs. absolute magnitude, (c) vs. average change-in-neurons-magnitude}
    \vspace{-1mm}
\end{figure}
\subsection{Generalization Performance On $\text{E}^3$-MoE}\label{eee_results}

We present generalization results of the pruned $\text{E}^3$-MoE models fine-tuned on CIFAR-10 and CIFAR-100 in Figure \ref{cifar_10_only_pruned_eee}, \ref{cifar_10_pruned_finetuned_eee} and Figure \ref{cifar_100_only_pruned_eee}, \ref{cifar_100_pruned_finetuned_eee}, respectively. As we can see and as reported in the main paper, similar results are found for the pruned $\text{E}^3$-MoE models as for the pruned V-MoE models. More specifically, for the pruned $\text{E}^3$-MoE, 75\% of the experts and 45\% of the whole model can be pruned for both of the datasets while maintaining generalization performance within 1\% of the unpruned model. We do not test the generalization performance of the pruned $\text{E}^3$-MoE on ImageNet as the released $\text{E}^3$-MoE model is pre-trained on ImageNet itself.

\subsection{Comparison with Other Methods for Pruning V-MoE on CIFAR-100}\label{baseline_cifar100}
Figure \ref{cifar_100_imp_score}, \ref{cifar_100_magn} and, \ref{cifar_100_change_magn} presents the comparative results for pruning V-MoE on CIFAR-100 among different expert pruning methods. The proposed method outperforms all other methods for this dataset.

\subsection{On the Inference Efficiency of V-MoE and $\text{E}^3$-MoE}\label{inference_more_results}
We present the results in Figure \ref{cifar_10_time} for inference time in predicting the whole CIFAR-10 test set by the pruned V-MoE. As we described in section \ref{experiments}, inference time falls linearly with pruning ratio. We also present results for linear reduction of FLOPs and inference time in the pruned V-MoE on CIFAR-100 and ImageNet in Figure \ref{cifar_100_flop_vmoe}, \ref{cifar_100_time_vmoe} and Figure \ref{imagenet_flop}, \ref{imagenet_time} and, in the pruned $\text{E}^3$-MoE on CIFAR-10 and CIFAR-100 in Figure \ref{cifar_10_flop_eee},\ref{cifar_10_time_eee} and Figure \ref{cifar_100_flop_eee},\ref{cifar_100_time_eee}, respectively. The slight sublinearity in Figure \ref{cifar_10_time_eee} and \ref{cifar_100_time_eee} occurs from the parallel implementation of ensembles. The relatively lower percentage of reduction of FLOPs and time in $\text{E}^3$-MoE compared to V-MoE occurs from the fact that the unpruned $\text{E}^3$-MoE is much smaller than the unpruned V-MoE in terms of parameters (167 million compared to 979 million).
\begin{figure}[ht]
    \begin{subfigure}[(a)]{0.32\linewidth}
        \centering
        \includegraphics[width=\linewidth]{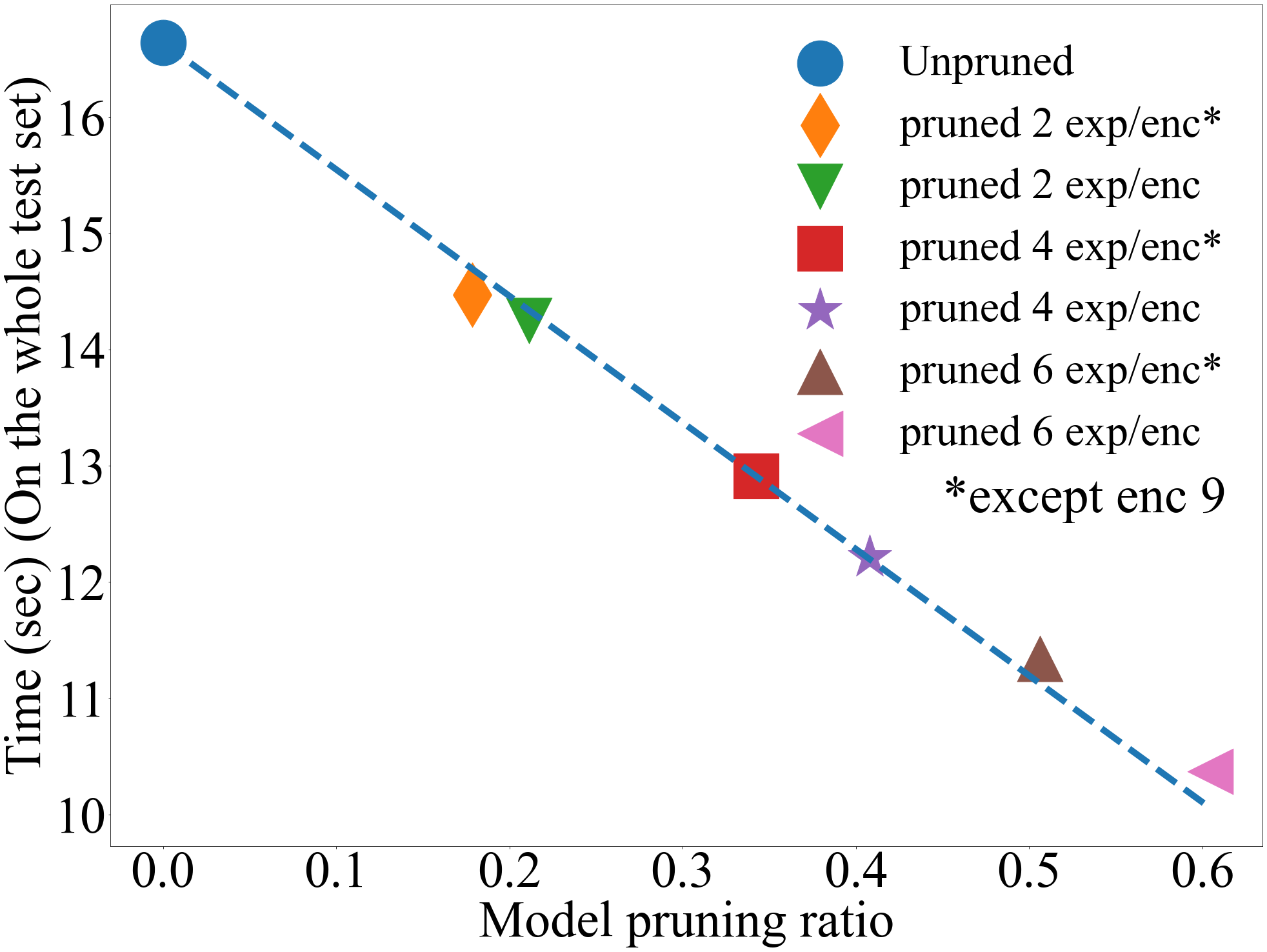}
        \caption{}
        \label{cifar_10_time}
    \end{subfigure}
    \hfill
    \begin{subfigure}[(b)]{0.32\linewidth}
        \centering
        \includegraphics[width=\linewidth]{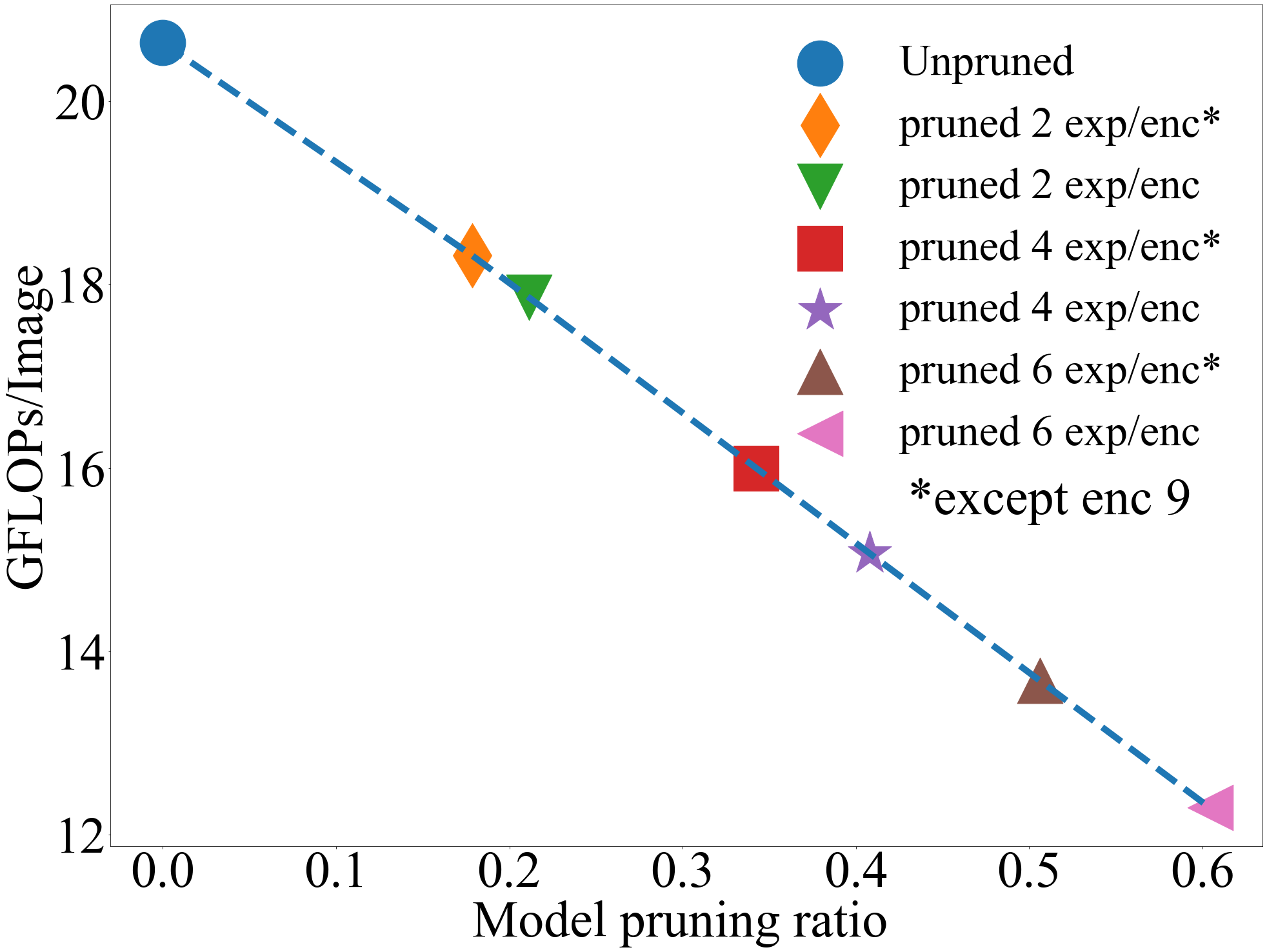}
        \caption{}
        \label{cifar_100_flop_vmoe}
    \end{subfigure}
    \hfill
    \begin{subfigure}[(c)]{0.32\linewidth}
        \centering
        \includegraphics[width=\linewidth]{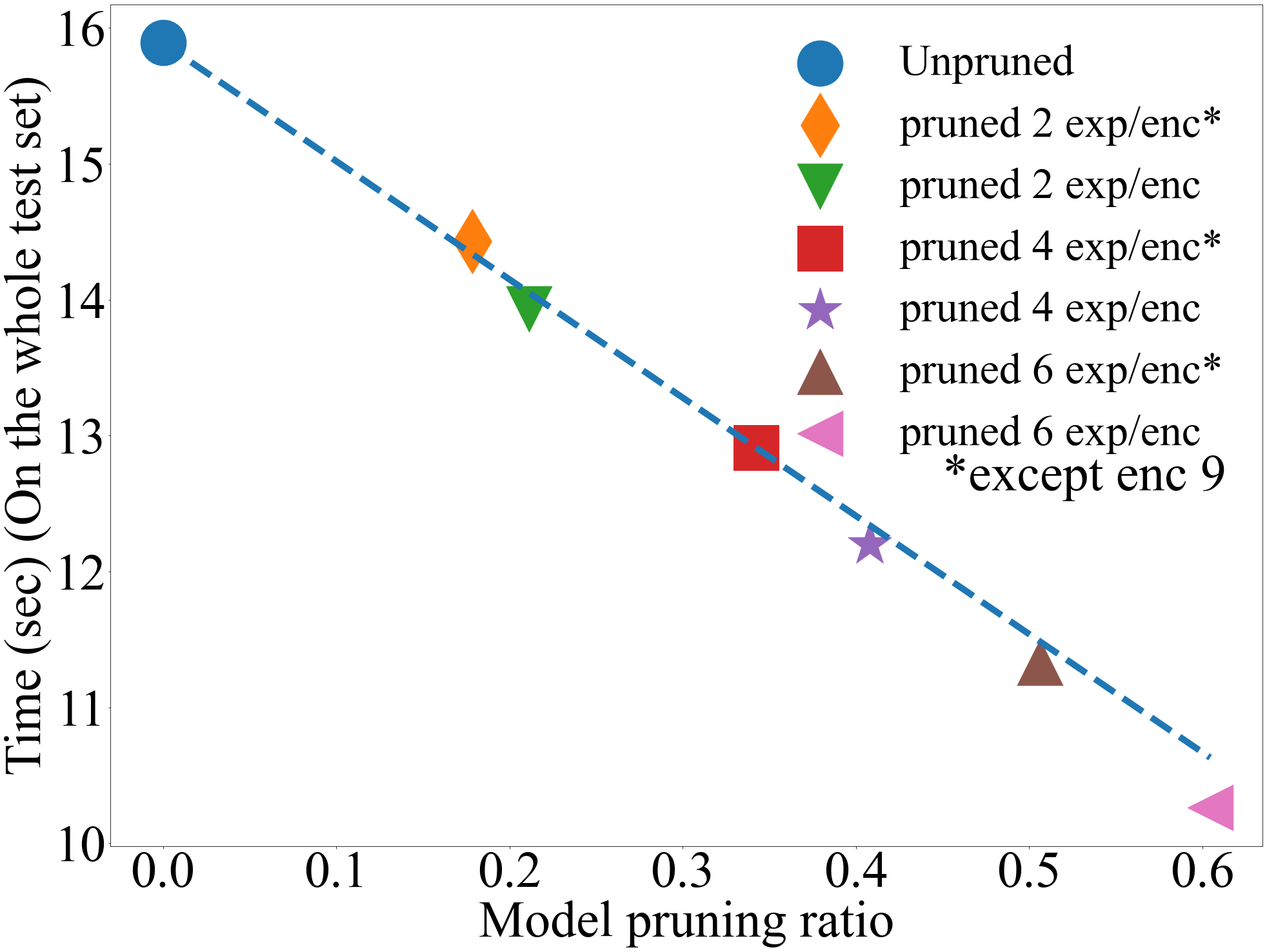}
        \caption{}
        \label{cifar_100_time_vmoe}
    \end{subfigure}
    \hfill
    \begin{subfigure}[(d)]{0.32\linewidth}
        \centering
        \includegraphics[width=\linewidth]{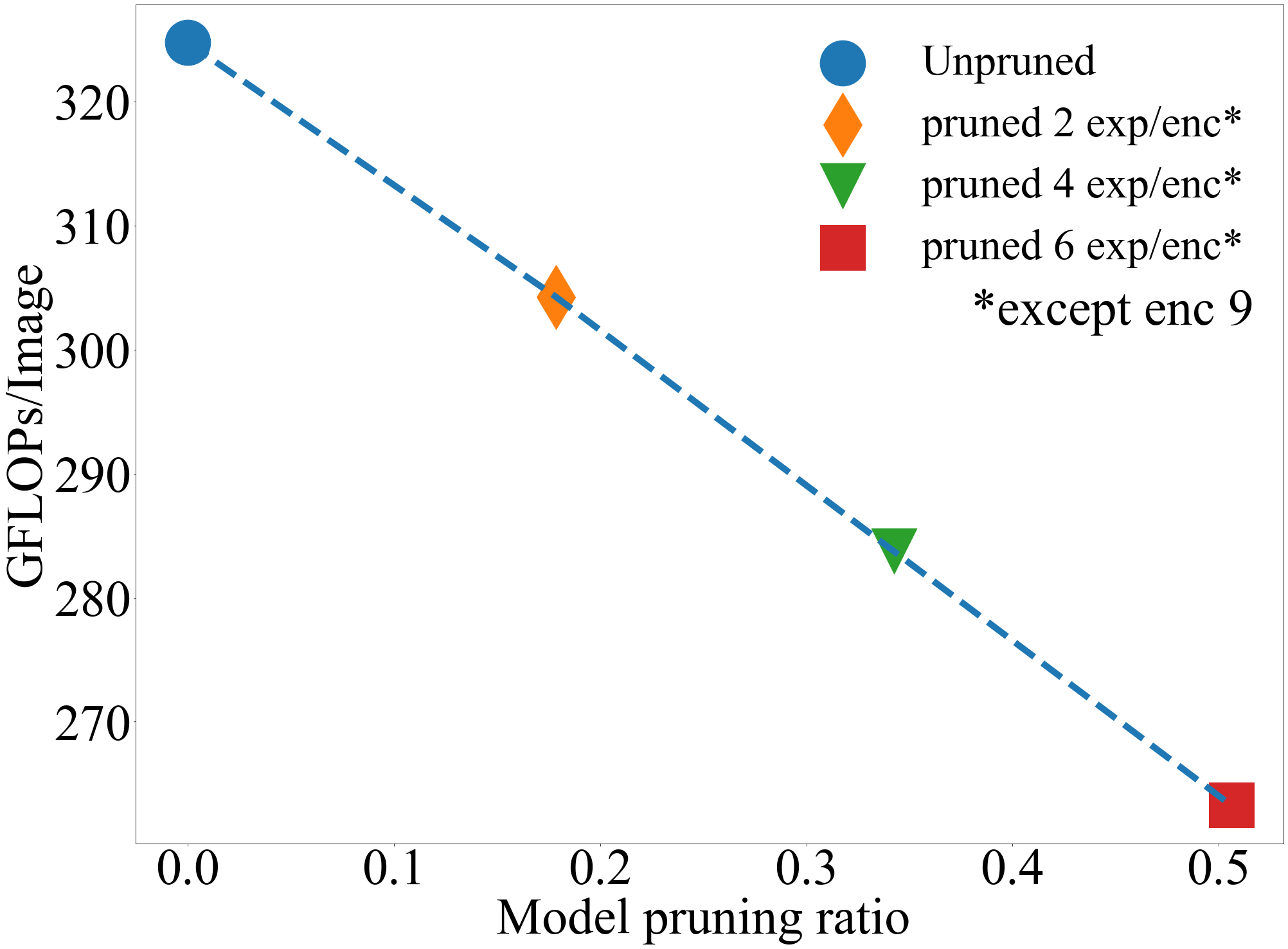}
        \caption{}
        \label{imagenet_flop}
    \end{subfigure}
    ~
    \begin{subfigure}[(e)]{0.32\linewidth}
        \centering
        \includegraphics[width=\linewidth]{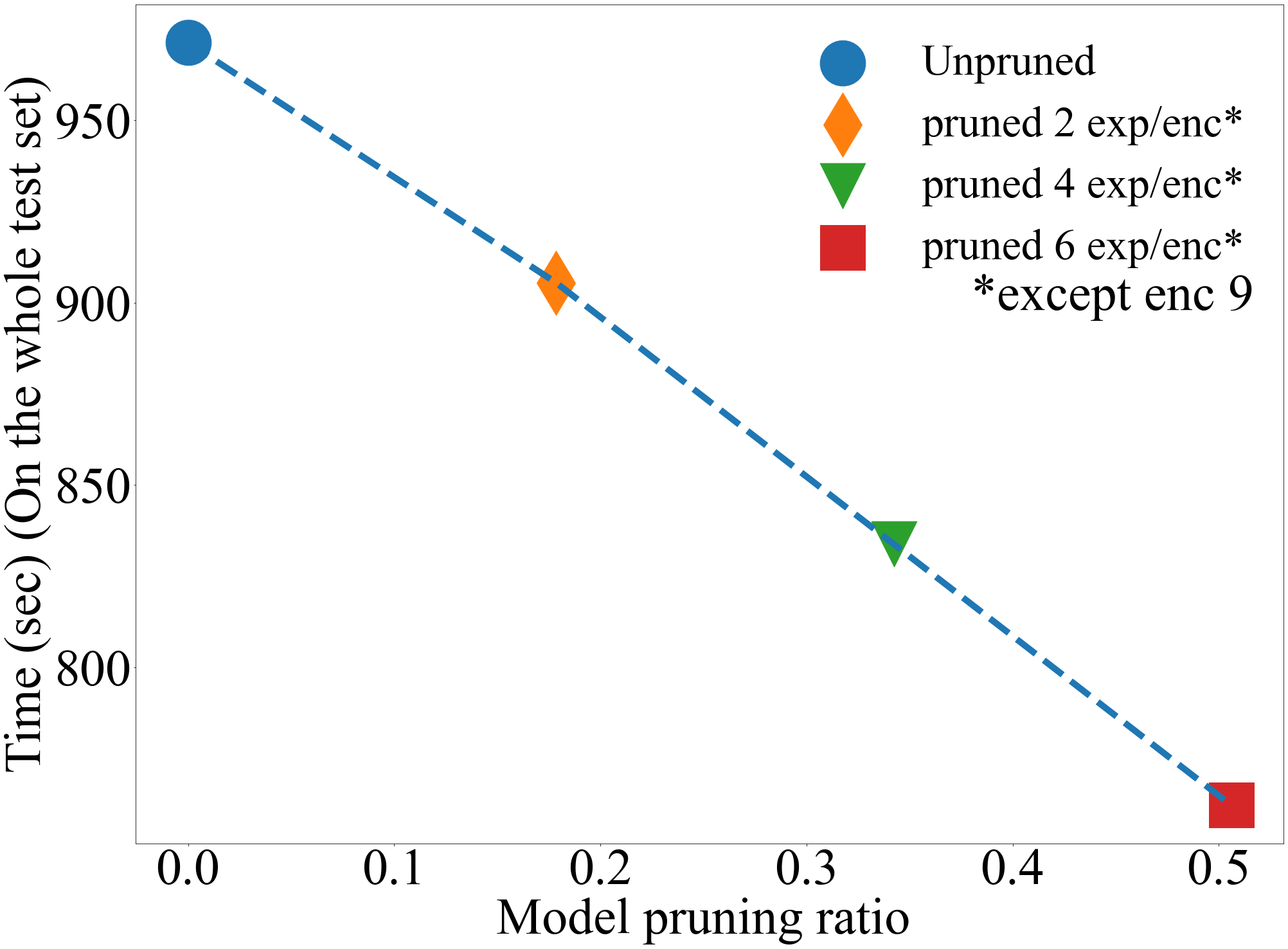}
        \caption{}
        \label{imagenet_time}
    \end{subfigure}
    \hfill
    \begin{subfigure}[(d)]{0.32\linewidth}
        \centering
        \includegraphics[width=\linewidth]{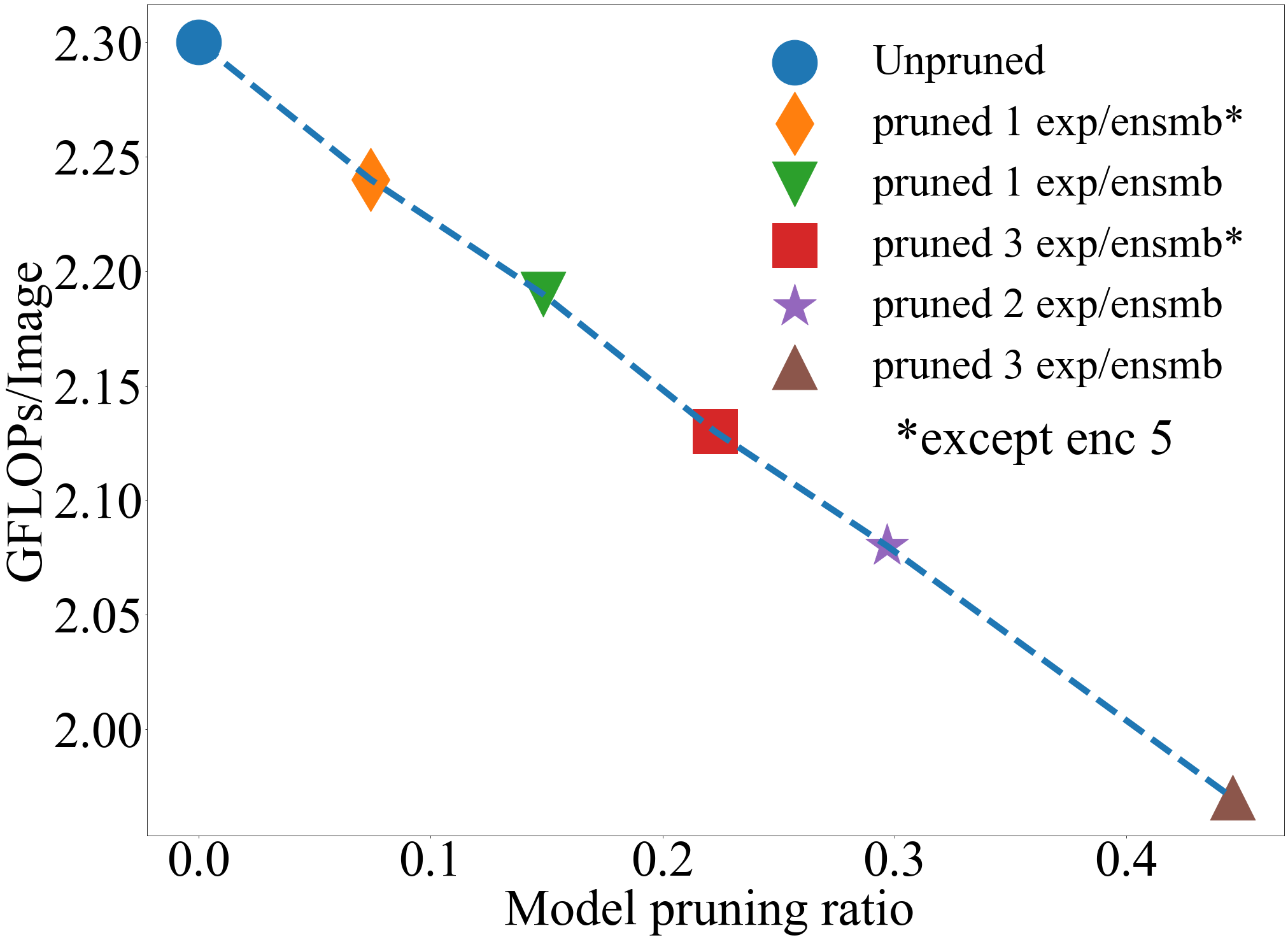}
        \caption{}
        \label{cifar_10_flop_eee}
    \end{subfigure}
    \hfill
    \begin{subfigure}[(d)]{0.32\linewidth}
        \centering
        \includegraphics[width=\linewidth]{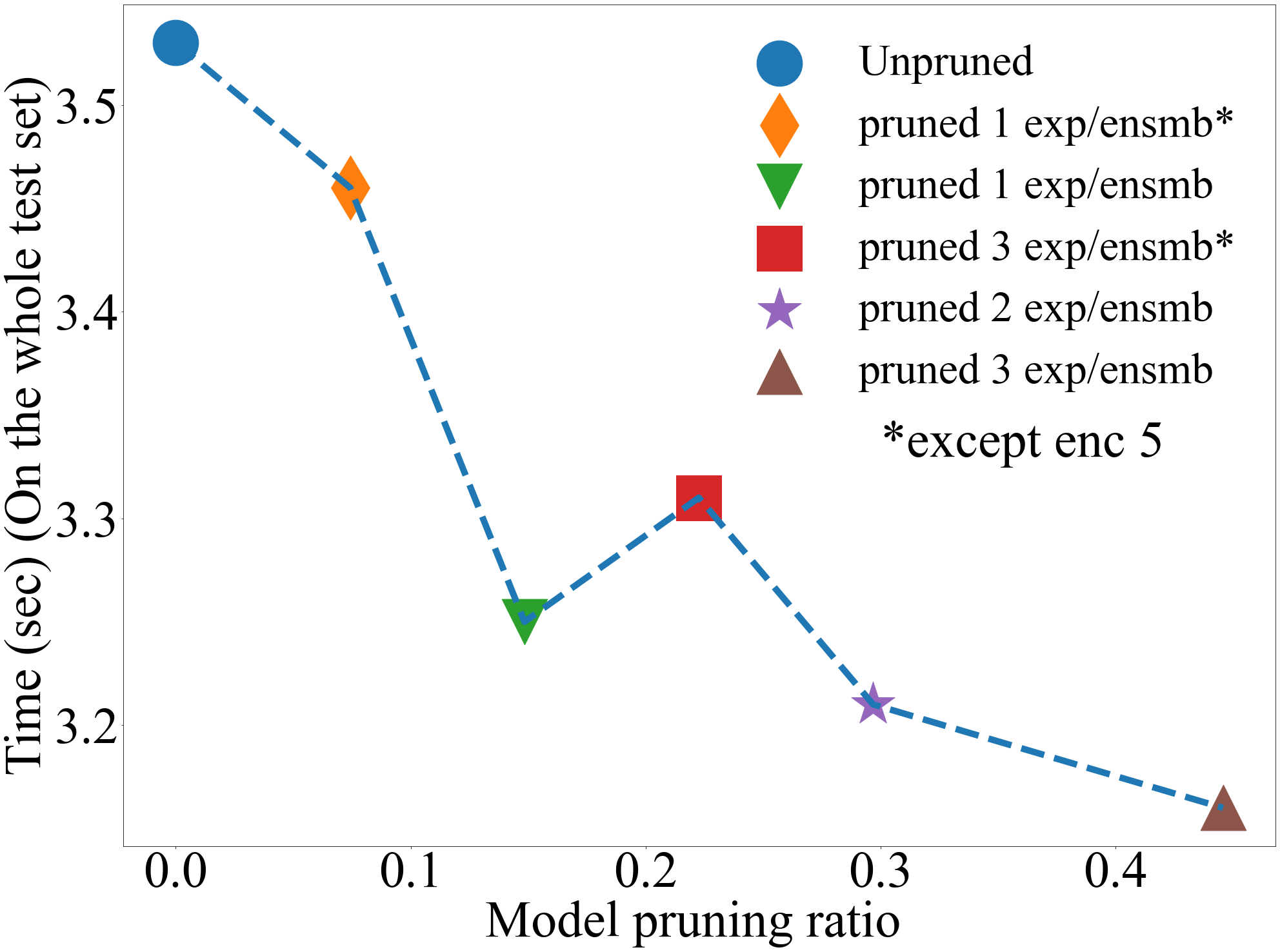}
        \caption{}
        \label{cifar_10_time_eee}
    \end{subfigure}
    \hfill
    \begin{subfigure}[(d)]{0.32\linewidth}
        \centering
        \includegraphics[width=\linewidth]{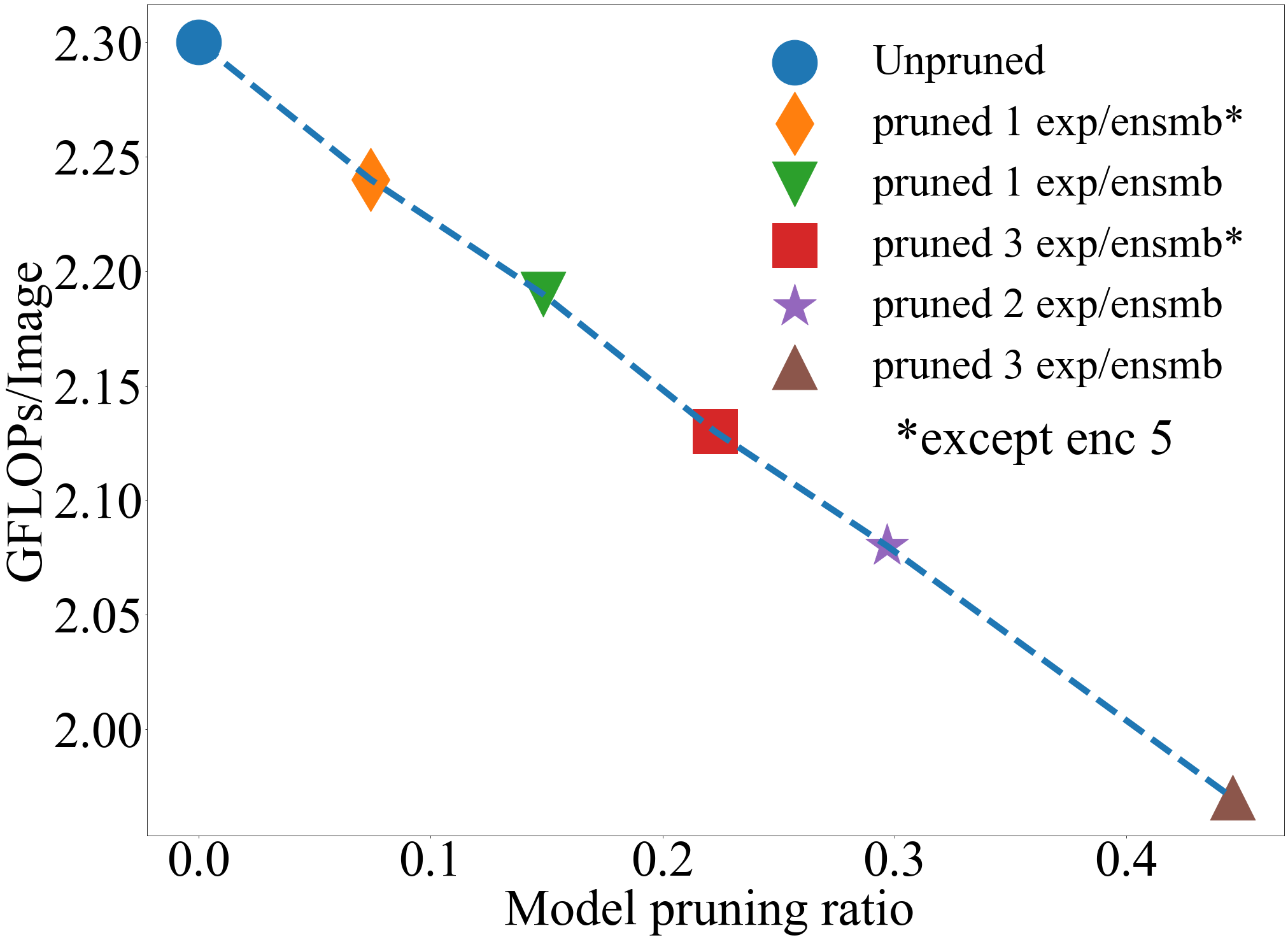}
        \caption{}
        \label{cifar_100_flop_eee}
    \end{subfigure}
    \hfill
    \begin{subfigure}[(d)]{0.32\linewidth}
        \centering
        \includegraphics[width=\linewidth]{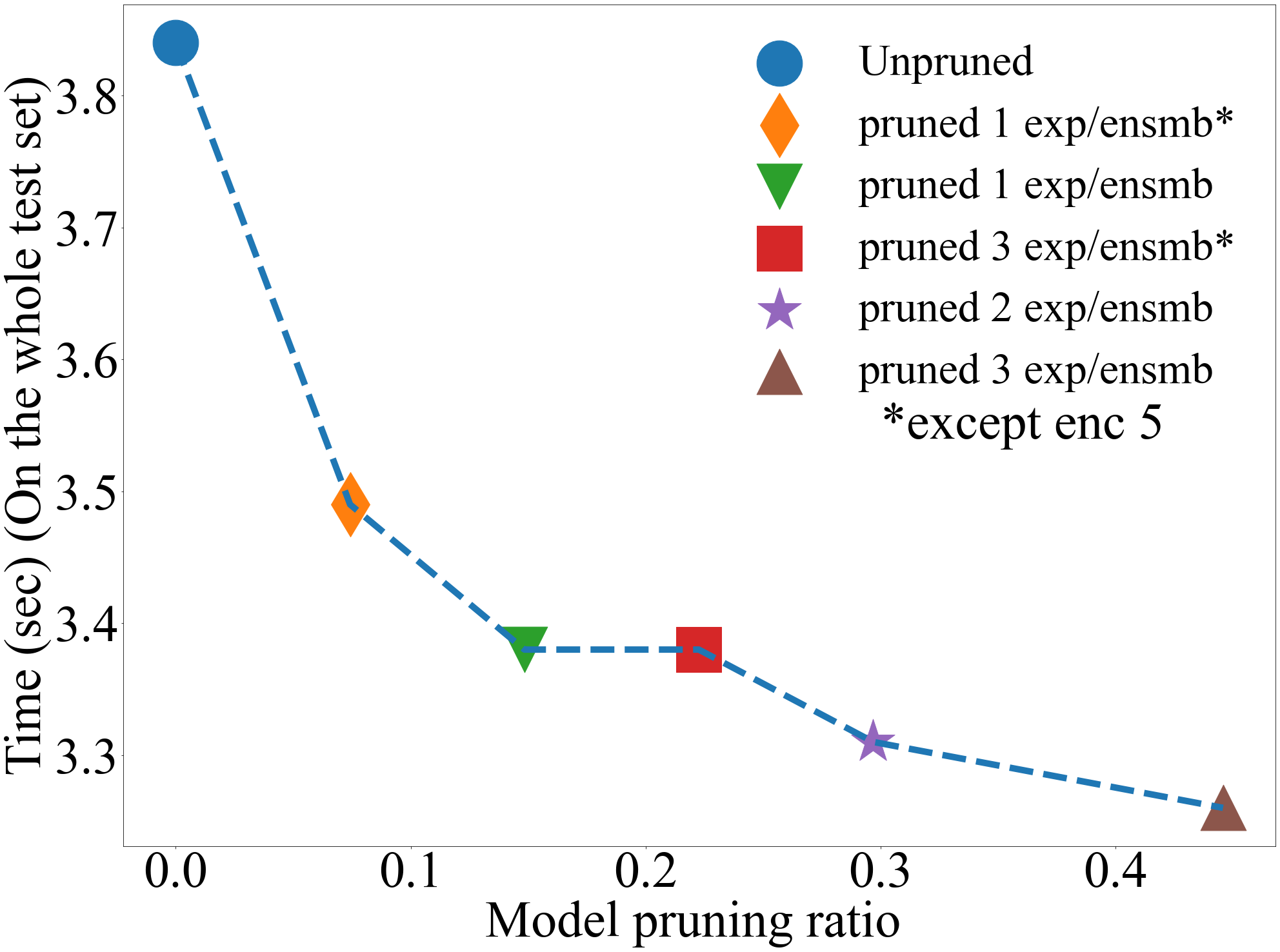}
        \caption{}
        \label{cifar_100_time_eee}
    \end{subfigure}
    \caption{Inference cost of the pruned models: (a) Inference time in V-MoE on CIFAR-10, (b) FLOPs per Image in V-MoE on CIFAR-100, (c) Inference time in V-MoE on CIFAR-100, (d) FLOPs per Image in V-MoE on ImageNet, (e) Inference time in V-MoE on ImageNet, (f) FLOPs per Image in $\text{E}^3$-MoE on CIFAR-10, (g) Inference time in $\text{E}^3$-MoE on CIFAR-10, (h) FLOPs per Image in $\text{E}^3$-MoE on CIFAR-100, (i) Inference time in $\text{E}^3$-MoE on CIFAR-100}
    \vspace{-1mm}
\end{figure}

\section{Preliminaries}
As we analyze the \textit{expert-choice} routing, for any fine-tuning step $t$, we re-write (\ref{analyzed_model}) as,

$f^{(t)}(x)=\sum_{s=1}^ka^{(s)}\sum_{j\in J_s^{(t)}(x)}G_j^{(s,t)}\sum_{r=1}^m\text{ReLU}\left(\langle w_r^{(s,t)},x^{(j)}\rangle\right)$ 

For any input $(x,y)$ and iteration $t$, the gradient of the hidden neuron $r\in[m]$ of the expert $s\in[k]$ is calculated as,

$\cfrac{\partial l^{(t)}(x,y)}{\partial w_r^{(s,t)}}=-ya^{(s)}\sum_{j\in J_s^{(t)}(x)}G_j^{(s,t)}x^{(j)}1_{\langle w_r^{(s,t)},x^{(j)}\rangle\ge0}$

and the gradient of the router of the expert $s\in[k]$ is calculated as,

$\cfrac{\partial l^{(t)}(x,y)}{\partial w_s^{(t)}}=-ya^{(s)}\sum_{j\in J_s^{(t)(x)}}\sigma_j^{(s,t)}G_j^{(s,t)}\sum_{i\in J_s^{(t)}(x)\backslash j}G_i^{(s,t)}(x^{(j)}-x^{(i)})$, where, $\sigma_j^{(s,t)}=\sum_{r=1}^m\text{ReLU}(\langle w_r^{(s,t)},x^{(j)}\rangle)$

We express the batch gradient of SGD at iteration $t$ for the batch $\mathcal{B}_t$ as, $\cfrac{\partial l}{\partial w_r^{(s,t)}}=\cfrac{1}{B}\sum_{x\in\mathcal{B}_t}\cfrac{\partial l^{(t)}(x,y)}{\partial w_r^{(s,t)}}$ for the expert and, $\cfrac{\partial l}{\partial w_s^{(t)}}=\cfrac{1}{B}\sum_{x\in\mathcal{B}_t}\cfrac{\partial l^{(t)}(x,y)}{\partial w_s^{(t)}}$  for the router. 

We present the pruning algorithm considered for analysis at Algorithm \ref{alg_1}\footnote{Here, $TOP_z(S)$ function provides the set of top $z$ values of the set $S$.}. 

\begin{algorithm}\caption{The Expert Pruning Algorithm for the Theoretical Analysis}\label{alg_1}
\textbf{Input} : Training data $\{(x_i,y_i)\}_{i=1}^N$, learning rates $\eta_r$ and $\eta_e$, number of iterations $T$ and $T^\prime$, batch- \\
\indent\hspace{1.2cm}size $B$, expert-pruning-ratio $\rho$\\
\textbf{Step-1}: Initialize the pre-trained weights $\{w_s^{(0)}, w_r^{(s,t)}, a^{(s)}\}_{s\in[k],r\in[m]}$\\
\textbf{Step-2}: for $t=0,1, ... ,T-1$ do:\\
\centerline{$w_{s}^{(t+1)}=w_{s}^{(t)}-\eta_r\cfrac{\partial l}{\partial w_s^{(t)}}, \forall s\in[k]$}\\
\centerline{$w_{r,s}^{(t+1)}=w_{r,s}^{(t)}-\eta_e\cfrac{\partial l}{\partial w_{r,s}^{(t)}}, \hspace{0.2cm}\forall r\in[m], s\in[k]$}\\
\textbf{Step-3}: Construct the set $S_{k^\prime}=\{s\in S_1: \Delta_s^{(T)}\in\text{TOP}_{(1-\rho)|S_1|}\left(\{\Delta_s\}_{s\in S_1}\right)\}\bigcup\{s\in S_2: \Delta_s^{(T)}\in\text{TOP}_{(1-\rho)|S_2|}\left(\{\Delta_s\}_{s\in S_2}\right)\}$\\
\textbf{Step-4}: for $t=0,1, ... ,T^\prime-1$ do:\\
\centerline{$w_{s}^{(t+1)}=w_{s}^{(t)}-\eta_r\cfrac{\partial l}{\partial w_s^{(t)}}, \forall s\in S_{k^\prime}$}\\
\centerline{$w_{r,s}^{(t+1)}=w_{r,s}^{(t)}-\eta_e\cfrac{\partial l}{\partial w_{r,s}^{(t)}}, \hspace{0.2cm}\forall r\in[m], s\in S_{k^\prime}$}\\
\end{algorithm}

\textbf{Notations:}
\begin{enumerate}
    \item $\Tilde{O}(\cdot)$ and $\Tilde{\Omega}(\cdot)$ hides factor $\log(\text{poly}(d))$ with a sufficiently large polynomial $poly(\cdot)$
    \item With high probability (abbreviated as \textit{w.h.p.}) refers to the probability $1-\cfrac{1}{\text{poly}(d)}$
\end{enumerate}

\textbf{Definitions:}\\

We denote, 
\begin{itemize}
    \item $G_1^{(s,t)}(x)$: Gating value of $o_1$ for the sample $x$ at the expert $s$ at time $t$
    \item $G_2^{(s,t)}(x)$: Gating value of $o_2$ for the sample $x$ at the expert $s$ at time $t$
    \item $G_q^{(s,t)}(x)$: Gating value of the task-irrelevant pattern $q\in\mathcal{P}$ for the sample $x$ at the expert $s$ at time $t$
    \item $l_q^{(s,t)}(x):=|\{j\in J_s^{(t)}(x): x^{(j)}=q\}|$, is the number of copies of the task-irrelevant pattern $q\in\mathcal{P}$ in the set of top $l$ tokens for the sample $x$ at the expert $s$ at time $t$
\end{itemize}

We restate the definition of the set, $S_1:=\{s\in[k]:a^{(s)}=+1\}$ and $S_2:=\{s\in[k]:a^{(s)}=-1\}$

We restate the definition of the probability measures,\\ $p_1^{(s,t)}:=\mathbb{P}\left[(x,y)\sim\mathcal{D}: \exists j\in J_s^{(t)}(x) \text{ s.t. } x^{(j)}=o_1 \text{ and } G_1^{(s,0)}\ge\cfrac{1}{l}|y=+1\right]$

and $p_2^{(s,t)}:=\mathbb{P}\left[(x,y)\sim\mathcal{D}: \exists j\in J_s^{(t)}(x) \text{ s.t. } x^{(j)}=o_2 \text{ and } G_2^{(s,0)}\ge\cfrac{1}{l}|y=-1\right]$

For any $x$, $\sigma_j^{(s,t)}(x):=\sum_{r=1}^m\text{ReLU}(\langle w_r^{(s,t)},x^{(j)}\rangle)$ and, for any pattern $q\in\mathcal{P}$, $\sigma_q^{(s,t)}:=\sum_{r=1}^m\text{ReLU}(\langle w_r^{(s,t)},q\rangle)$

$\sigma_1^{(s,t)}:=\sum_{r=1}^m\text{ReLU}(\langle w_r^{(s,t)},o_1\rangle)$ and $\sigma_2^{(s,t)}:=\sum_{r=1}^m\text{ReLU}(\langle w_r^{(s,t)},o_2\rangle)$

For any task-irrelevant pattern $q\in\mathcal{P}$, $\delta_{1,q}^{(s,t)}(x):=\sigma_1^{(s,t)}-\sigma_q^{(s,t)}$ and $\delta_{2,q}^{(s,t)}(x):=\sigma_2^{(s,t)}-\sigma_q^{(s,t)}$. Therefore, $\delta_{q,1}^{(s,t)}=-\delta_{1,q}^{(s,t)}$ and $\delta_{q,2}^{(s,t)}=-\delta_{2,q}^{(s,t)}$

For any two different task-irrelevant patterns $q,q^\prime\in\mathcal{P}$, $\delta_{q,q^\prime}^{(s,t)}:=\sigma_q^{(s,t)}-\sigma_{q^\prime}^{(s,t)}$

Note that, for any $x$ and $q$, $l_q^{(s,t)}(x)G_q^{(s,t)}\le1$.

Let us define, $C_1:=\max\{\left|\left|w_s^{(0)}\right|\right|\}_{s\in[k]}$ and $C_2:=\max\{\left|\left|w_r^{(s,0)}\right|\right|\}_{s\in[k],r\in[m]}$.

WLOG, we analyze the case where $l\ge e^{2C_1}$. Therefore, $\forall s\in[k]$, $\left|\left|w_s^{(0)}\right|\right|\le\cfrac{1}{2}\log l$.

\textbf{Assumptions on the pre-trained model}.
Based on the experts' proficiency measure described in section \ref{setup}, we define the important and unimportant experts for the downstream task in the pre-trained model.
\begin{definition}[Important and Unimportant experts in pre-trained model\footnote{Our definitions of important and unimportant experts are minimal in the sense that (I) for an expert to be important for a particular downstream pattern, it should learn to route the pattern with at least the uniform gating value ($1/l$) for a minimal fraction of samples ($p_1^{(s,0)}=\Omega(1)$) during pre-training and, (II) if an expert does not learn to route the pattern at all ($p_1^{(s,0)}=O(1/d)$), it should be unimportant for the pattern.}]
    An expert $s\in[k]$ is \textit{important} for the class-1 task-specific feature $o_1$ if $p_1^{(s,0)}=\Omega(1)$, and \textit{unimportant} for the feature if $p_1^{(s,0)}=O(1/d)$. 
    Similarly, an expert $s\in[k]$ is \textit{important} for the class-2 task-specific feature $o_2$ if $p_2^{(s,0)}=\Omega(1)$, and \textit{unimportant} for the feature if $p_2^{(s,0)}=O(1/d)$.
\end{definition}
We assume the existence of at least one class-1 important expert in the set $S_1$ and at least one class-2 important expert in the set $S_2$ of the pre-trained model. This assumption is minimal as we need only one important expert per class.

For any class-1 important expert $s$, we assume that the router's component along $o_1$ (i.e., $\langle w_s^{(0)},o_1\rangle$) is minimally separated ($\Omega(1/d)$) from the components along any irrelevant pattern $q\in\mathcal{P}$ (i.e., $\langle w_s^{(0)},q\rangle$), i.e.,\\
$\left|\langle w_s^{(0)},o_1-q\rangle\right|=C_p \text{ where, } C_p=\Theta(1/d)$, for any $s$ such that $p_1^{(s,0)}=\Omega(1)$ and, for any $q\in\mathcal{P}$.

We further assume that, for any class-1 important expert $s$, at least a fraction of the $m$ neurons (i.e., $\Omega(1)$ fraction) of the expert are minimally activated (i.e., $\langle w_r^{(s,0)},o_1\rangle\ge 0$) by $o_1$, i.e.,\\
$\left|\{r\in[m]:\langle w_r^{(s,0)},o_1\rangle\ge 0\}\right|=\Omega(1)$, for any $s$ such that $p_1^{(s,0)}=\Omega(1)$.

We assume that similar assumptions hold for the class-2 important experts. Again, these two assumptions are minimal for an expert to be important for a task-specific pattern.

\textbf{Components of the Router-gradients:}

We calculate the components of the gradient of the router $s\in[k]$ for the input $(x,y)$ along the task-specific patterns, i.e., $o_1$ and $o_2$ and along any task-irrelevant pattern $q\in\mathcal{P}$ at time $t$ as follows:

\begin{equation}\label{eq_a_1}
    \left\langle \cfrac{\partial l(x,y)}{\partial w_s^{(t)}},q\right\rangle=
\begin{cases}
    0 & \text{if $\not\exists j\in J_s^{(t)}(x)$ s.t. $x^{(j)}=q$}\\\\
    ya^{(s)}l_q^{(s,t)}G_q^{(s,t)}\sum_{j\in J_s^{(t)}(x)/\{j:x^{(j)}=q\}}G_j^{(s,t)}\left(\sigma_j^{(s,t)}-\sigma_q^{(s,t)}\right) & \text{if $\exists j \in J_s^{(t)}(x)$ s.t. $x^{(j)}=q$}
\end{cases}
\end{equation}
\begin{equation}\label{eq_a_2}
    \left\langle \cfrac{\partial l(x,y)}{\partial w_s^{(t)}},o_1\right\rangle=
\begin{cases}
    0 & \text{if $y=-1$}\\\\
    0 & \text{if $y=+1$ but}\\ &\text{$\not\exists j\in J_s^{(t)}(x)$ s.t.}\\ &\text{$x^{(j)}=o_1$}\\\\
    a^{(s)}G_1^{(s,t)}(x)\sum_{i\in J_s^{(t)}(x)/\{j:x^{(j)}=o_1\}}G_j^{(s,t)}\left(\sigma_j^{(s,t)}-\sigma_1^{(s,t)}\right) & \text{if $y=+1$ and}\\&\text{$\exists j \in J_s^{(t)}(x)$ s.t.}\\ &\text{$x^{(j)}=o_1$}
\end{cases}
\end{equation}
\begin{equation}\label{eq_a_3}
    \left\langle \cfrac{\partial l(x,y)}{\partial w_s^{(t)}},o_2\right\rangle=
\begin{cases}
    0 & \text{if $y=+1$}\\\\
    0 & \text{if $y=-1$ but}\\ &\text{$\not\exists j\in J_s^{(t)}(x)$ s.t.}\\ &\text{$x^{(j)}=o_2$}\\\\
    -a^{(s)}G_2^{(s,t)}(x)\sum_{i\in J_s^{(t)}(x)/\{j:x^{(j)}=o_2\}}G_j^{(s,t)}\left(\sigma_j^{(s,t)}-\sigma_2^{(s,t)}\right) & \text{if $y=-1$ and}\\&\text{$\exists j \in J_s^{(t)}(x)$ s.t.}\\ &\text{$x^{(j)}=o_2$}
\end{cases}
\end{equation}

\textbf{Components of the Neuron-gradients:}

We calculate the components of the gradient of the hidden neuron $r\in[m]$ of the expert $s\in[k]$ for the input $(x,y)$ along the task-specific patterns, i.e., $o_1$ and $o_2$ and along any task-irrelevant pattern $q\in\mathcal{P}$ at time $t$ as follows:

\begin{equation}\label{eq_a_4}
    \left\langle \cfrac{\partial l(x,y)}{\partial w_r^{(s,t)}},q\right\rangle=
\begin{cases}
    0 & \text{if $\langle w_r^{(s,t)},q\rangle<0$}\\\\
    0 & \text{if $\langle w_r^{(s,t)},q\rangle\ge0$ but $\not\exists j \in J_s^{(t)}(x)$ s.t. $x^{(j)}=q$}\\\\
    -ya^{(s)} l_q^{(s,t)}G_q^{(s,t)} & \text{if $\langle w_r^{(s,t)},q\rangle\ge0$, $\exists j \in J_s^{(t)}(x)$ s.t. $x^{(j)}=q$}
\end{cases}
\end{equation}

\begin{equation}\label{eq_a_5}
    \left\langle \cfrac{\partial l(x,y)}{\partial w_r^{(s,t)}},o_1\right\rangle=
\begin{cases}
    0 & \text{if $y=-1$}\\\\
    0 & \text{if $y=+1$ but $\langle w_r^{(s,t)},o_1\rangle<0$}\\\\
    0 & \text{if $y=+1$ and $\langle w_r^{(s,t)},o_1\rangle\ge0$ but $\not\exists j \in J_s^{(t)}(x)$ s.t. $x^{(j)}=o_1$}\\\\
    -a^{(s)}G_1^{(s,t)}(x) & \text{if $y=+1$, $\langle w_r^{(s,t)},o_1\rangle\ge0$ and $\exists j \in J_s^{(t)}(x)$ s.t. $x^{(j)}=o_1$}
\end{cases}
\end{equation}

\begin{equation}\label{eq_a_6}
    \left\langle \cfrac{\partial l(x,y)}{\partial w_r^{(s,t)}},o_2\right\rangle=
\begin{cases}
    0 & \text{if $y=+1$}\\\\
    0 & \text{if $y=-1$ but $\langle w_r^{(s,t)},o_2\rangle<0$}\\\\
    0 & \text{if $y=-1$ and $\langle w_r^{(s,t)},o_2\rangle\ge0$ but $\not\exists j \in J_s^{(t)}(x)$ s.t. $x^{(j)}=o_2$}\\\\
    a^{(s)}G_2^{(s,t)}(x) & \text{if $y=-1$, $\langle w_r^{(s,t)},o_2\rangle\ge0$ and $\exists j \in J_s^{(t)}(x)$ s.t. $x^{(j)}=o_2$}
\end{cases}
\end{equation}

\section{Proof of Lemma \ref{lemma_1}}

\begin{lemma}[Full version of the \textbf{Lemma \ref{lemma_1}}]\label{lemma_m_1}
Suppose the expert learning rate $\eta_e$ such that, the router learning rate  $\eta_r=O\left(\cfrac{\eta_eC_p}{mdl^2C_2^2}\right)$,  the batch-size $B=\Tilde{\Omega}(l^2d^2)$, and the number of iterations $T=\Omega\left( \cfrac{l^2C_2}{\eta_e}\sqrt{\cfrac{\log l}{C_p}}\right)$. Then,

    For any expert $s\in S_1$ such that $p_1^{(s,0)}=\Omega(1)$, 
    we have
    
    (i) $p_1^{(s,T)}=1$, \\
    (ii) for every $(x,+1)\sim\mathcal{D}$, $G_j^{(s,T)}(x)>1/2$, if $x^{(j)}=o_1$,\\
    (iii) $\langle w_r^{(s,T)},o_1\rangle=\Omega( lC_2\sqrt{d\log l})$, for a constant fraction  $r\in[m]$,\\ 
    (iv) $\Delta_s^{(T)}>\cfrac{3}{2}\log l$.
    
    and, for any expert $s\in S_2$ such that $p_2^{(s,0)}=\Omega(1)$, we have
    
    (v) $p_2^{(s,T)}=1$, \\
    (vi) for every $(x,-1)\sim\mathcal{D}$, $G_j^{(s,T)}(x)>1/2$, if $x^{(j)}=o_2$,\\
    (vii) $\langle w_r^{(s,T)},o_2\rangle=\Omega( lC_2\sqrt{d\log l})$, for a constant fraction  $r\in[m]$,\\ 
    (viii) $\Delta_s^{(T)}>\cfrac{3}{2}\log l$.
\end{lemma}
\begin{proof}
    For any expert $s\in S_1$ such that $p_1^{(s,0)}=\Omega(1)$, from Lemma \ref{lemma_a_4}, for any task-irrelevant pattern $q$, $\delta_{1,q}^{(s,T^{\prime\prime})}=\Omega\left(mC_2\right)$ and $|\delta_{(2,q)}^{(s,T^{\prime\prime})}|=O(mC_2)$. Furthermore, for any different two task-irrelevant pattern $q$ and $q^\prime$, $|\delta_{q,q^\prime}^{(s,T^{\prime\prime})}|=O(mC_2)$\\

Therefore, $\langle \cfrac{\partial l}{\partial w_s^{(T^{\prime\prime})}},o_1\rangle\le-\Omega(\cfrac{mC_2}{l})+\Tilde{O}(\cfrac{mC_2}{l\sqrt{B}})$,\\ as from Lemma \ref{lemma_a_4}, $p_1^{(s,T^{\prime\prime})}=p_1^{(s,0)}$, $\forall q$ s.t. $\langle w_s^{(T^{\prime\prime})},q\rangle<\langle w_s^{(T^{\prime\prime})},o_1\rangle$; $\langle w_s^{(T^{\prime\prime})},o_1-q\rangle<2\log l$ which implies from Lemma \ref{lemma_a_6} that, $\forall (x,y=+1)\sim\mathcal{D}$ s.t. $\exists j\in J_s^{(T^{\prime\prime})}(x)$ with $x^{(j)}=q$, $G_1^{(s,T^{\prime\prime})}(x)(1-G_1^{(s,T^{\prime\prime})}(x))\ge\cfrac{1}{4l}$.\\

Therefore, $\langle w_s^{(T^{\prime\prime}+1)},o_1\rangle\ge\langle w_s^{(T^{\prime\prime})},o_1\rangle+\Omega(\cfrac{mC_2}{l}\eta_r)-\Tilde{O}(\cfrac{mC_2}{l\sqrt{B}}\eta_r)$.\\

Now, for any $q$, $\langle \cfrac{\partial l}{\partial w_s^{(T^{\prime\prime})}},q\rangle\ge-O(\cfrac{mC_2}{d})-\Tilde{O}{(\cfrac{mC_2}{\sqrt{B}})}$.\\

Therefore, for any $q$, $\langle w_s^{(T^{\prime\prime}+1)},q\rangle\le\langle w_s^{(T^{\prime\prime})},q\rangle+O(\cfrac{mC_2}{d}\eta_r)+\Tilde{O}({\cfrac{mC_2}{\sqrt{B}}}\eta_r)$.\\

Therefore, for any $q$, $\langle w_s^{(T^{\prime\prime}+1)},o_1-q\rangle\ge\langle w_s^{(T^{\prime\prime})},o_1-q\rangle+\Omega(\cfrac{mC_2}{l}\eta_r)$ as we are selecting $B=\Tilde{\Omega}(l^2d^2)$.\\

Now, from Lemma \ref{lemma_a_4}, for any $q$ s.t. $\langle w_s^{(T^{\prime\prime})},q\rangle<\langle w_s^{(T^{\prime\prime})},o_1\rangle$, $\langle w_s^{(T^{\prime\prime})},o_1-q\rangle=\Omega(C_p)$, which implies $\langle w_s^{(T^{\prime\prime}+1)},o_1-q\rangle\ge\Omega(C_p)+\Omega(\cfrac{mC_2}{l}\eta_r)$.\\

Therefore, for any $q$ s.t. $\langle w_s^{(T^{\prime\prime})},q\rangle<\langle w_s^{(T^{\prime\prime})},o_1\rangle$, $\langle w_s^{(T^{\prime\prime}+1)},o_1\rangle>\langle w_s^{(T^{\prime\prime}+1)},q\rangle$ and hence $p_1^{(s,T^{\prime\prime}+1)}\ge p_1^{(s,T^{\prime\prime})}=p_1^{(s,0)}$.\\

Now, as $p_1^{(s,T^{\prime\prime})}=p_1^{(s,0)}$, using the same procedure as in the proof of Lemma \ref{lemma_a_4}, $\forall q$, $ \delta_{1,q}^{(s,T^{\prime\prime}+1)}=\delta_{1,q}^{(s,T^{\prime\prime})}+\Omega(\cfrac{m\eta_e}{l})$ and, $\delta_{2,q}^{(s,T^{\prime\prime}+1)}\le\delta_{2,q}^{(s,T^{\prime\prime})}=O(mC_2)$. Furthermore, for any two different task-irrelevant patterns, $|\delta_{q,q^\prime}^{(s,T^{\prime\prime}+1)}-\delta_{q,q^\prime}^{(s,T^{\prime\prime})}|=O(\cfrac{m}{d}\eta_e)$.\\

Hence, by induction, $\forall t\ge T^{\prime\prime}$ such that, for any $q$ with $\langle w_s^{(t)},q\rangle<\langle w_s^{(t)},o_1\rangle$ satisfies $\langle w_s^{(t)},o_1-q\rangle\le2\log l$, $p_1^{(s,t)}\ge p_1^{(s,0)}$ and hence,\\

for all $q$,
\begin{equation}\label{eq: 1}
    \delta_{1,q}^{(s,t+1)}-\delta_{1,q}^{(s,t)}=\Omega(\cfrac{m\eta_e}{l})
\end{equation}
$\delta_{2,q}^{(s,t)}=O(mC_2)$, $|\delta_{q,q^\prime}^{(s,t)}|=O(mC_2)$ and hence for all $q$,
\begin{equation}\label{eq: 2}
    \langle w_s^{(t+1)},o_1-q\rangle\ge\langle w_s^{(T^{\prime\prime})},o_1-q\rangle+\Omega(\cfrac{1}{l}mC_2\eta_r)(t-T^{\prime\prime})+\Omega(\cfrac{1}{l^2}m\eta_e\eta_r)\cfrac{(t-T^{\prime\prime})^2-(t-T^{\prime\prime})}{2}
\end{equation}

Now, from Lemma \ref{lemma_a_4}, for all $q$, $\langle w_s^{(T^{\prime\prime})},o_1-q\rangle<\log l$. Let us assume that upto $T$ iterations, for all $q$, $\langle w_s^{(T)},o_1-q\rangle\le2\log l$.\\

Now, if $\exists q^\prime$ s.t. $\langle w_s^{(T)},o_1-q^\prime\rangle\le\log l$, then $\forall q\neq q^\prime$,
\begin{equation}\label{eq: 3}
    \langle w_s^{(T)},o_1-q\rangle\le\langle w_s^{(T)},o_1-q^\prime\rangle+\langle w_s^{(T)},q^\prime-q\rangle\le\langle w_s^{(T)},q^\prime-q\rangle+\log l
\end{equation}

On the other hand, from inequality (\ref{eq: 2}), by selecting $T=\Omega(\cfrac{l^2C_2}{\eta_e}\sqrt{\cfrac{\log l}{C_p}})$ we get,\\ $\langle w_s^{(T+1)},o_1-q^\prime\rangle>\log l$.\\

Now, $\forall q$, and $\forall t$, $\delta_{1,q}^{(s,t)}-\delta_{1,q}^{(s,t-1)}=O(m\eta_e)$ and $\delta_{2,q}^{(s,t-1)}-\delta_{2,q}^{(s,t)}=O(m\eta_e)$ and $|\delta_{q,q^\prime}^{(s,t-1)}-\delta_{q,q^\prime}^{(s,t)}|=O(\cfrac{m}{d}\eta_e)$.\\

Hence, $\langle w_s^{(t+1)},q\rangle\ge\langle w_s^{(0)},q\rangle-O(\cfrac{mC_2}{d}\eta_r)t-O(\cfrac{m}{d}\eta_e\eta_r)\cfrac{t^2-t}{2}$.\\

On the other hand, $\langle w_s^{(t+1)},q^\prime\rangle\le\langle w_s^{(0)},q^\prime\rangle+O(\cfrac{mC_2}{d}\eta_r)t$.\\

Therefore,
\begin{equation}\label{eq: 4}
    \langle w_s^{(t+1)},q^\prime-q\rangle\le\langle w_s^{(0)},q^\prime-q\rangle+O(\cfrac{mC_2}{d}\eta_r)t+O(\cfrac{m}{d}\eta_e\eta_r)t^2
\end{equation}

Now, for the choice of $T=O(\cfrac{l^2C_2}{\eta_e}\sqrt{\cfrac{\log l}{C_p}})$ and as $\langle w_s^{(0)},q^\prime-q\rangle\le\cfrac{1}{2}\log l$, using inequality (\ref{eq: 4}) we get,
\begin{align*}
    \langle w_s^{(T)},q^\prime-q\rangle\le\cfrac{1}{2}\log l+O(\cfrac{l^2}{d}\log l)\le\log l
\end{align*}

Therefore, from inequality (\ref{eq: 3}), for the choice of $T=O(\cfrac{l^2C_2}{\eta_e}\sqrt{\cfrac{\log l}{C_p}})$, $\forall q$, $\langle w_s^{(T)},o_1-q\rangle\le2\log l$ holds.\\

Now, as for any $q$, $\langle w_s^{(T+1)},o_1-q\rangle>\log l$, $p_1^{(s,T+1)}=1$ and $\forall (x,y=+1)\sim\mathcal{D}$, $G_1^{(s,T+1)}(x)>\cfrac{1}{2}$.\\

Now, as for any $t$, $p_1^{(s,T^{\prime\prime})}\ge p_1^{(s,0)}$, at least for $\Omega(1)$ fraction of $r\in[m]$ of $s$, $\langle w_r^{(s,t)},o_1\rangle\ge\langle w_r^{s,0},o_1\rangle+\Omega\left(\cfrac{\eta_e}{l}t\right)$. Therefore, for $T=O(\cfrac{l^2C_2}{\eta_e}\sqrt{\cfrac{\log l}{C_p}})$, $\langle w_r^{(s,T)},o_1\rangle=\Omega( lC_2\sqrt{d\log l})$ for $\Omega(1)$ fraction of $r\in[m]$ of $s$.  

Now, $\left|\left| w_s^{(T+1)}\right|\right|=\left(\langle w_s^{(T+1)},o_1\rangle^2+\langle w_s^{(T+1)},o_2\rangle^2+\sum_{i=1}^{d-2}\langle w_s^{(T+1)},q_i\rangle^2\right)^{\cfrac{1}{2}}\ge\langle w_s^{(T+1)},o_1\rangle>2\log l$ where, $q_i$ denotes the $i$-th task-irrelevant pattern. Now, as $\left|\left|w_s^{(0)}\right|\right|\le\cfrac{1}{2}\log l$, $\Delta_s^{(T)}>\cfrac{3}{2}\log l$.

Similarly, using Lemma \ref{lemma_a_5}, for $s\in S_2$ such that, $p_2^{(s,0)}=\Omega(1)$, we can proof the statements (v), (vi), (vii) and (viii).
\end{proof}

\section{Lemmas Used to Prove the Lemma \ref{lemma_1}}
\begin{lemma}\label{lemma_a_1}
    Let, $S\subset \mathcal{D}$ such that $p:=\mathbb{P}\left[(x,y)\sim{D}: (x,y)\in S\right]$. Then, w.h.p. over any randomly sampled batch $\mathcal{B}_t$ of size $B$ at the iteration $t$, $\bigg|\big|\mathcal{B}_t\cap S\big|-Bp\bigg|=\Tilde{O}\left(\sqrt{B}\right)$.
\end{lemma}

\begin{proof}
    Let us define a random variable $X$ associated with any sample $(x,y)\sim\mathcal{D}$ such that,\\ $X:=\begin{cases}
       1 & \text{if $(x,y)\in S $}\\
       0 & \text{if $(x,y)\not\in S$}
    \end{cases}$\\
    
    Therefore, $X\sim\text{Ber}(p)$.\\
    
    Now, for any randomly sampled batch $\mathcal{B}_t:=\{(x_1,y_1),(x_2,y_2), ..., (x_B,y_B)\}$ of size $B$, we can denote the $B$ i.i.d. random variables following the same distribution as $X$ by $X_1, X_2, ..., X_B$ corresponding to the $B$ samples of the batch, respectively.\\
    
    Therefore, $\big|\mathcal{B}_t\cap S \big|=\sum_{i=1}^BX_i$.\\
    
    Now, $\mathbb{E}\left[\big|\mathcal{B}_t\cap S \big|\right]=\sum_{i=1}^B\mathbb{E}\left[X_i\right]=Bp$.\\
    
    Therefore, using the Hoeffding's inequality, $\mathbb{P}\left[\bigg|\big|\mathcal{B}_t\cap S\big|-Bp\bigg|=\Tilde{O}\left(\sqrt{B}\right)\right]\ge1-\cfrac{1}{\text{poly}(d)}$ which completes the proof.
\end{proof}

\begin{lemma}\label{lemma_a_2}
    For any expert $s\in S_1$ such that $p_1^{(s,0)}=\Omega(1)$, \textit{w.h.p.} over a randomly sampled batch of size $B$ we can ensure that,

    (i) $\left|\langle \cfrac{\partial l}{\partial w_s^{(0)}},q\rangle\right|\le O\left(\cfrac{mC_2}{d}\right)+\Tilde{O}\left(\cfrac{mC_2}{\sqrt{B}}\right)$\\
    (ii)$\left|\langle \cfrac{\partial l}{\partial w_s^{(0)}},o
_1\rangle\right|\le O\left(mC_2\right)+\Tilde{O}\left(\cfrac{mC_2}{\sqrt{B}}\right)$\\
(iii) $\langle \cfrac{\partial l}{\partial w_r^{(s,0)}},o_1\rangle\le-\cfrac{1}{2l}+\Tilde{O}\left(\cfrac{1}{l\sqrt{B}}\right)$ and $\langle \cfrac{\partial l}{\partial w_r^{(s,0)}},o_1\rangle\ge-\cfrac{1}{2}-\Tilde{O}\left(\cfrac{1}{\sqrt{B}}\right)$\\
(iv) $\langle \cfrac{\partial l}{\partial w_r^{(s,0)}},o_2\rangle\ge0$ and $\langle \cfrac{\partial l}{\partial w_r^{(s,0)}},o_2\rangle\le\cfrac{1}{2}+\Tilde{O}\left(\cfrac{1}{\sqrt{B}}\right)$\\
(v) $\left|\langle \cfrac{\partial l}{\partial w_r^{(s,0)}},q\rangle\right|\le O\left(\cfrac{1}{d}\right)+\Tilde{O}\left(\cfrac{1}{\sqrt{B}}\right)$

\end{lemma}

\begin{proof}
    For any task-irrelevant pattern $q$,

$|\delta_{1,q}^{(s,0)}|=|\sigma_1^{(s,0)}-\sigma_q^{(s,0)}|=|\sum_{r=1}^m\text{ReLU}(\langle w_r^{(s,0)},o_1\rangle)-\sum_{r=1}^m\text{ReLU}(\langle w_r^{(s,0)},q\rangle)|= O(mC_2)$

Similarly, $|\delta_{2,q}^{(s,0)}|= O(mC_2)$ and for any two different task-irrelevant patterns $q,q^\prime\in\mathcal{P}$, $\left|\delta_{q,q^\prime}^{(s,t)}\right|=O(mC_2)$.\\\\

We denote $\mathcal{B}_0$ as the randomly sampled batch before the first update of SGD.\\\\
(i)\\

$\displaystyle\langle \cfrac{\partial l}{\partial w_s^{(0)}},q\rangle=\cfrac{1}{B}\sum_{x\in\mathcal{B}_0}\langle \cfrac{\partial l(x,y)}{\partial w_s^{(0)}},q\rangle$\\

Let us define the set $S_q:=\{(x,y)\sim\mathcal{D}:\exists j\in J_s^{(0)} \text{ s.t. } x^{(j)}=q\}$. Then, $p_q^{(s,0)}:=\mathbb{P}\left[(x,y)\sim\mathcal{D}:(x,y)\in S_q\right]$\\ Here, $p_q^{(s,0)}=O(\cfrac{1}{d})$\\

Therefore, $\displaystyle\langle \cfrac{\partial l}{\partial w_s^{(0)}},q\rangle=\cfrac{1}{B}\sum_{x\in\mathcal{B}_0\cap S_q}\langle \cfrac{\partial l(x,y)}{\partial w_s^{(0)}},q\rangle+\cfrac{1}{B}\sum_{x\in\mathcal{B}_0\cap S^\prime_q}\langle \cfrac{\partial l(x,y)}{\partial w_s^{(0)}},q\rangle$\\

Now, from equation (\ref{eq_a_1}), for any $(x,y)\in \mathcal{D}\backslash S_q$, $\langle \cfrac{\partial l(x,y)}{\partial w_s^{(0)}},q\rangle=0$\\

Therefore, $\displaystyle\langle \cfrac{\partial l}{\partial w_s^{(0)}},q\rangle=\cfrac{1}{B}\sum_{x\in\mathcal{B}_0\cap S_q}\langle \cfrac{\partial l(x,y)}{\partial w_s^{(0)}},q\rangle$\\

Now, as $s\in S_1$, $a^{(s)}=+1$\\
Furthermore, for any $(x,y)$, $l_q^{(s,0)}(x)G_q^{(s,0)}(x)\le1$\\
Therefore, as $|\delta_{1,q}^{(s,0)}|=O(mC_2)$, $|\delta_{2,q}^{(s,0)}|= O(mC_2)$ and $|\delta_{q,q^\prime}^{(s,0)}|= O(mC_2)$, from equation (\ref{eq_a_1}), $\displaystyle\left|\langle \cfrac{\partial l}{\partial w_s^{(0)}},q\rangle\right|\le\cfrac{\left|\mathcal{B}_0\cap S_q\right|}{B}O(mC_2)$\\
Now, from Lemma \ref{lemma_a_1}, \textit{w.h.p.} $\cfrac{\left|\mathcal{B}_0\cap S_q\right|}{B}\le p_q^{(s,0)}+\Tilde{O}\left(\cfrac{1}{\sqrt{B}}\right)$ which implies,\\ $\left|\langle \cfrac{\partial l}{\partial w_s^{(0)}},q\rangle\right|\le O\left(\cfrac{mC_2}{d}\right)+\Tilde{O}\left(\cfrac{mC_2}{\sqrt{B}}\right)$\\\\

(ii) Using equation (\ref{eq_a_2}) and the fact that $p_1^{(s,0)}=\Omega(1)$, by following the same procedure as in the proof of statement (i), we can complete the proof.\\\\

(iii) Using equation (\ref{eq_a_5}) and the fact that $p_1^{(s,0)}=\Omega(1)$,  by following the same procedure as in the proof of the statement (i) we can complete the proof.\\\\

(iv) As $p_2^{(s,0)}\ge0$, using equation (\ref{eq_a_6}) and the same procedure as in the proof of the statement (i) we can show that $\langle \cfrac{\partial l}{\partial w_r^{(s,0)}},o_2\rangle\ge0$. Similarly, as $G_2^{(s,0)}(x)\le 1$ we can show that $\langle \cfrac{\partial l}{\partial w_r^{(s,0)}},o_2\rangle\le\cfrac{1}{2}+\Tilde{O}\left(\cfrac{1}{\sqrt{B}}\right)$.\\

(v) Using equation (\ref{eq_a_4}) and the fact that for any $(x,y)\sim\mathcal{D}$, $l_q^{(s,t)}(x)G_q^{(s,t)}(x)\le1$ and by following the same procedure as in the proof of the statement (i) we can complete the proof.
\end{proof}

\begin{lemma}\label{lemma_a_3}
    For any expert $s\in S_2$ such that $p_2^{(s,0)}=\Omega(1)$, \textit{w.h.p.} over a randomly sampled batch of size $B$ we can ensure that,

    (i) $\left|\langle \cfrac{\partial l}{\partial w_s^{(0)}},q\rangle\right|\le O\left(\cfrac{mC_2}{d}\right)+\Tilde{O}\left(\cfrac{mC_2}{\sqrt{B}}\right)$\\
    (ii) $\left|\langle \cfrac{\partial l}{\partial w_s^{(0)}},o
_2\rangle\right|\le O\left(mC_2\right)+\Tilde{O}\left(\cfrac{mC_2}{\sqrt{B}}\right)$\\
(iii) $\langle \cfrac{\partial l}{\partial w_r^{(s,0)}},o_1\rangle\le\cfrac{1}{2}+\Tilde{O}\left(\cfrac{1}{\sqrt{B}}\right)$ and $\langle \cfrac{\partial l}{\partial w_r^{(s,0)}},o_1\rangle\ge0$\\
(iv) $\langle \cfrac{\partial l}{\partial w_r^{(s,0)}},o_2\rangle\ge-\cfrac{1}{2}-\Tilde{O}\left(\cfrac{1}{\sqrt{B}}\right)$ and $\langle \cfrac{\partial l}{\partial w_r^{(s,0)}},o_2\rangle\le-\cfrac{1}{2l}+\Tilde{O}\left(\cfrac{1}{l\sqrt{B}}\right)$\\
(v) $\left|\langle \cfrac{\partial l}{\partial w_r^{(s,0)}},q\rangle\right|\le O\left(\cfrac{1}{d}\right)+\Tilde{O}\left(\cfrac{1}{\sqrt{B}}\right)$
\end{lemma}

\begin{proof}
    Using the Lemma \ref{lemma_a_1} and following the same procedure as in Lemma \ref{lemma_a_2}, we can complete the proof.
\end{proof}

\begin{lemma}\label{lemma_a_4}
    For any expert $s\in S_1$ such that $p_1^{(s,0)}=\Omega(1)$, by selecting $\eta_r=O\left(\cfrac{\eta_eC_p}{ml^2C_2^2}\right)$ and $B=\Tilde{\Omega}\left(l^2d^2\right)$, we can ensure that after $T^{\prime\prime}=O\left(\cfrac{lC_2}{\eta_e}\right)$ iterations,
    
    (i) for any task-irrelevant pattern $q$, $\delta_{1,q}^{(s,T^{\prime\prime})}=\Omega\left(mC_2\right)$, $|\delta_{2,q}^{(s,T^{\prime\prime})}|=O(mC_2)$ and, for any different two task-irrelevant pattern $q$ and $q^\prime$ $|\delta_{q,q^\prime}^{(s,T^{\prime\prime})}|=O(mC_2)$\\
    (ii) for any task-irrelevant pattern $q$ such that $\langle w_s^{(0)},q\rangle<\langle w_s^{(0)},o_1\rangle$, $\langle w_s^{(T^{\prime\prime})},o_1\rangle-\langle w_s^{(T^{\prime\prime})},q\rangle=\Omega(C_p)$ and $\langle w_s^{(T^{\prime\prime})},o_1\rangle-\langle w_s^{(T^{\prime\prime})},q\rangle<\log l$\\
    (iii) for any task-irrelevant pattern $q$ such that $\langle w_s^{(0)},q\rangle>\langle w_s^{(0)},o_1\rangle$, $\langle w_s^{(T^{\prime\prime})},q\rangle-\langle w_s^{(T^{\prime\prime})},o_1\rangle=\Omega(C_p)$ and $\langle w_s^{(T^{\prime\prime})},q\rangle-\langle w_s^{(T^{\prime\prime})},o_1\rangle<\log l$\\
    (iii) $p_1^{(s,T^{\prime\prime})}=p_1^{(s,0)}$
\end{lemma}

\begin{proof}
    From the statement (i) of the Lemma \ref{lemma_a_2}, \textit{w.h.p.} over a randomly sampled batch, $\left|\langle \cfrac{\partial l}{\partial w_s^{(0)}},q\rangle\right|\le O(\cfrac{mC_2}{d})+\Tilde{O}(\cfrac{mC_2}{\sqrt{B}})$\\

Therefore, $\langle w_s^{(1)},q\rangle=\langle w_s^{(0)},q\rangle-\eta_r\langle \cfrac{\partial l}{\partial w_s^{(0)}},q\rangle\le\langle w_s^{(0)},q\rangle+O(\cfrac{mC_2}{d}\eta_r)+\Tilde{O}(\cfrac{mC_2}{\sqrt{B}}\eta_r)$\\

On the other hand, from the statement (ii) of the Lemma \ref{lemma_a_2}, \textit{w.h.p.} over a randomly sampled batch, $\left|\langle \cfrac{\partial l}{\partial w_s^{(0)}},o_1\rangle\right|\le O(mC_2)+\Tilde{O}(\cfrac{mC_2}{\sqrt{B}})$\\

Therefore, $\langle w_s^{(1)},o_1\rangle=\langle w_s^{(0)},o_1\rangle-\eta_r\langle \cfrac{\partial l}{\partial w_s^{(0)}},o_1\rangle\ge\langle w_s^{(0)},o_1\rangle-O(mC_2\eta_r)-\Tilde{O}(\cfrac{mC_2}{\sqrt{B}}\eta_r)$\\

Therefore, if $\langle w_s^{(0)},q\rangle<\langle w_s^{(0)},o_1\rangle$,
\begin{align*}
    \langle w_s^{(1)},o_1\rangle-\langle w_s^{(1)},q\rangle&\ge\langle w_s^{(0)},o_1\rangle-\langle w_s^{(0)},q\rangle-O(mC_2\eta_r)-O(\cfrac{mC_2}{d}\eta_r)-\Tilde{O}(\cfrac{mC_2}{\sqrt{B}}\eta_r)\\
    &\ge C_p-O(mC_2\eta_r)-\Tilde{O}(\cfrac{mC_2}{\sqrt{B}}\eta_r)
\end{align*}
and if $\langle w_s^{(0)},q\rangle>\langle w_s^{(0)},o_1\rangle$
\begin{align*}
    \langle w_s^{(1)},q\rangle-\langle w_s^{(1)},o_1\rangle&\le\langle w_s^{(0)},q\rangle-\langle w_s^{(0)},o_1\rangle+O(mC_2\eta_r)+O(\cfrac{mC_2}{d}\eta_r)+\Tilde{O}(\cfrac{mC_2}{\sqrt{B}}\eta_r)\\
    &\le \cfrac{1}{\sqrt{2}}\log l+O(mC_2\eta_r)+\Tilde{O}(\cfrac{mC_2}{\sqrt{B}}\eta_r)
\end{align*}\\
Now, by selecting $\eta_r=O(\cfrac{C_p}{mC_2})$ and $B=\Tilde{\Omega}(d^2)$, for $\langle w_s^{(0)},q\rangle<\langle w_s^{(0)},o_1\rangle$ we get ,\\ $\langle w_s^{(1)},o_1\rangle-\langle w_s^{(1)},q\rangle=\Omega(C_p)$ which ensures that $p_1^{(s,1)}\ge p_1^{(s,0)}$. Similarly for $\langle w_s^{(0)},q\rangle>\langle w_s^{(0)},o_1\rangle$ we get, $\langle w_s^{(1)},q\rangle-\langle w_s^{(1)},o_1\rangle<\log l$ as, $C_p=\Theta\left(\cfrac{1}{d}\right)$.\\

Now, for any $r\in [m]$ such that $\langle w_r^{(s,0)},o_1\rangle\ge0$, from the statement (iii) of the Lemma \ref{lemma_a_2}, \textit{w.h.p.} $\langle \cfrac{\partial l}{\partial w_r^{(s,0)}},o_1\rangle\le -(\cfrac{1}{2l}-\Tilde{O}(\cfrac{1}{l\sqrt{B}}))$ which implies $\langle w_r^{(s,1)},o_1\rangle=\langle w_r^{(s,0)},o_1\rangle-\eta_e\langle \cfrac{\partial l}{\partial w_r^{(s,0)}},o_1\rangle\ge\langle w_r^{(s,0)},o_1\rangle+\cfrac{\eta_e}{2l}-\Tilde{O}(\cfrac{\eta_e}{l\sqrt{B}})$ and hence $\sigma_1^{(s,1)}\ge\sigma_1^{(s,0)}+\Omega(\cfrac{m\eta_e}{l})-\Tilde{O}(\cfrac{m\eta_e}{l\sqrt{B}})$.\\

Again, for any $r\in [m]$ such that $\langle w_r^{(s,0)},o_2\rangle\ge0$, from the statement (iv) of the Lemma \ref{lemma_a_2}, \textit{w.h.p.} $\langle \cfrac{\partial l}{\partial w_r^{(s,0)}},o_2\rangle\ge 0$ which implies\\ $\langle w_r^{(s,1)},o_2\rangle=\langle w_r^{(s,0)},o_2\rangle-\eta_e\langle \cfrac{\partial l}{\partial w_r^{(s,0)}},o_2\rangle\le\langle w_r^{(s,0)},o_2\rangle$ and hence $\sigma_2^{(s,1)}\le\sigma_2^{(s,0)}$.\\

On the other hand, for any $r\in[m]$ and any $q$ such that $\langle w_r^{(s,0)},q\rangle\ge0$, from the statement (v) of the Lemma \ref{lemma_a_2}, \textit{w.h.p.} $\left|\langle \cfrac{\partial l}{\partial w_r^{(s,0)}},q\rangle\right|\le O\left(\cfrac{1}{d}\right)+\Tilde{O}(\cfrac{1}{\sqrt{B}})$.\\ 

Therefore, for any $q$, $\langle w_r^{(s,1)},q\rangle=\langle w_r^{(s,0)},q\rangle-\eta_e\langle \cfrac{\partial l}{\partial w_r^{(s,0)}},q\rangle\ge\langle w_r^{(s,0)},q\rangle-O(\cfrac{\eta_e}{d})-\Tilde{O}(\cfrac{\eta_e}{\sqrt{B}})$ and $\langle w_r^{(s,1)},q\rangle\le\langle w_r^{(s,0)},q\rangle+O(\eta_e\cfrac{1}{d})+\Tilde{O}(\cfrac{\eta_e}{\sqrt{B}})$ which implies $\sigma_q^{(s,1)}\ge 0$ and\\ $\sigma_q^{(s,1)}\le\sigma_q^{(s,0)}+O(\cfrac{m}{d}\eta_e)+\Tilde{O}(\cfrac{m\eta_e}{\sqrt{B}})$.\\

Therefore, for any task-irrelevant pattern $q$, $\delta_{1,q}^{(s,1)}\ge\delta_{1,q}^{(s,0)}+\Omega(\cfrac{m\eta_e}{l})-\Tilde{O}(\cfrac{m\eta_e}{\sqrt{B}})$ and for any two different task-irrelevant pattern,  $|\delta_{q,q^\prime}^{(s,1)}-\delta_{q,q^\prime}^{(s,0)}|\le O(\cfrac{m}{d}\eta_e)+\Tilde{O}(\cfrac{m\eta_e}{\sqrt{B}})$.\\

Now, selecting $B=\Tilde{\Omega}(l^2d^2)$ we get, $\delta_{1,q}^{(s,1)}-\delta_{1,q}^{(s,0)}=\Omega(\cfrac{m\eta_e}{l})$ and $|\delta_{q,q^\prime}^{(s,1)}-\delta_{q,q^\prime}^{(s,0)}|\le O(\cfrac{m}{d}\eta_e)$.\\

In that case, $\langle \cfrac{\partial l}{\partial w_s^{(1)}},q\rangle\ge-O(\cfrac{mC_2}{d})-O(\cfrac{m}{d^2}\eta_e)-\Tilde{O}(\cfrac{mC_2}{\sqrt{B}})$ and $\langle \cfrac{\partial l}{\partial w_s^{(1)}},o_1\rangle\le O(mC_2)+\Tilde{O}(\cfrac{mC_2}{\sqrt{B}})$ as $p_1^{(s,1)}\ge p_1^{(s,0)}$.\\

Therefore, for any $T^{\prime\prime}$ such that $\forall 0\le t\le T^{\prime\prime}$, $p_1^{(s,t)}\ge p_1^{(s,0)}$, by induction we can show that,\\ for any $q$, $\langle \cfrac{\partial l}{\partial w_s^{(t)}},q\rangle\ge-O(\cfrac{mC_2}{d})-\Tilde{O}(\cfrac{mC_2}{\sqrt{B}})$ and $\langle \cfrac{\partial l}{\partial w_s^{(t)}},o_1\rangle\le O(mC_2)+\Tilde{O}(\cfrac{mC_2}{\sqrt{B}})$.\\

Therefore, for $\langle w_s^{(0)},q\rangle<\langle w_s^{(0)},o_1\rangle$, $\langle w_s^{(T^{\prime\prime})},o_1\rangle-\langle w_s^{(T^{\prime\prime})},q\rangle\ge C_p-O(mC_2\eta_rT^{\prime\prime})-O(\cfrac{m}{d^2}\eta_e\eta_rT^{\prime\prime})-\Tilde{O}(\cfrac{mC_2}{\sqrt{B}}\eta_rT^{\prime\prime})$ and for $\langle w_s^{(0)},q\rangle>\langle w_s^{(0)},o_1\rangle$, $\langle w_s^{(T^{\prime\prime})},q\rangle-\langle w_s^{(T^{\prime\prime})},o_1\rangle\le \cfrac{1}{2}\log l+O(mC_2\eta_rT^{\prime\prime})+O(\cfrac{m}{d^2}\eta_e\eta_rT^{\prime\prime})+\Tilde{O}(\cfrac{mC_2}{\sqrt{B}}\eta_rT^{\prime\prime})$.\\\\

On the other hand, assuming for all $0\le t\le T^{\prime\prime}$, $p_1^{(s,t)}\ge p_1^{(s,0)}$, for any $q$ we get, $\delta_{1,q}^{(s,T^{\prime\prime})}\ge\delta_{1,q}^{(s,0)}+\Omega(\cfrac{m\eta_e}{l}T^{\prime\prime})-\Tilde{O}(\cfrac{m\eta_e}{\sqrt{B}}T^{\prime\prime})$, which implies we need $T^{\prime\prime}=O(\cfrac{lC_2}{\eta_e})$ iterations to achieve $\delta_{(1,q)}^{(s,T^{\prime\prime})}=\Omega(mC_2)$. In that case, for any different two task-irrelevant pattern $q$ and $q^\prime$, $|\delta_{q,q^\prime}^{(s,T^{\prime\prime})}|=O(mC_2)$ and $|\delta_{2,q}^{(s,T^{\prime\prime})}|=O(mC_2)$.\\

Now, to ensure that for all $0\le t\le T^{\prime\prime}$, $p_1^{(s,t)}\ge p_1^{(s,0)}$, we need $\eta_r=O(\cfrac{C_p}{mC_2}\cfrac{1}{T^{\prime\prime}})=O(\cfrac{C_p\eta_e}{mlC_2^2})$ such that, for any $q$ s.t. $\langle w_s^{(0)},q \rangle<\langle w_s^{(0)},o_1 \rangle$, $\langle w_s^{(t)},o_1\rangle-\langle w_s^{(t)},q\rangle\ge\Omega(C_p)$ for all $0\le t\le T^{\prime\prime}$.\\

In that case, for any $q$ s.t. $\langle w_s^{(0)},o_1 \rangle<\langle w_s^{(0)},q \rangle$, $\langle w_s^{(T^{\prime\prime})},q\rangle-\langle w_s^{(T^{\prime\prime})},o_1\rangle< \log l$.\\

Now, $\langle w_s^{(1)},o_1\rangle\le\langle w_s^{(0)},o_1\rangle+O(mC_2\eta_r)+\Tilde{O}(\cfrac{mC_2}{\sqrt{B}}\eta_r)$ and, $\langle w_s^{(1)},q\rangle\ge\langle w_s^{(0)},q\rangle-O(\cfrac{mC_2}{d}\eta_r)-\Tilde{O}(\cfrac{mC_2}{\sqrt{B}}\eta_r)$.\\

Therefore, if $\langle w_s^{(0)},o_1\rangle<\langle w_s^{(0)},q\rangle$,
\begin{align*}
    \langle w_s^{(1)},q\rangle-\langle w_s^{(1)},o_1\rangle&\ge\langle w_s^{(0)},q\rangle-\langle w_s^{(0)},o_1\rangle-O(mC_2\eta_r)-O(\cfrac{mC_2}{d}\eta_r)-\Tilde{O}(\cfrac{mC_2}{\sqrt{B}}\eta_r)\\
    &\ge C_p-O(mC_2\eta_r)-\Tilde{O}(\cfrac{mC_2}{\sqrt{B}}\eta_r)
\end{align*}

Therefore, for any $q$ s.t. $\langle w_s^{(0)},o_1\rangle<\langle w_s^{(0)},q\rangle$, by the selection of $\eta_r$, $\langle w_s^{(1)},q\rangle-\langle w_s^{(1)},o_1\rangle=\Omega(C_p)$ which implies $p_1^{(s,1)}=p_1^{(s,0)}$.\\

Now, for any $r\in [m]$ such that $\langle w_r^{(s,0)},o_1\rangle\ge0$, from the statement (iii) of the Lemma \ref{lemma_a_2}, \textit{w.h.p.} $\langle \cfrac{\partial l}{\partial w_r^{(s,0)}},o_1\rangle\ge -O(1)-\Tilde{O}(\cfrac{1}{\sqrt{B}})$ which implies $\langle w_r^{(s,1)},o_1\rangle\le\langle w_r^{(s,0)},o_1\rangle+O(\eta_e)+\Tilde{O}(\cfrac{\eta_e}{\sqrt{B}})$ and hence $\sigma_1^{(s,1)}\le\sigma_1^{(s,0)}+O(m\eta_e)+\Tilde{O}(\cfrac{m\eta_e}{\sqrt{B}})$.\\

Therefore, for any $q$, $\delta_{1,q}^{(s,1)}-\delta_{1,q}^{(s,0)}\le O(m\eta_e)$ and as from the statement (v) of the Lemma \ref{lemma_a_2}, \textit{w.h.p.} $\langle \cfrac{\partial l}{\partial w_r^{(s,0)}},o_2\rangle\le\cfrac{1}{2}+\Tilde{O}\left(\cfrac{1}{\sqrt{B}}\right)$, $\delta_{q,2}^{(s,1)}-\delta_{q,2}^{(s,0)}\le O(m\eta_e)$.\\

Now, as $p_1^{(s,1)}=p_1^{(s,0)}$,
\begin{align*}
    \langle w_s^{(2)},o_1\rangle\le\langle w_s^{(1)},o_1\rangle+O(mC_2\eta_r)+O(m\eta_e\eta_r)+\Tilde{O}(\cfrac{mC_2}{\sqrt{B}}\eta_r)+\Tilde{O}(\cfrac{m\eta_e}{\sqrt{B}}\eta_r)
\end{align*}

and,
\begin{align*}
    \langle w_s^{(2)},q\rangle\ge\langle w_s^{(1)},q\rangle-O(\cfrac{mC_2}{d}\eta_r)-O(\cfrac{m}{d}\eta_e\eta_r)-\Tilde{O}(\cfrac{mC_2}{\sqrt{B}}\eta_r)-\Tilde{O}(\cfrac{m\eta_e}{d\sqrt{B}}\eta_r)
\end{align*}

Therefore, for any $q$ s.t. $\langle w_s^{(0)},o_1\rangle<\langle w_s^{(0)},q\rangle$,
\begin{align*}
    &\langle w_s^{(2)},q\rangle-\langle w_s^{(2)},o_1\rangle\\
    &\ge\langle w_s^{(1)},q\rangle-\langle w_s^{(1)},o_1\rangle-O(mC_2\eta_r)-O(m\eta_e\eta_r)-\Tilde{O}(\cfrac{mC_2}{\sqrt{B}}\eta_r)-\Tilde{O}(\cfrac{m\eta_e}{\sqrt{B}}\eta_r)\\
    &=C_p-O(C_p\cfrac{2}{T})-O(C_p\cfrac{l}{T^{\prime\prime^2}})=\Omega(C_p)
\end{align*} which implies $p_1^{(s,2)}=p_1^{(s,0)}$.\\

Therefore, by induction, for any $t\le T^{\prime\prime}$, and for any $q$ s.t. $\langle w_s^{(0)},o_1\rangle<\langle w_s^{(0)},q\rangle$, by selecting $\eta_r=O(\cfrac{\eta_eC_p}{ml^2C_2^2})$,
\begin{align*}
    \langle w_s^{(t)},q\rangle-\langle w_s^{(t)},o_1\rangle\ge C_p-O(C_p\cfrac{t}{T^{\prime\prime}})-O(C_p\cfrac{t^2}{T^{\prime\prime^2}})
\end{align*} which implies, for any $q$ s.t. $\langle w_s^{(0)},o_1\rangle<\langle w_s^{(0)},q\rangle$, $\langle w_s^{(T^{\prime\prime})},o_1\rangle<\langle w_s^{(T^{\prime\prime})},q\rangle=\Omega(C_p)$ and hence $p_1^{(s,T^{\prime\prime})}=p_1^{(s,0)}$.\\

Similarly, for any $t\le T^{\prime\prime}$, and for any $q$ s.t. $\langle w_s^{(0)},q\rangle<\langle w_s^{(0)},o_1\rangle$,
\begin{align*}
    \langle w_s^{(t)},o_1\rangle-\langle w_s^{(t)},q\rangle&\le \langle w_s^{(1)},o_1\rangle-\langle w_s^{(1)},q\rangle+O(C_p\cfrac{t}{T^{\prime\prime}})+O(C_p\cfrac{l(t-1)}{T^{\prime\prime^2}})\\
    &\le\cfrac{1}{2}\log l+O(C_p\cfrac{t}{T^{\prime\prime}})+O(C_p\cfrac{t^2}{T^{\prime\prime^2}})
\end{align*} which implies, for any $q$ s.t. $\langle w_s^{(0)},q\rangle<\langle w_s^{(0)},o_1\rangle$, $\langle w_s^{(T^{\prime\prime})},o_1\rangle-\langle w_s^{(T^{\prime\prime})},q\rangle< \log l$.\\
\end{proof}
\begin{lemma}\label{lemma_a_5}
    For any expert $s\in S_2$ such that $p_2^{(s,0)}=\Omega(1)$, for $\eta_r$ and $B$ as in Lemma \ref{lemma_a_4}, we can ensure that after $T^{\prime\prime}$ iterations where $T^{\prime\prime}$ is as in Lemma \ref{lemma_a_4},

    (i) for any task-irrelevant pattern $q$, $\delta_{2,q}^{(s,T^{\prime\prime})}=\Omega\left(mC_2\right)$ , $|\delta_{1,q}^{(s,T^{\prime\prime})}|=O(mC_2)$ and, for any different two task-irrelevant pattern $q$ and $q^\prime$ $|\delta_{q,q^\prime}^{(s,T^{\prime\prime})}|=O(mC_2)$\\
    (ii) for any task-irrelevant pattern $q$ such that $\langle w_s^{(0)},q\rangle<\langle w_s^{(0)},o_2\rangle$, $\langle w_s^{(T^{\prime\prime})},o_2\rangle-\langle w_s^{(T^{\prime\prime})},q\rangle=\Omega(C_p)$, and $\langle w_s^{(T^{\prime\prime})},o_2\rangle-\langle w_s^{(T^{\prime\prime})},q\rangle<\log l$\\
    (iii) for any task-irrelevant pattern $q$ such that $\langle w_s^{(0)},q\rangle>\langle w_s^{(0)},o_2\rangle$, $\langle w_s^{(T^{\prime\prime})},q\rangle-\langle w_s^{(T^{\prime\prime})},o_2\rangle=\Omega(C_p)$, and $\langle w_s^{(T^{\prime\prime})},q\rangle-\langle w_s^{(T^{\prime\prime})},o_2\rangle<\log l$\\
    (iv) $p_2^{(s,T^{\prime\prime})}= p_2^{(s,0)}$
\end{lemma}
    
\begin{proof}
    Using Lemma \ref{lemma_a_3} and following the same procedure as in Lemma \ref{lemma_a_4} we can complete the proof.
\end{proof}

\begin{lemma}\label{lemma_a_6}
    For any $t$ and $q$ with $\langle w_s^{(t)},q\rangle<\langle w_s^{(t)},o_1\rangle$ such that $\langle w_s^{(t)},o_1\rangle-\langle w_s^{(t)},q\rangle\le2\log l$, $G_1^{(s,t)}(x)(1-G_1^{(s,t)}(x))\ge\cfrac{1}{4l}$, where $(x,y=+1)\sim\mathcal{D}$ s.t. $\exists j\in J_s^{(t)}(x)$ with $x^{(j)}=q$. Similarly, for any $t$ and $q$ with $\langle w_s^{(t)},q\rangle<\langle w_s^{(t)},o_2\rangle$ such that $\langle w_s^{(t)},o_2\rangle-\langle w_s^{(t)},q\rangle\le2\log l$, $G_2^{(s,t)}(x)(1-G_2^{(s,t)}(x))\ge\cfrac{1}{4l}$, where $(x,y=-1)\sim\mathcal{D}$ s.t. $\exists j\in J_s^{(t)}(x)$ with $x^{(j)}=q$
\end{lemma}

\begin{proof}
    For any $t$ and $q$ with $\langle w_s^{(t)},q\rangle<\langle w_s^{(t)},o_1\rangle$, as $\langle w_s^{(t)},o_1\rangle-\langle w_s^{(t)},q\rangle\le\log l$, for any $(x,y)\sim\mathcal{D}$ s.t. $\exists j\in J_s^{(t)}(x)$ with $x^{(j)}=q$, $G_1^{(s,t)}(x)(1-G_1^{(s,t)}(x))\ge\min\{\cfrac{\cfrac{(l-1)}{l^2}}{(1+\cfrac{(l-1)}{l^2})^2},\cfrac{(l-1)}{l^2}\}=\cfrac{\cfrac{(l-1)}{l^2}}{(1+\cfrac{(l-1)}{l^2})^2}$.\\

Now, $\cfrac{\cfrac{(l-1)}{l^2}}{(1+\cfrac{(l-1)}{l^2})^2}=\cfrac{l^2(l-1)}{(l^2+l-1)^2}$.\\

Now, let there exists a constant $C>0$ such that $\cfrac{l^2(l-1)}{(l^2+l-1)^2}\ge\cfrac{C}{l}\Leftrightarrow l^4(1-C)-l^3(1+2C)+Cl^2+2Cl-C\ge 0$.\\

Now, $Cl^2+2Cl-C>0$ as $l\ge2$. Therefore, $l^3(1+2C)\le l^4(1-C)$ satisfies $l^4(1-C)-l^3(1+2C)+Cl^2+2Cl-C\ge 0$.\\

Now, $l^3(1+2C)\le l^4(1-C)\Leftrightarrow C\le\cfrac{l-1}{l+2}$. Now, $\cfrac{l-1}{l+2}\ge\cfrac{1}{4}$ as $l\ge 2$. Hence, picking $C=\cfrac{1}{4}$ satisfies that $\cfrac{l^2(l-1)}{(l^2+l-1)^2}\ge\cfrac{1}{4l}$ which implies $G_1^{(s,t)}(x)(1-G_1^{(s,t)}(x))\ge\cfrac{1}{4l}$.\\

Similarly, we can show that for any $t$ and $q$ with $\langle w_s^{(t)},q\rangle<\langle w_s^{(t)},o_2\rangle$ such that $\langle w_s^{(t)},o_2\rangle-\langle w_s^{(t)},q\rangle\le2\log l$, $G_2^{(s,t)}(x)(1-G_2^{(s,t)}(x))\ge\cfrac{1}{4l}$, where $(x,y=-1)\sim\mathcal{D}$ s.t. $\exists j\in J_s^{(t)}(x)$ with $x^{(j)}=q$.
\end{proof}

\section{Proof of Lemma \ref{lemma_2}}

\begin{lemma}[Full version of the \textbf{Lemma \ref{lemma_2}}]\label{lemma_m_2}
    Suppose the expert learning rate $\eta_e$ such that, the router learning rate  $\eta_r=O\left(\cfrac{\eta_eC_p}{mdl^2C_2^2}\right)$,  the batch-size $B=\Tilde{\Omega}(l^2d^2)$, and the number of iterations $T=\Omega\left( \cfrac{l^2C_2}{\eta_e}\sqrt{\cfrac{\log l}{C_p}}\right)$. Then,

    For any expert $s\in S_1$ such that $p_1^{(s,0)}=
    O\left(\cfrac{1}{d}\right)$, we have

    (i) $p_1^{(s,T)}=O(1/d)$,\\
    (ii) $\Delta_s^{(T)}=O(\log^2 l/\sqrt{d})+O(l^4\log^2l/d^2)$.

    and, for any expert $s\in S_2$ such that $p_2^{(s,0)}=O\left(\cfrac{1}{d}\right)$, we have

    (iii) $p_2^{(s,T)}=O(1/d)$,\\
    (iv) $\Delta_s^{(T)}=O(\log^2 l/\sqrt{d})+O(l^4\log^2l/d^2)$.
\end{lemma}

\begin{proof}
For any expert $s\in S_1$ such that $p_1^{(s,0)}=
    O\left(\cfrac{1}{d}\right)$,
    for any $q$, $\langle \cfrac{\partial l}{\partial w_s^{(0)}},q\rangle\le O(\cfrac{mC_2}{d})+\Tilde{O}(\cfrac{mC_2}{\sqrt{B}})$.\\

Therefore, with $B=\Tilde{\Omega}(l^2d^2)$, $\langle \cfrac{\partial l}{\partial w_s^{(0)}},q\rangle\le O(\cfrac{mC_2}{d})$.\\

Hence, $\langle w_s^{(1)},q\rangle\ge\langle w_s^{(0)},q\rangle-O(\cfrac{mC_2}{d}\eta_r)$.\\

Similarly, $\langle w_s^{(1)},o_1\rangle\le\langle w_s^{(0)},o_1\rangle+O(\cfrac{mC_2}{d}\eta_r)$.\\

Therefore, for any $q$ s.t. $\langle w_s^{(0)},o_1\rangle<\langle w_s^{(0)},q\rangle$, $\langle w_s^{(1)},q-o_1\rangle\ge\langle w_s^{(0)},q-o_1\rangle-O(\cfrac{mC_2}{d}\eta_r)$.\\

Now, with $\eta_r=O(\cfrac{\eta_eC_p}{ml^2C_2^2})$ and $\langle w_s^{(0)},q-o_1\rangle\ge C_p$, we get $\langle w_s^{(1)},q-o_1\rangle=\Omega(C_p)$.\\

Therefore, $p_1^{(s,1)}\le p_1^{(s,0)}$.\\

Now, for $B=\Tilde{\Omega}(l^2d^2)$, w.h.p. for any $q$, as $p_1^{(s,0)}=O(\cfrac{1}{d})$,\\ $\delta_{1,q}^{(s,1)}-\delta_{1,q}^{(s,0)}\le O(\cfrac{m\eta_e}{d})$, $\delta_{2,q}^{(s,0)}-\delta_{2,q}^{(s,1)}\le O(m\eta_e)$ and for any two different task-irrelevant patterns $q$ and $q^\prime$, $|\delta_{q,q^\prime}^{(s,0)}-\delta_{q,q^\prime}^{(s,1)}|=O(\cfrac{m\eta_e}{d})$.\\

Therefore, for any $q$ s.t. $\langle w_s^{(0)},o_1\rangle<\langle w_s^{(0)},q\rangle$,\\ $\langle w_s^{(2)},q-o_1\rangle\ge\langle w_s^{(1)},q-o_1\rangle-O(\cfrac{mC_2}{d}\eta_r)-O(\cfrac{m\eta_e}{d}\eta_r)=\Omega(C_p)$ which implies $p_1^{(s,2)}\le p_1^{(s,0)}$.\\

Therefore,\\ by induction, for any $t^\prime$ s.t. $\forall 0\le t\le t^\prime$, $p_1^{(s,t)}\le p_1^{(s,0)}=O(\cfrac{1}{d})$, for any $q$ s.t. $\langle w_s^{(0)},o_1\rangle<\langle w_s^{(0)},q\rangle$,
\begin{equation}\label{eq: 5}
    \langle w_s^{(t^\prime+1)},q-o_1\rangle\ge\langle w_s^{(0)},q-o_1\rangle-O(\cfrac{mC_2}{d}\eta_r)t^\prime-O(\cfrac{m\eta_e}{d^2}\eta_r)t^{\prime2}
\end{equation}

Now, given, $T=O(\cfrac{l^2C_2}{\eta_e}\sqrt{\cfrac{\log l}{C_p}})$.\\

Therefore, using inequality (\ref{eq: 5}), $\langle w_s^{(T+1)},q-o_1\rangle=\Omega(C_p)$ which implies that,\\ for any $t\le T+1$, $p_1^{(s,t)}\le p_1^{(s,0)}=O(\cfrac{1}{d})$.\\

Now, as $p_1^{(s,T)}\le p_1^{(s,0)}=O(\cfrac{1}{d})$, $\forall q$,
\begin{equation}\label{eq: 6}
    \langle w_s^{(T+1)},q\rangle\le\langle w_s^{(0)},q\rangle+O(\cfrac{mC_2}{d}\eta_r)T+O(\cfrac{m\eta_e}{d^2}\eta_r)T^2
\end{equation}
as $\forall t\le T$, $\delta_{q,2}^{(s,t-1)}-\delta_{q,2}^{(s,t)}\le 0$, $\delta_{q,1}^{(s,t)}-\delta_{q,1}^{(s,t-1)}\le O(\cfrac{m\eta_e}{d})$, $|\delta_{q,q^\prime}^{(s,t-1)}-\delta_{q,q^\prime}^{(s,t)}|=O(\cfrac{m\eta_e}{d})$, $\left| \delta_{2,q}^{(s,0)}\right|=O(mC_2)$ and, $\left| \delta_{q,q^\prime}^{(s,0)}\right|=O(mC_2)$.\\

Therefore, as given $T=O(\cfrac{l^2C_2}{\eta_e}\sqrt{\cfrac{\log l}{C_p}})$, using inequality (\ref{eq: 6}),\\ $\langle w_s^{(T+1)},q\rangle-\langle w_s^{(s,0)},q\rangle\le O(\cfrac{1}{d^{3/2}}\log l)$.\\

Therefore, $\langle w_s^{(T+1)},q\rangle^2\le \langle w_s^{(0)},q\rangle^2+O(\cfrac{1}{d^{3/2}}\log^2 l)$.\\

Similarly, we can show that, $\langle w_s^{(T+1)},o_1\rangle^2\le \langle w_s^{(0)},o_1\rangle^2+O(\cfrac{1}{d^{3/2}}\log^2 l)$.\\

Again as $\delta_{2,q}^{(t-1)}-\delta_{2,q}^{(t)}\le O(\cfrac{m\eta_e}{d})$,
\begin{equation}\label{eq: 7}
    \langle w_s^{(T+1)},o_2\rangle\le\langle w_s^{(0)},o_2\rangle+O(mC_2\eta_r)T+O(\cfrac{m\eta_e}{d}\eta_r)T^2
\end{equation}

Therefore, with $T=O(\cfrac{l^2C_2}{\eta_e}\sqrt{\cfrac{\log l}{C_p}})$, using inequality (\ref{eq: 7}),\\ $\langle w_s^{(T+1)},o_2\rangle-\langle w_s^{(s,0)},o_2\rangle\le O(\cfrac{1}{d^{1/2}}\log l)+O(\cfrac{l^2}{d}\log l)$.\\

Therefore, $\langle w_s^{(T+1)},o_2\rangle^2\le \langle w_s^{(0)},o_2\rangle^2+O(\cfrac{l^4}{d^{2}}\log^2 l)$.\\

Therefore,
\begin{align*}
    \left|\left| w_s^{(T+1)}\right|\right|^2- \left|\left| w_s^{(0)}\right|\right|^2\le O(\cfrac{1}{d^{1/2}}\log^2 l)+O(\cfrac{l^4}{d^2}\log^2 l)
\end{align*}

Therefore, $\Delta_s^{(T)}\le O(\cfrac{1}{d^{1/2}}\log^2 l)+O(\cfrac{l^4}{d^2}\log^2 l)$.

Similarly, for any expert $s\in S_2$ such that $p_1^{(s,0)}=
    O\left(\cfrac{1}{d}\right)$, we can complete the proof for the statements (iii) and (iv).
\end{proof}

\section{Proof of Theorem \ref{thm_1}}

\begin{proof}

From Lemma \ref{lemma_m_1} (i.e., Lemma \ref{lemma_1}) and \ref{lemma_m_2} (i.e., Lemma \ref{lemma_2}), after pruning the model with expert-pruning-ratio $\rho\le\gamma$, there exist $s\in S_1$ such that $p_1^{(s,0)}=\Omega(1)$ and exist $s\in S_2$ such that $p_2^{(s,0)}=\Omega(1)$.

Again, from Lemma \ref{lemma_m_1}, for any expert $s\in S_1$ such that $p_1^{(s,0)}=\Omega(1)$, $\langle w_r^{(s,T+1)},o_1\rangle=\Omega(lC_2\sqrt{d}\sqrt{\log l})$

Similarly, for any $s\in S_2$ s.t. $p_2^{(s,0)}=\Omega(1)$, $\langle w_r^{(s,T+1)},o_2\rangle=\Omega(lC_2\sqrt{d}\sqrt{\log l})$.\\

Now, $\forall s\in[k]$, $\forall q$ and $\forall r\in[m]$ s.t. $\langle w_r^{(s,0)},q\rangle\ge0$, $\forall t$, $\langle \cfrac{\partial l}{\partial w_r^{(s,t)}},q\rangle\ge-O(\cfrac{1}{d})-\Tilde{O}(\cfrac{1}{\sqrt{B}})$.\\

Therefore,
\begin{align*}
    \langle w_r^{(s,T+1)},q\rangle\le \langle w_r^{(s,0)},q\rangle+O(\cfrac{1}{d}\eta_eT)=C_2+O(\cfrac{l^2}{\sqrt{d}}\sqrt{\log l})C_2
\end{align*}

Again, $\forall s\in S_2$, $\forall t$ and, $\forall r\in [m]$, $\langle \cfrac{\partial l}{\partial w_r^{(s,t)}},o_1\rangle\ge0$ which implies $\langle w_r^{(s,T+1)},o_1\rangle\le\langle w_r^{(s,0)},o_1\rangle\le C_2$.\\

Similarly, $\forall s\in S_1$, $\forall t$ and, $\forall r\in[m]$, $\langle w_r^{(s,T+1)},o_2\rangle\le C_2$.\\

Now, after $T$ iterations of SGD and pruning with $\rho\le\gamma$, $\forall (x,y=+1)\sim\mathcal{D}$,
{\allowdisplaybreaks
\begin{align*}
    f^{(T+1,\rho)}(x)&=\sum_{s\in S_{k^\prime}}f_s^{(T+1)}(x)\\
    &=\sum_{s\in S_{k^\prime}}a^{(s)}\sum_{j\in J_s^{(T+1)}(x)}G_j^{(s,T+1)}\sum_{r=1}^m\text{ReLU}(\langle w_r^{(s,T+1)},x^{(j)}\rangle)\\
    &=\sum_{s\in S_1\cap S_{k^\prime}}\sum_{j\in J_s^{(T+1)}(x)}G_j^{(s,T+1)}\sum_{r=1}^m\text{ReLU}(\langle w_r^{(s,T+1)},x^{(j)}\rangle)\\
    &\hspace{1cm}-\sum_{s\in S_2 \cap S_{k}}\sum_{j\in J_s^{(T+1)}(x)}G_j^{(s,T^\prime+1)}\sum_{r=1}^m\text{ReLU}(\langle w_r^{(s,T+1)},x^{(j)}\rangle)\\
    &\ge \sum_{s\in S_1\cap\{s:p_1^{(s,0)}=\Omega(1)\}\cap S_{k^\prime}}\sum_{j\in J_s^{(T+1)}(x)}G_j^{(s,T+1)}\sum_{r=1}^m\text{ReLU}(\langle w_r^{(s,T+1)},x^{(j)}\rangle)\\
    &\hspace{1.5cm}-\sum_{s\in S_2\cap S_{k^\prime}}\sum_{j\in J_s^{(T+1)}(x)}G_j^{(s,T+1)}\sum_{r=1}^m\text{ReLU}(\langle w_r^{(s,T+1)},x^{(j)}\rangle)\\
    &\overset{(I)}{\ge}\sum_{s\in S_1\cap\{s:p_1^{(s,0)}=\Omega(1)\}\cap S_{k^\prime}}G_1^{(s,T+1)}(x)\sum_{r=1}^{m}\text{ReLU}(\langle w_r^{(s,T+1)},o_1\rangle)\\
    &\hspace{1.5cm}-\sum_{s\in S_2\cap S_{k^\prime}}\sum_{r=1}^{m}\text{ReLU}(\langle w_r^{(s,T+1)},q^{(s)}\rangle)\\
    &\overset{(II)}{\ge}\Omega(mlC_2\sqrt{d}\sqrt{\log l})-O(kmC_2)-O(km\cfrac{1}{\sqrt{d}}l^2\sqrt{\log l}C_2)
\end{align*}
}

Here, $(I)$ comes from the statement $(i)$ of Lemma \ref{lemma_m_1}, $q^{(s)}$ in $(I)$ denotes the task-irrelevant pattern routed to expert $s\in S_2\cap S_{k^\prime}$ and, $(II)$ comes from the statement $(ii)$ of Lemma \ref{lemma_m_1} and the fact that $\exists s\in S_1\cap S_{k^\prime}$ s.t. $p_1^{(s,0)}=\Omega(1)$.\\

Therefore, for $k=O(\sqrt{d})$, $\forall (x,y=+1)\sim\mathcal{D}, f^{(T+1,\rho)}(x)>1$.\\

Similarly, using the statements (vi),(vii) and (viii) of Lemma \ref{lemma_m_1}, and as $\exists s\in S_2\cap S_{k^\prime}$ s.t. $p_2^{(s,0)}=\Omega(1)$, we can show that, $\forall (x,y=-1)\sim\mathcal{D}$, $f^{(T+1,\rho)}(x)<-1$.\\

Therefore, $\forall (x,y)\sim\mathcal{D}$, $yf^{(T+1,\rho)}(x)>1$ which implies that,\\ $\hat{l}(f^{(T+1,\rho)}(x),y)=0$ and, $\mathbb{P}\left[\forall (x,y)\sim\mathcal{D}: yf^{(T+1,\rho)}(x)>0\right]=1$.\\
\end{proof}

\section{Proof of Lemma \ref{lemma_3}}

\begin{lemma}[Full version of the \textbf{Lemma \ref{lemma_3}}]\label{lemma_m_3}
    Suppose the expert learning rate $\eta_e$ such that, the router learning rate  $\eta_r=O\left(\cfrac{\eta_eC_p}{mdl^2C_2^2}\right)$,  the batch-size $B=\Tilde{\Omega}(l^2d^2)$, and the number of iterations $T=\Omega\left( \cfrac{l^2C_2}{\eta_e}\sqrt{\cfrac{\log l}{C_p}}\right)$. For any expert $s\in S_1$ such that $\Delta_s^{(T)}>\cfrac{3}{2}\log l$, we have
    
     (i) $p_1^{(s,T)}=1$, \\
    (ii) for every $(x,+1)\sim\mathcal{D}$, $G_j^{(s,T)}(x)>1/2$, if $x^{(j)}=o_1$,
    
    Moreover, after pruning experts with the ratio $\rho\geq\gamma$, if $s\in S_{k^\prime}$, and the number of post-pruning fine-tuning steps   satisfies $T^\prime=\Omega(kC_2l^2\sqrt{\log l}/\eta_e)$, then we have,
    
    (iii) $\langle w_r^{(s,T,T^\prime)},o_1\rangle=\Omega( kl^2C_2\sqrt{\log l})$ for a constant fraction $r\in[m]$ of $s$

    Similarly, for $s\in S_2$ such that $\Delta_s^{(T)}>\cfrac{3}{2}\log l$, we have
    
     (iv) $p_2^{(s,T)}=1$, \\
    (v) for every $(x,-1)\sim\mathcal{D}$, $G_j^{(s,T)}(x)>1/2$, if $x^{(j)}=o_2$,
    
    Moreover, after pruning experts with the ratio $\rho\geq\gamma$, if $s\in S_{k^\prime}$, and the number of post-pruning fine-tuning steps   satisfies $T^\prime=\Omega(kC_2l^2\sqrt{\log l}/\eta_e)$, then we have,
    
    (vi) $\langle w_r^{(s,T,T^\prime)},o_2\rangle=\Omega( kl^2C_2\sqrt{\log l})$ for a constant fraction $r\in[m]$ of $s\in[k]$
\end{lemma}

\begin{proof}
For any $s\in S_1$ such that $\Delta_s^{(T)}>\cfrac{3}{2}\log l$,
{\allowdisplaybreaks
\begin{align*}
    &\left|\left|w_s^{(T)}\right|\right|-\left|\left|w_s^{(0)}\right|\right|>\cfrac{3}{2}\log l\\
    &\Rightarrow\left|\left|w_s^{(T)}\right|\right|^2>\cfrac{9}{4}\log^2l+\left|\left|w_s^{(0)}\right|\right|^2+3\log l\left|\left|w_s^{(0)}\right|\right|\\
    &\Rightarrow \langle w_{s}^{(T+1)},o_1\rangle^2+\langle w_{s}^{(T+1)},o_2\rangle^2+\sum_{i=1}^{d-2}\langle w_{s}^{(T+1)},q_i\rangle^2>\cfrac{9}{4}\log^2l+\left|\left|w_s^{(0)}\right|\right|^2+3\log l\left|\left|w_s^{(0)}\right|\right|\\
    &\Rightarrow \langle w_{s}^{(T+1)},o_1\rangle^2+\left(\langle w_{s}^{(T+1)},o_2\rangle-\langle w_{s}^{(0)},o_2\rangle\right)^2+2\left(\langle w_{s}^{(T+1)},o_2\rangle-\langle w_{s}^{(0)},o_2\rangle\right)\langle w_{s}^{(0)},o_2\rangle\\
    &\hspace{0.5cm}+\langle w_{s}^{(0)},o_2\rangle^2+\sum_{i=1}^{d-2}\langle w_{s}^{(0)},q_i\rangle^2+\sum_{i=1}^{d-2}\left(\langle w_{s}^{(T+1)},q_i\rangle-\langle w_{s}^{(0)},q_i\rangle\right)^2\\
    &\hspace{0.5cm}+2\sum_{i=1}^{d-2}\left(\langle w_{s}^{(T+1)},q_i\rangle-\langle w_{s}^{(0)},q_i\rangle\right)\langle w_{s}^{(0)},q_i\rangle>\cfrac{9}{4}\log^2l+\left|\left|w_s^{(0)}\right|\right|^2+3\log l\left|\left|w_s^{(0)}\right|\right|\\
    &\Rightarrow \langle w_{s}^{(T+1)},o_1\rangle^2-\langle w_{s}^{(0)},o_1\rangle^2+\left(\langle w_{s}^{(T+1)},o_2\rangle-\langle w_{s}^{(0)},o_2\rangle\right)^2\\
    &\hspace{0.5cm}+2\left(\langle w_{s}^{(T+1)},o_2\rangle-\langle w_{s}^{(0)},o_2\rangle\right)\langle w_{s}^{(0)},o_2\rangle+\sum_{i=1}^{d-2}\left(\langle w_{s}^{(T+1)},q_i\rangle-\langle w_{s}^{(0)},q_i\rangle\right)^2\\
    &\hspace{0.5cm}+2\sum_{i=1}^{d-2}\left(\langle w_{s}^{(T+1)},q_i\rangle-\langle w_{s}^{(0)},q_i\rangle\right)\langle w_{s}^{(0)},q_i\rangle>\cfrac{9}{4}\log^2l
\end{align*}
}
Now, $\forall q$, $\langle w_{s}^{(T+1)},q\rangle-\langle w_{s}^{(0)},q\rangle\le O(\cfrac{1}{d^{3/2}}\log l)$ and, $\langle w_{s}^{(0)},q\rangle<\cfrac{1}{2}\log l$.\\

Also, $\langle w_{s}^{(T+1)},o_2\rangle-\langle w_{s}^{(0)},o_2\rangle\le O(\cfrac{1}{d^{1/2}}\log l)$ and, $\langle w_{s}^{(0)},o_2\rangle<\cfrac{1}{2}\log l$.\\

Therefore,
{\allowdisplaybreaks
\begin{align*}
    &\langle w_{s}^{(T+1)},o_1\rangle^2>\cfrac{9}{4}\log^2l\\
    &\Rightarrow \langle w_{s}^{(T+1)},o_1\rangle>\cfrac{3}{2}\log l
\end{align*}
}
On the other hand, $\forall q$, $\langle w_{s}^{(T+1)},q\rangle\le \cfrac{1}{2}\log l$.\\

Therefore, $\forall q, \langle w_{s}^{(T+1)},o_1-q\rangle>\log l$.\\

Hence, $p_1^{(s,T+1)}=1$ and, $\forall (x,+1)\sim\mathcal{D}$, $G_1^{(s,T+1)}(x)>\cfrac{1}{2}$.\\

Now, as $\langle w_{s}^{(T+1)},o_1\rangle>\cfrac{3}{2}\log l$,
\begin{align*}
    &\langle w_{s}^{(T+1)},o_1\rangle-\langle w_{s}^{(0)},o_1\rangle>\log l\Rightarrow-\sum_{t=0}^T\langle\cfrac{\partial l}{\partial w_s^{(t)}},o_1\rangle\eta_r>\log l\Rightarrow\sum_{t=0}^T\delta_{1,q}^{(s,t)}>\cfrac{2}{\eta_r}\log l\Rightarrow\sum_{t=0}^T\sigma_1^{(s,t)}>\cfrac{2}{\eta_r}\log l
\end{align*}
Now, as $\sigma_1^{(s,0)}=O(mC_2)$, for the selection of $\eta_r$ and $T$, there is contradiction if there is only $o(1)$ fraction of $r\in [m]$ of $s$ with $\langle w_r^{(s,0)},o_1\rangle\ge 0$. Therefore, there is $\Omega(1)$ fraction of $r\in[m]$ of $s$ such that $\langle w_r^{(s,0)},o_1\rangle\ge 0$.

Therefore, we need $T^\prime=O(kC_2l^2\sqrt{\log l}/\eta_e)$ iterations to show, $\langle w_r^{(s,T,T^\prime)},o_1\rangle=\Omega( kl^2C_2\sqrt{\log l})$ for $\Omega(1)$ fraction of $r\in[m]$ of $s$.

Similarly, we can complete the proof of the statements (iv), (v) and (vi).

\end{proof}

\section{Proof of Theorem \ref{thm_2}}

\begin{proof}
    For the pruning ratio of $\rho=1-O(\cfrac{1}{k})$, $|S_1\cap S_{k^\prime}|=O(1)$ and $|S_2\cap S_{k^\prime}|=O(1)$.

    Now, using Lemma \ref{lemma_3} and following the same procedure as in the proof of Theorem \ref{thm_1} we can complete the proof.
\end{proof}

\end{document}